\documentclass[english,11pt]{article}
\makeatletter

\usepackage{dsfont}

\usepackage{float}
\usepackage{color}
\usepackage{array}
\usepackage[colorlinks=true, linkcolor=red, urlcolor=blue, citecolor=gray]{hyperref}
\usepackage[margin=1.1in,footskip=0.25in]{geometry}

\usepackage{algorithm}
\usepackage[noend]{algpseudocode}
\usepackage{amsopn,amssymb,amsthm,amsmath,amsthm}
\usepackage{times}
\usepackage{graphicx} % more modern
\usepackage{caption}
\usepackage{subcaption}
\usepackage{epstopdf}
\usepackage{bm}
\usepackage{paralist}

\usepackage{tkz-graph}
\usepackage{tikz}
\usetikzlibrary{arrows}

\theoremstyle{plain}
\newtheorem{thm}{\protect\theoremname}
\theoremstyle{plain}
\newtheorem{claim}[thm]{\protect\claimname}
\theoremstyle{plain}
\newtheorem{prop}[thm]{\protect\propositionname}
\theoremstyle{plain}
\newtheorem{lem}[thm]{\protect\lemmaname}
\theoremstyle{plain}
\newtheorem{cor}[thm]{\protect\corollaryname}
\theoremstyle{definition}
\newtheorem{defn}[thm]{\protect\definitionname}
\theoremstyle{definition}

\theoremstyle{definition}
\newtheorem{rem}{\protect\remarkname}
\theoremstyle{plain}

\makeatother

\providecommand{\claimname}{Claim}
\providecommand{\lemmaname}{Lemma}
\providecommand{\propositionname}{Proposition}
\providecommand{\theoremname}{Theorem}
\providecommand{\corollaryname}{Corollary}
\providecommand{\definitionname}{Definition}
\providecommand{\assumptionname}{Assumption}
\providecommand{\remarkname}{Remark}

\global\long\def\RR{\mathbb{R}}
\global\long\def\CC{\mathbb{C}}
\global\long\def\ZZ{\mathbb{Z}}

\global\long\def\R{{\cal R}}
\global\long\def\H{{\cal H}}
\global\long\def\X{{\cal X}}
\global\long\def\Y{{\cal Y}}

\global\long\def\x{{\mathbf{x}}}
\global\long\def\w{{\mathbf{w}}}

\global\long\def\a{{\mathbf{a}}}

\global\long\def\u{{\mathbf{u}}}
\global\long\def\t{{\mathbf{t}}}
\global\long\def\y{{\mathbf{y}}}
\global\long\def\w{{\mathbf{w}}}

\global\long\def\z{{\mathbf{z}}}

\global\long\def\s{{\mathbf{s}}}

\global\long\def\f{{\mathbf{f}}}

\global\long\def\mat#1{{\ensuremath{\bm{\mathrm{#1}}}}}

\global\long\def\matA{\ensuremath{{\bm{\mathrm{A}}}}}
\global\long\def\matB{\ensuremath{{\bm{\mathrm{B}}}}}

\global\long\def\matD{\ensuremath{{\bm{\mathrm{D}}}}}

\global\long\def\matM{\ensuremath{{\bm{\mathrm{M}}}}}
\global\long\def\matR{\mat R}
\global\long\def\matS{\mat S}
\global\long\def\matSigma{\mat \Sigma}
\global\long\def\matPhi{\mat \Phi}
\global\long\def\matY{\mat Y}
\global\long\def\matI{\mat I}

\global\long\def\matZ{\mat Z}
\global\long\def\matV{\mat V}

\global\long\def\matK{\mat K}

\global\long\def\TNormS#1{\|#1\|_{2}^{2}}
\global\long\def\LTNormS#1{\left\|#1\right\|_{2}^{2}}
\global\long\def\TNorm#1{\|#1\|_{2}}
\global\long\def\LTNorm#1{\left\|#1\right\|_{2}}
\global\long\def\InfNorm#1{\|#1\|_{\infty}}

\global\long\def\FNormS#1{\|#1\|^{2}_{F}}

\global\long\def\XNorm#1#2{\|#1\|_{#2}}
\global\long\def\XNormS#1#2{\|#1\|^2_{#2}}

\global\long\def\T{\textsc{T}}

\global\long\def\Expect#1{{\mathbb{E}}\left[#1\right]}
\global\long\def\ExpectC#1#2{{\mathbb{E}}_{#1}\left[#2\right]}
\global\long\def\dotprod#1#2#3{\langle#1,#2\rangle_{#3}}

\global\long\def\Trace#1{\mathrm{Tr}\left(#1\right)}

\global\long\def\spn#1{{\bf Span}\left(#1\right)}
\global\long\def\rank#1{\mathrm{rank}(#1)}
\global\long\def\diag#1{{\mathrm{diag}}\left(#1\right)}
\global\long\def\sinc#1{\mathrm{sinc}\left(#1\right)}

\global\long\def\poly#1{\mathrm{poly}\left(#1\right)}
\newcommand*{\rect}{\mathrm{rect}}
\global\def\eqdef{\equiv}
\global\long\def\conj{*}
\newcommand{\wh}{\widehat}

\newcommand{\norm}[1]{\|#1\|}
\newcommand{\Fc}{\mathcal{F}}

\newcommand{\bs}[1]{\boldsymbol{#1}}
\newcommand{\bv}[1]{\mathbf{#1}}

\global\long\def\valpha{{\bs{\alpha}}}
\global\long\def\veta{{\bs{\eta}}}
\global\long\def\vxi{{\bs{\xi}}}

\global\long\def\rad{R}

\title{Random Fourier Features for Kernel Ridge Regression: \\
  Approximation Bounds and Statistical Guarantees\footnote{An extended abstract of this work appears in the Proceedings of the 34th International Conference on Machine Learning (ICML 2017)~\cite{ICML_paper}.} }
\author{
Haim Avron\\ \small Tel Aviv University\\ \small\texttt{haimav@post.tau.ac.il}
\and
Michael Kapralov\\ \small EPFL\\ \small \texttt{michael.kapralov@epfl.ch}
\and\and
Cameron Musco\\ \small MIT\\ \small \texttt{cnmusco@mit.edu}
\and\and\and\and
Christopher Musco\\ \small MIT\\ \small \texttt{cpmusco@mit.edu}
\and\and
Ameya Velingker\\ \small EPFL\\ \small \texttt{ameya.velingker@epfl.ch}
\and
Amir Zandieh\\ \small EPFL\\ \small \texttt{amir.zandieh@epfl.ch}
}

\begin{document}
\maketitle

\begin{abstract}

Random Fourier features is one of the most popular techniques for scaling up kernel methods, such as kernel ridge regression. However, despite impressive empirical results, the
statistical properties of random Fourier features are still not well understood.
In this paper we take steps toward filling this gap.
% in theoretical understanding.
%as well as designing improved methods.
Specifically, we approach random Fourier features
from a spectral matrix approximation point of view, give tight bounds on the
number of Fourier features required to achieve a spectral approximation, and show
how spectral matrix approximation bounds imply statistical guarantees for kernel
ridge regression.

Qualitatively, our results are twofold: on the one hand, we show that random Fourier feature approximation can provably speed up kernel ridge regression under reasonable assumptions.
At the same time, we show that the method is suboptimal, and sampling from a modified distribution in Fourier space, given by the \emph{leverage function} of the kernel, yields provably better performance.  We study this optimal sampling distribution for the  Gaussian kernel, achieving a nearly complete characterization for the case of low-dimensional bounded datasets. Based on this characterization, we propose an efficient sampling scheme with guarantees superior to random Fourier features in this regime.

\end{abstract}

\section{Introduction}
\label{sec:introduction}

Kernel methods constitute a powerful paradigm for devising non-parametric modeling techniques for a wide range
of problems in machine learning. One of the most elementary
%kernel-based methods
is {\em Kernel Ridge Regression (KRR)}.
Given training data $(\x_1, y_1),\dots,(\x_n,y_n)\in \X\times\Y$, where $\X \subseteq \RR^d$ is an input domain and
$\Y \subseteq \RR$ is an output domain, a positive definite kernel function $k:\X\times\X\to\RR$, and a regularization
parameter $\lambda > 0$, the response for a given input $\x$ is estimated as:
\begin{equation*}
\bar{f}(\x) \equiv \sum^n_{j=1} k(\x_j, \x) \alpha_j
\end{equation*}
where $\valpha = ( \alpha_1  \cdots  \alpha_n)^\T$ is the solution of the equation
\begin{equation}
\label{eq:linear}
(\matK + \lambda \matI_n)\valpha = \y.
\end{equation}
In the above, $\matK\in\RR^{n\times n}$ is the {\em kernel matrix} or {\em Gram matrix} defined by $\matK_{ij}\equiv k(\x_i, \x_j)$
and $\y \equiv [y_1 \cdots y_n]^\T$ is the vector of responses.
The KRR estimator can be derived by minimizing
a regularized square loss objective function over a hypothesis space defined by the reproducing kernel Hilbert space associated with
$k(\cdot, \cdot)$; however, the details are not important for this paper.

While simple, KRR is a powerful technique that is well understood statistically and capable of achieving impressive empirical results.
%Furthermore, the statistical aspects of the estimator are well understood.
Nevertheless, the method has a key
weakness: computing the KRR estimator can be prohibitively expensive for large datasets. Solving~\eqref{eq:linear} generally requires
$\Theta(n^3)$ time\footnote{The running time can be improved using fast matrix products. However fast matrix products are typically not employed in practice due to large hidden constants.} and $\Theta(n^2)$ memory. %This can be prohibitive for large datasets.
Thus, the design of scalable methods for KRR (and other kernel based methods) has been the focus of intensive research in recent years~\cite{ZhangDuchiWainwright15, AlaouiMahoney15, MuscoMusco16, ACW16}.

One of the most popular approaches to scaling up kernel based methods is random Fourier features sampling, originally proposed by Rahimi and Recht~\cite{RahimiRecht07}. For shift-invariant kernels (e.g. the Gaussian kernel), Rahimi and Recht~\cite{RahimiRecht07} presented
a distribution $D$ on functions from $\X$ to $\CC^s$ ($s$ is a parameter) such that for every $\x,\z \in \RR^d$
\begin{equation*}
k(\x,\z) = \ExpectC{\varphi \sim D}{\varphi(\x)^\conj \varphi(\z)}\,.
\end{equation*}
The random features approach is then to sample a $\varphi$ from $D$ and use $\tilde{k}(\x, \z)\equiv \varphi(\x)^\conj \varphi(\z)$
as a surrogate kernel. The resulting approximate KRR estimator can be computed in $O(ns^2)$ time and $O(ns)$ memory
(see \S\ref{sec:random-fourier-features} for details), giving substantial computational savings if $s \ll n$.

This approach naturally raises the question:  how large should $s$ be to ensure a high quality estimator?
%This naturally raises the question: how large should $s$ be for a high quality estimator?
Or, using
the exact KRR estimator as a natural baseline: how large should $s$ be
for the random Fourier features estimator to be almost as good as the exact KRR estimator?
Answering this question can help us determine when random Fourier features can be useful, whether the method needs to be improved, and how to go about improving it.

The original random Fourier features analysis
\cite{RahimiRecht07} bounds the point-wise distance between $k(\cdot, \cdot)$ and $\tilde{k}(\cdot, \cdot)$
(for other approaches for analyzing random Fourier features, see \S\ref{sec:related-work}).
However, the bounds do not naturally lead to an answer to the aforementioned question.
In contrast, spectral approximation bounds on the entire surrogate kernel matrix, i.e. of the form
\begin{equation}
\label{eq:approx-bound-intro}
(1-\Delta) (\matK + \lambda \matI_n) \preceq \tilde \matK + \lambda \matI_n \preceq (1 + \Delta) (\matK + \lambda \matI_n)\,,
\end{equation}
naturally have statistical and algorithmic implications. Indeed, in~\S\ref{sec:approx-bounds} we show that when~\eqref{eq:approx-bound-intro} holds
we can bound the excess risk introduced by the random Fourier features estimator when compared to the KRR estimator. We also show that $\tilde \matK+ \lambda \matI_n$ can be used as an effective preconditioner for the solution of~\eqref{eq:linear}.
This motivates the study of how large $s$ should be as a function of $\Delta$ for~\eqref{eq:approx-bound-intro} to hold.

In this paper we rigorously analyze the relation between the number of random Fourier features and the
spectral approximation bound~\eqref{eq:approx-bound-intro}. Our main results are
the following:
\begin{itemize}

\item We give an upper bound on the number of random features needed to achieve~\eqref{eq:approx-bound-intro} (Theorem~\ref{thm:rf-bound}). This bound, in conjunction with
the results in \S\ref{sec:approx-bounds},  positively shows that random Fourier features can give guarantees for KRR under reasonable assumptions.

\item We give a lower bound showing that our upper bound is tight for the Gaussian kernel (Theorem~\ref{thm:rr-samples}).

\item We show that the upper bound can be improved dramatically by modifying the sampling distribution used in
classical random Fourier features  (\S\ref{sec:ridge-and-rf}). Our sampling distribution is based on an appropriately defined
{\em leverage function} of the kernel, closely related to so-called leverage scores frequently encountered in the analysis of sampling based
methods for linear regression. Unfortunately, it is unclear how to efficiently sample using the leverage function.

\item To address the lack of an efficient way to sample using the leverage
function, we propose a novel, easy-to-sample distribution for the Gaussian
kernel which approximates the true leverage function distribution and allows
random Fourier features to achieve a significantly improved upper bound
(Theorem \ref{thm:improvedBound}). The upper bound
has an exponential dependence on the data dimension, so it is only applicable to
low dimensional datasets. Nevertheless, our results demonstrate that the classic random Fourier sampling distribution can be improved for spectral approximation
and motivates further study. As an application, our improved understanding of
the leverage function yields a novel asymptotic bound on the statistical
dimension of Gaussian kernel matrices over bounded datasets, which may be of
independent interest (Corollary~\ref{cor:sd-bound}).
\end{itemize}

\section{Preliminaries}

\subsection{Setup and Notation}

%We denote scalars using Greek letters or using $x,y,\dots$. Vectors
%are denoted by $\x,\y,\dots$ and matrices by $\matA,\mat B,\dots$.
The complex conjugate of $x \in \CC$ is denoted by $x^\conj$. For
a vector $\x$ or a matrix $\matA$, $\x^\conj$ or $\matA^\conj$ denotes
the Hermitian transpose.
 The $l\times l$ identity matrix is denoted $ \matI_l$.
We use the convention that vectors are column-vectors.
%We use $\nnz{\cdot}$ to denote the number of nonzeros in a vector or matrix.

A Hermitian matrix $\matA$ is positive semidefinite (PSD) if
$\x^\conj \matA \x \geq 0$ for every vector $\x$. %It is positive definite (PD) if
%$\x^\conj \matA \x > 0$ for every vector $\x \neq 0$.
For any two Hermitian matrices $\matA$ and $\matB$ of the same size, $\matA \preceq \matB$ means
that $\matB - \matA$ is PSD.

We use $L_2(d\rho)=L_2(\RR^d, d\rho)$ to denote
the space of complex-valued square-integrable functions with respect to some measure $\rho(\cdot)$. $L_2(d\rho)$ is
a Hilbert space equipped with the inner product
\begin{eqnarray*}
\langle f,g \rangle_{L_2(d\rho)} = \int_{\RR^d} f(\veta)g(\veta)^\conj d\rho(\veta) = \int_{\RR^d} f(\veta)g(\veta)^\conj p_\rho(\veta) d\veta\,.
\end{eqnarray*}
In the above, $p_\rho(\cdot)$ is the density associated with $\rho(\cdot)$ (assuming one exists).

We denote the training set by $(\x_1, y_1), \dots, (\x_n, y_n) \in \X \times \Y \subseteq \RR^d \times \RR$.
Note that $n$ denotes the number of training examples, and $d$ their dimension.
We denote the kernel, which is a function from $\X \times \X$ to $\RR$, by $k$.
We denote the kernel matrix by $\matK$, with $\matK_{ij} \equiv k(\x_i, \x_j)$.
The associated reproducing kernel Hilbert space (RKHS) is denoted by ${\cal H}_k$,
and the associated inner product by $\dotprod{\cdot}{\cdot}{{\cal H}_k}$.
Some results are stated for the Gaussian kernel $k(\x,\z) = \exp(-\TNormS{\x-\z}/2\sigma^2)$
for some bandwidth parameter $\sigma$.

We use $\lambda = \lambda_n$ to denote the ridge regularization parameter. While for brevity we omit the $n$ subscript,
the choice of regularization parameter generally depends on $n$. Typically, $\lambda_n = \omega(1)$ and
$\lambda_n = o(n)$. See Caponnetto and De Vito~\cite{CaponnettoDeVito2007} and Bach~\cite{Bach13} for discussion on
the asymptotic behavior of $\lambda_n$, noting that in our notation, $\lambda$ is scaled by an $n$ factor as compared to those works. As the ratio between $n$ and $\lambda$ will be an important quantity in our bounds, we denote it as $n_\lambda \eqdef n/\lambda$.

The {\em statistical dimension} or {\em effective degrees of freedom} given
the regularization parameter $\lambda$ is denoted by
$s_\lambda(\matK)\equiv\Trace{(\matK+\lambda \matI_n)^{-1} \matK}$.

%\begin{defn}
%	The Fourier transform of a continuous function $f : \RR^d \to \RR$ is defined as follows:
%	$$\wh{f}(\veta) =  \F\Big\{ f \Big\}\Big|_\veta =  \int_{\RR^d} f(\x) e^{-2\pi i \x^\T \veta} d\x $$	
%\end{defn}
%Defining the Fourier transform in the above way, we have
%\begin{equation*}
%f(\x) = \int_{\RR^d} \wh{f}(\veta) e^{2\pi i \veta^\T \x} dt 	
%\end{equation*}
%and Plancherel Theorem and Parseval's Equality takes the following form:
%\begin{eqnarray*}
%  \int_{\RR^d} \wh{f}(\veta) \wh{g}(\veta)^\conj d\veta & = & \int_{\RR^d} f(\x)g(\x)^\conj d\x\\
%  \int_{\RR^d} |\wh{f}(\veta)|^2 d\veta & = & \int_{\RR^d} |f(\x)|^2 d\x
%\end{eqnarray*}

\subsection{Random Fourier Features}
\label{sec:random-fourier-features}

\subsubsection{Classical Random Fourier Features}

Random Fourier features~\cite{RahimiRecht07} is an approach to scaling up kernel methods for shift-invariant kernels.
A shift-invariant kernel is a kernel of the form $k(\x,\z) = k(\x-\z)$ where $k(\cdot)$ is a positive definite function
(we abuse notation by using $k$ to denote both the kernel and the defining positive definite function).

The underlying observation behind random Fourier features is a simple consequence of Bochner's Theorem:
for every shift-invariant kernel for which $k(\bs 0) = 1$ there is a probability measure $\mu_k(\cdot)$
and possibly a corresponding probability density function $p_k(\cdot)$, both on $\RR^d$, such that
\begin{align}
k(\x,\z) = \int_{\RR^d} e^{-2\pi i \veta^\T(\x-\z)} d\mu_k(\veta)
=\int_{\RR^d} e^{-2\pi i \veta^\T(\x-\z)} p_k(\veta) d\veta~.\label{eq:bochner}
\end{align}
In other words, the inverse Fourier transform of the kernel $k(\cdot)$ is a probability density function, $p_k(\cdot)$.
For simplicity we typically drop the $k$ subscript, writing $\mu(\cdot)= \mu_k(\cdot)$ and $p(\cdot) = p_k(\cdot)$, with the associated kernel function clear from context.
We remark that while it is not always the case that the probability measure $\mu_k(\cdot)$ has an associated density
function $p_k(\cdot)$, we assume the existence of a density function for the kernels we consider in this paper.

If $\veta_1,\dots,\veta_s$ are drawn according to
$p(\cdot)$, and we define $\varphi(\x) \equiv \frac{1}{\sqrt{s}}\left(e^{-2\pi i \veta_1^T\x}, \cdots, e^{-2\pi i \veta_s^T\x} \right)^\conj$, then it is not
hard to see that
\begin{equation*}
k(\x,\z) = \ExpectC{\varphi}{\varphi(\x)^\conj \varphi(\z)}\,.
\end{equation*}
The idea of the Random Fourier features method is then to define the substitute kernel:
\begin{equation}
\tilde{k}(\x,\z) \equiv \varphi(\x)^\conj \varphi(\z) = \frac{1}{s} \sum^s_{l=1} e^{-2\pi i \veta^\T_l(\x - \z)}
\end{equation}

To summarize, the density function $p(\cdot)$ is just the $d$-dimensional Fourier transform of the kernel $k(\cdot)$, and the random Fourier features
method approximates $k(\cdot)$ by sampling $s$ ($d$-dimensional) frequencies $\veta_1,...,\veta_s$ according to their weight in the Fourier transform. Note that in order for $p(\cdot)$ to be a proper probability density function (integrating to $1$) we must have $k(\bs 0) = 1$. We assume this without loss of generality, since any kernel can be scaled to satisfy this condition.

Now suppose that $\matZ \in \CC^{n\times s}$ is the matrix whose $j^{th}$ row is $\varphi(\x_j)^\conj$,
and let $\tilde{\matK}= \matZ\matZ^\conj$. $\tilde{\matK}$ is the kernel matrix corresponding to $\tilde{k}(\cdot,\cdot)$.
The resulting random Fourier features KRR estimator
is $\tilde{f}(\x) \equiv \sum^n_{j=1} \tilde{k}(\x_j, \x) \tilde{\alpha}_j$
where $\bs{\tilde{\alpha}}$ is the solution of $(\tilde{\matK} + \lambda
\matI_n)\bs{\tilde{\alpha}} = \y$. Typically, $s < n$ and we can represent
$\tilde{f}(\cdot)$ more efficiently as:
\begin{equation*}
\tilde{f}(\x) = \varphi(\x)^\conj \w
\end{equation*}
where
\begin{equation*}
\w = (\matZ^\conj\matZ + \lambda \matI_s)^{-1} \matZ^\conj \y
\end{equation*}
(this is a simple consequence of the Woodbury formula).
We can compute $\w$ in $O(ns^2)$ time, making random Fourier features computationally attractive if $s < n$.

\subsubsection{Modified Random Fourier Features}

While it seems to be a natural choice, there is no fundamental reason that we
must sample the frequencies $\veta_1,\dots,\veta_s$ using the Fourier transform
density function $p(\cdot)$. In fact, we will see that it is advantageous to use
a different sampling distribution based on the kernel leverage function (defined
later).
%% ICML17 version:
%There is actually no reason to insist on sampling $\veta_1,\dots,\veta_s$ using $p(\cdot)$, and, in fact, we shall see that it is advantageous to use a different sampling distribution.

Let $q(\cdot)$ be any probability density function whose support
includes that of $p(\cdot)$. If we sample $\veta_1,\dots,\veta_s$ using $q(\cdot)$,
and define
\begin{equation*}
\varphi(\x) \equiv \frac{1}{\sqrt{s}}\left(\sqrt{\frac{p(\veta_1)}{q(\veta_1)}}e^{-2\pi i \veta_1^T\x}, \cdots, \sqrt{\frac{p(\veta_s)}{q(\veta_s)}}e^{-2\pi i \veta_s^T\x} \right)^\conj
\end{equation*}
we still have $k(\x,\z) = \ExpectC{\varphi}{\varphi(\x)^\conj \varphi(\z)}$.
We refer to this method as {\em modified random Fourier features} and remark that it can be viewed as a form
of importance sampling.

\subsubsection{Additional Notations and Identities}

Now that we have defined (modified) random Fourier features, we can introduce some additional
notation and identities.
The $(j,l)$ entry of $\matZ$ is given by:
\begin{equation}
\label{eq:Z-def}
\matZ_{jl}=\frac{1}{\sqrt{s}}e^{-2 \pi i \x_j^\T \veta_l}\sqrt{p(\veta_l) / q(\veta_l)}.
\end{equation}
Let $\z : \RR^d \to \CC^n$ be defined by
\begin{equation*}
\z(\veta)_j = e^{-2 \pi i \x^\T_j \veta}~.
\end{equation*}
Note that column $l$ of $\matZ$ from the previous section is exactly $\z(\veta_l)\sqrt{p(\veta_l)/[s\cdot q(\veta_l)]}$.
So we have:
\begin{equation*}
\matZ \matZ^\conj = \frac{1}{s}\sum^s_{l=1} \frac{p(\veta_l)}{q(\veta_l)} \z(\veta_l)\z(\veta_l)^\conj.
\end{equation*}
Finally, by \eqref{eq:bochner} we have
$$\matK = \int_{\RR^d} \z(\veta)\z(\veta)^\conj d\mu(\veta) = \int_{\RR^d} \z(\veta)\z(\veta)^\conj p(\veta) d\veta\,.$$
and thus $\Expect{\matZ \matZ^\conj} = \matK$.

\subsection{Related Work}
\label{sec:related-work}
Rahimi and Recht's original analysis of random Fourier features~\cite{RahimiRecht07} bounded the point-wise distance between $k(\cdot,\cdot)$ and $\tilde{k}(\cdot,\cdot)$.

In follow-up work, they give learning rate bounds for a broad class
of estimators
using random Fourier features~\cite{RahimiRecht09}. However, their results do not apply to classic
KRR. %Additionally, as pointed out in~\cite{RudiEtAl16},
Furthermore, their main bound becomes relevant only when the number of sampled features is on
order of the training set size.

Rudi et al.~\cite{RudiEtAl16} prove generalization properties for KRR with random features, under
somewhat difficult to verify technical assumptions, some of which can be seen as
constraining the leverage function distribution that we study.
They leave open improving their bounds via a more refined sampling approach.
Bach~\cite{Bach15} analyzes random Fourier features from a function approximation point of view. He defines a similar leverage function distribution to the one that we consider, but leaves open establishing bounds on and effectively sampling from this distribution, both of which we address in this work. Finally,
Tropp~\cite{Tropp15} analyzes the distance between the kernel matrix and its approximation in terms of the
spectral norm, $\TNorm{\matK - \tilde{\matK}}$, which can be a significantly weaker error metric than \eqref{eq:approx-bound-intro}.

Outside of work on random Fourier features, risk inflation bounds for approximate KRR and leverage score sampling have been used to analyze and improve
the Nystr\"{o}m method for kernel approximation~\cite{Bach13,AlaouiMahoney15,rudi2015less,MuscoMusco16}. We apply a number of techniques from this line of work.

Spectral approximation bounds, such as~\eqref{eq:approx-bound-intro}, are quite popular in the
sketching literature; see Woodruff's survey~\cite{Woodruff14}. Most closely related to our work is
analysis of spectral approximation bounds without regularization (i.e. $\lambda = 0$) for the polynomial
kernel~\cite{ANW14}. Improved bounds with regularization (still for the polynomial kernel) were recently
proved by Avron et al.~\cite{ACW16}.

\section{Spectral Bounds and Statistical Guarantees}
\label{sec:approx-bounds}

Given a feature transformation, like random Fourier features, how do we analyze it and relate its use to
non-approximate methods? A common approach, taken for example in the original paper on random Fourier features~\cite{RahimiRecht07}, is to bound the difference
between the true kernel $k(\cdot, \cdot)$ and the approximate kernel $\tilde{k}(\cdot,\cdot)$.
%Indeed, this is
%exactly what is analyzed by~\cite{RahimiRecht07}.
However, it is unclear how such bounds translate to downstream guarantees on statistical learning
methods, such as KRR.
In this paper we advocate and focus on spectral approximation bounds on the
regularized kernel matrix, specifically, bounds of the form
\begin{equation}
\label{eq:approx-bound}
(1-\Delta) (\matK + \lambda \matI_n) \preceq \matZ \matZ^\conj + \lambda \matI_n \preceq (1 + \Delta) (\matK + \lambda \matI_n)
\end{equation}
for some $\Delta < 1$.
\begin{defn}
We say that a matrix $\matA$ is a {\em $\Delta$-spectral approximation} of another matrix $\matB$,
if $(1-\Delta)\matB \preceq \matA \preceq (1+\Delta)\matB$.
\end{defn}
\begin{rem}
When $\lambda=0$, bounds of the form of~\eqref{eq:approx-bound} can be viewed as a low-distortion subspace embedding bounds.
Indeed, when $\lambda=0$ it follows from~\eqref{eq:approx-bound} that $\spn{k(\x_1,\cdot),\dots,k(\x_n,\cdot)}\subseteq \H_k$ can be
embedded with $\Delta$-distortion in $\spn{\varphi(\x_1),\dots,\varphi(\x_n)}\subseteq \RR^s$.
\end{rem}

The main mathematical question we seek to address in this paper is: when using random Fourier features,
how large should $s$ be in order to guarantee that $\matZ \matZ^\conj + \lambda \matI_n$ is a $\Delta$-spectral approximation
of $\matK + \lambda \matI_n$? To motivate this question, in the following two subsections we show that such bounds can be used
to derive risk inflation bounds for approximate kernel ridge regression. We also show that they can be used to analyze the use
of $\matZ \matZ^\conj + \lambda \matI_n$ as a preconditioner for $\matK + \lambda \matI_n$.

While this paper focuses on KRR for conciseness, we remark that in the sketching literature, spectral approximation bounds also form the basis for analyzing
sketching based methods for tasks like low-rank approximation, k-means and more. In the kernel setting, such bounds
where analyzed, without regularization, for the polynomial kernel~\cite{ANW14}. Cohen et al.~\cite{CohenMuscoMusco17} recently showed that \eqref{eq:approx-bound} along with a trace condition on $\matZ\matZ^\conj$ (which holds for all sampling approaches we consider) yields a so called ``projection-cost preservation'' condition for the kernel approximation. With $\lambda$ chosen appropriately, this condition ensures that $\matZ\matZ^\conj$ can be used in place of $\matK$ for approximately solving kernel k-means clustering and for certain versions of kernel PCA and kernel CCA. See Musco and Musco~\cite{MuscoMusco16} for details, where this analysis is carried out for the Nystr\"{o}m method.

\subsection{Risk Bounds}
\label{sec:risk-bounds}
%As explained earlier, the classical use of random Fourier features is as a means for replacing
%the kernel $k(\cdot,\cdot)$ with the random features approximated kernel $\tilde{k}(\cdot, \cdot)$.
%This amounts to replacing the kernel matrix $\matK$ with $\tilde{\matK}=\matZ \matZ^\conj$ when
%computing $\valpha$ (although from a computational point of view the resulting model and better computed
%and represented as $\x \mapsto \varphi(\x)^\conj (\matZ \matZ^\conj + \lambda \matI_s)^{-1} \matZ^\conj \y$).

One way to analyze estimators is via risk bounds; several recent papers on approximate KRR employ
such an analysis~\cite{Bach13, AlaouiMahoney15,MuscoMusco16}.
In particular, these papers
consider the fixed design setting and seek to bound the expected in-sample predication error of
the KRR estimator $\bar{f}$,
viewing it as an empirical estimate of the statistical risk. More specifically, the underlying
assumption is that $y_i$ satisfies
\begin{equation}
\label{eq:stat-model}
y_i = f^\star(\x_i) + \nu_i
\end{equation}
for some $f^\star:\X \to \RR$. The $\{\nu_i\}$'s are i.i.d noise terms, distributed as normal variables
with variance $\sigma^2_\nu$. The empirical risk of an estimator $f$, which can be viewed as
a measure of the quality of the estimator, is
$$
\R(f) \equiv \ExpectC{\{\nu_i\}}{\frac{1}{n}\sum^n_{j=1}(f(\x_i) - f^\star(\x_i))^2}
$$
(note that $f$ itself might be a function of $\{\nu_i\}$).

Let $\f \in \RR^n$ be the vector whose $j^{th}$ entry is $f^\star(\x_j)$.
It is quite straightforward to show that for the KRR estimator $\bar{f}$ we have~\cite{Bach13, AlaouiMahoney15}:
\begin{equation*}
\R(\bar{f}) = n^{-1}\lambda^2 \f^\T (\matK + \lambda \matI_n)^{-2} \f + n^{-1}\sigma_\nu^2 \Trace{\matK^2(\matK + \lambda\matI_n)^{-2}}.
\end{equation*}
Since $\lambda^2 \f^\T (\matK + \lambda \matI_n)^{-2} \f \leq  \lambda \f^\T (\matK + \lambda \matI_n)^{-1} \f$
and $\Trace{\matK^2(\matK + \lambda\matI_n)^{-2}} \leq \Trace{\matK(\matK + \lambda\matI_n)^{-1}} = s_\lambda(\matK)$,
we define
$$\wh{\R}_{\matK}(\f) \equiv  n^{-1}\lambda \f^\T (\matK + \lambda \matI_n)^{-1} \f + n^{-1}\sigma_\nu^2 s_\lambda(\matK)$$
and note that $\R(\bar{f}) \leq \wh{\R}_{\matK}(\f)$. The first term in the above expressions
for $\R(\bar{f})$ and $\wh{\R}_{\matK}(\f)$ is frequently referred to as the bias term,
while the second is the variance term.

\begin{lem}
\label{lem:risk-bound}
Suppose that~\eqref{eq:stat-model} holds, and let $\f \in \RR^n$ be the vector whose $j^{th}$ entry is $f^\star(\x_j)$.
Let $\bar{f}$ be the KRR estimator, and let $\tilde{f}$  be KRR estimator
obtained using some other kernel $\tilde{k}(\cdot,\cdot)$ whose kernel matrix is $\tilde{\matK}$.
Suppose that $\tilde{\matK} + \lambda\matI_n$ is a $\Delta$-spectral approximation to $\matK + \lambda\matI_n$ for some $\Delta<1$,
and that $\TNorm{\matK} \geq 1$. The following bound holds:
\begin{align}\label{estimateUpperBound}
\R(\tilde{f}) \leq (1-\Delta)^{-1} \wh{\R}_{\matK}(\f) + \frac{\Delta}{(1+\Delta)}\cdot\frac{\rank{\tilde{\matK}}}{n}\cdot\sigma^2_\nu
\end{align}
%\begin{align}\label{trueRiskUpperBound}
%\R(\tilde{f}) \leq (1 + 2\Delta\TNorm{\matK}/\lambda)^2 \R(\bar{f}).
%\end{align}
\end{lem}

\begin{proof}
\label{appendix:lem:risk-bound}
%{\bf First bound \eqref{estimateUpperBound}.}
Note that $\matA \preceq \matB$ implies that $\matB^{-1} \preceq \matA^{-1}$ so for the bias term we have:
\begin{align}\label{easyBiasBound}
\f^\T (\tilde{\matK} + \lambda \matI_n)^{-1} \f \leq (1-\Delta)^{-1}\f^\T (\matK + \lambda \matI_n)^{-1} \f.
\end{align}

We now consider the variance term. Denote $s = \rank{\tilde{\matK}}$, and let
$\lambda_1(\matA) \geq \lambda_2(\matA) \geq \dots \geq \lambda_n(\matA)$ denote the eigenvalues
of a matrix $\matA$. We have:
\begin{align*}
s_\lambda(\tilde{\matK})  = \Trace{(\tilde{\matK} + \lambda \matI_n)^{-1}\tilde{\matK}} &= \sum^s_{i=1}\frac{\lambda_i(\tilde{\matK})}{\lambda_i(\tilde{\matK}) + \lambda} \\
&= s -  \sum^s_{i=1}\frac{\lambda}{\lambda_i(\tilde{\matK}) + \lambda} \\
&\leq s - (1+\Delta)^{-1} \sum^s_{i=1}\frac{\lambda}{\lambda_i(\matK) + \lambda} \\
 &= s -  \sum^s_{i=1}\frac{\lambda}{\lambda_i(\matK) + \lambda} + \frac{\Delta}{1+\Delta}\sum^s_{i=1}\frac{\lambda}{\lambda_i(\matK) + \lambda} \\
 &\leq n -  \sum^n_{i=1}\frac{\lambda}{\lambda_i(\matK) + \lambda} + \frac{\Delta\cdot s}{1+\Delta} \\
 &= s_\lambda(\matK) + \frac{\Delta\cdot s}{1+\Delta} \\
 &\leq (1-\Delta)^{-1}s_\lambda(\matK) + \frac{\Delta\cdot s}{1+\Delta}
\end{align*}
where we use the fact that $\bv{A} \preceq \bv{B}$ implies that $\lambda_i(\matA) \leq \lambda_i(\matB)$
(this is a simple consequence of the Courant-Fischer minimax theorem).

Combining the above variance bound with the bias bound in \eqref{easyBiasBound} yields:
$$\wh{\R}_{\tilde{\matK}}(\f) \leq (1-\Delta)^{-1} \wh{\R}_{\matK}(\f) + \frac{\Delta}{(1+\Delta)}\cdot\frac{\rank{\tilde{\matK}}}{n}\cdot\sigma^2_\nu$$
and the bound $\R(\tilde{f}) \leq \wh{\R}_{\tilde{\matK}}(\f)$ completes the
proof.
\end{proof}

In short, Lemma \ref{lem:risk-bound} bounds the risk of the approximate KRR estimator as a function of both the risk upper bound $\wh{\R}_{\matK}(\f)$ and an additive term which is small if $\rank{\tilde{\matK}}$ and/or $\Delta$ is small. In particular, it is instructive to compare the additive term $(\Delta/(1+\Delta))n^{-1}\sigma_\nu^2 \cdot \rank{\tilde{\matK}}$ to the variance term $n^{-1}\sigma_\nu^2 \cdot s_\lambda(\matK)$.
\begin{rem}
An approximation $\tilde{\matK}$ is only useful computationally if $\rank{\tilde{\matK}} \ll n$ so $\bv{\tilde K}$ gives a significantly compressed approximation to the original kernel matrix.
%(otherwise, there is no
%computational advantage in replacing $\matK$ with $\tilde{\matK}$).
Ideally we should have
$\rank{\tilde{\matK}}/n \to 0$ as $n\to\infty$ and so the additive term in \eqref{estimateUpperBound} will also approach $0$ and generally be small when $n$ is large.
\end{rem}

\subsection{Random Features Preconditioning}
Suppose we choose to solve $(\matK + \lambda \matI_n)\valpha=\y$ using an iterative method (e.g. CG).
In this case, we can apply $\matZ \matZ^\conj + \lambda \matI_n$ as a preconditioner. Using standard
analysis of Krylov-subspace iterative methods it is immediate that if $\matZ\matZ^\conj + \lambda\matI_n$ is a
$\Delta$-spectral approximation of $\matK + \lambda \matI_n$ then
the number of iterations until convergence is $O(\sqrt{(1+\Delta)/(1-\Delta))})$. Thus, if $\matZ\matZ^\conj + \lambda\matI_n$ is, say,
a $1/2$-spectral approximation of $\matK + \lambda \matI_n$, then the number of iterations is bounded by a constant. The preconditioner
can be efficiently applied (after preprocessing) via the Woodbury formula, giving cost per iteration (if $s \leq n$)
of $O(n^2)$. The overall cost of computing the KRR estimator is therefore $O(ns^2 + n^2)$. Thus, as long as $s=o(n)$ this approach gives an advantage over direct methods which cost $O(n^3)$. For small $s$ it also beats non-preconditioned iterative methods cost $O(n^2 \sqrt{\kappa(\matK)})$. See Cutajar et al.~\cite{CutajarEtAl16} and Avron et al.~\cite{ACW16} for a detailed discussion. The upshot though is that we reach again the question that was poised earlier: how big should $s$ be so that $\matZ\matZ^\conj + \lambda\matI_n$ is a
$1/2$-spectral approximation of $\matK + \lambda \matI_n$?

\section{Ridge Leverage Function Sampling and Random Fourier Features}
\label{sec:ridge-and-rf}

In this section we present upper bounds on the number of random Fourier features needed to
guarantee that $\matZ\matZ^\conj + \lambda \matI_n$ is a $\Delta$-spectral approximation
to $\matK + \lambda \matI_n$. Our bounds apply to {\em any}
shift-invariant kernel
and a wide range of feature sampling distributions (in particular, classical random Fourier
features).

Our analysis is based on relating the sampling density to an appropriately defined \emph{ridge leverage function}. This function is a continuous generalization
of the popular leverage scores~\cite{MahoneyDrineas09} and ridge leverage scores~\cite{AlaouiMahoney15,CohenMuscoMusco17} used in the analysis of linear methods.
Bach~\cite{Bach15} defined the leverage function of the integral operator given by the kernel
function and the data distribution. For our purposes, a more appropriate definition
is with respect to a fixed input dataset:
\begin{defn}
For $\x_1,\dots,\x_n$ and shift-invariant kernel $k(\cdot,\cdot)$, define
the {\em ridge leverage function} as
\begin{equation*}
\tau_\lambda(\veta) \equiv p(\veta) \z(\veta)^\conj(\matK + \lambda \bv{I}_n)^{-1} \z(\veta)\,.
\end{equation*}
In the above, $\matK$ is the kernel matrix and $p(\cdot)$ is the distribution given by the inverse Fourier transform of
$k(\cdot,\cdot)$.
\end{defn}
We begin with two simple propositions. Recall that we assume $k(\x,\x) = k(\bv{0}) = 1$ for any $\bv{x}$, however our results apply to general shift invariant kernel after appropriate scaling.
\begin{prop}
\label{prop:simple-tau-bound} For all $\veta$,
$$
p(\veta)n/(n+\lambda) \leq \tau_\lambda(\veta) \leq p(\veta)n / \lambda.\
$$
\end{prop}
\begin{proof}
Since $k$ is positive definite and $k(\bv{0}) = 1$, $|k(\x,\z)| \leq 1$ for all $\x$ and $\z$. This implies that
the maximum eigenvalue of $\matK$ is bounded by $n$. The lower bound follows, after noting that $\TNormS{\z(\veta)} = n$.
The upper bound follows similarly, since all eigenvalues
of $\matK + \lambda \matI_n$ are lower bounded by $\lambda$.
\end{proof}
\begin{prop}
\label{prop:simple-slambda-bound}
$$
\int_{\RR^d} \tau_\lambda(\veta) d\veta = s_\lambda(\matK).
$$
\end{prop}
\begin{proof}
\begin{eqnarray*}
\int_{\RR^d} \tau_\lambda(\veta) d\veta & = & \int_{\RR^d} p(\veta) \z(\veta)^\conj(\matK + \lambda \matI_n)^{-1} \z(\veta) d\veta \\
& = & \int_{\RR^d}  \Trace{p(\veta)(\matK + \lambda \matI_n)^{-1} \z(\veta) \z(\veta)^\conj} d\veta \\
& = & \Trace{ \int_{\RR^d} p(\veta)(\matK + \lambda \matI_n)^{-1} \z(\veta) \z(\veta)^\conj d\veta } \\
& = & \Trace{ (\matK + \lambda \matI_n)^{-1} \int_{\RR^d} p(\veta) \z(\veta) \z(\veta)^\conj d\veta } \\
& = & \Trace{ (\matK + \lambda \matI_n)^{-1} \matK } = s_\lambda(\matK)\,.
\end{eqnarray*}
The second and third equalities follow from the cyclic property and linearity of the trace respectively.
\end{proof}

Recall that we denote the ratio $n/\lambda$, which appears frequently in our
analysis, by $n_\lambda = n / \lambda$. As discussed, theoretical bounds
generally set $\lambda = \omega(1)$ (as a function of $n$) so $n_\lambda =
o(n)$. However we remark that in practice, it may sometimes be the case that
$\lambda$ is very small and $n_\lambda \gg n$.

An immediate result of Propositions \ref{prop:simple-tau-bound} and \ref{prop:simple-slambda-bound} (which can also be obtained algebraically from $\matK$) is a generic bound on statistical dimension:
\begin{cor}
For any $\matK$, $\s_\lambda(\matK) \leq n_\lambda$.
\end{cor}

For any shift-invariant kernel with $k(\x,\x)=1$ and $k(\x,\z)\to 0$ as $\TNorm{\x -\z}\to\infty$ (e.g., the Gaussian kernel) if we allow points to be arbitrarily spread out, the kernel matrix converges to the identity matrix, and $s_\lambda(\matI_n) = n/(1+\lambda) = \Omega(n_\lambda)$ if $\lambda=\Omega(1)$ so the above bound is tight.
However, this requires datasets of increasingly large diameter (as $n$ grows). In contrast, the usual assumption
in statistical learning is that the data is sampled from a bounded domain $\X$.
In~\S\ref{sec:stat-bound} we show via a leverage function upper bound that for the important Gaussian kernel, for bounded datasets we have $s_\lambda(\matK) = o(n_\lambda)$.

In the matrix sketching literature it is well known that spectral approximation bounds similar to~\eqref{eq:approx-bound}
can be constructed by sampling columns relative to upper bounds on the leverage
scores. In the following, we generalize this for the case of sampling Fourier
features from a continuous domain. First, we need an auxiliary lemma.

\begin{lem}
	\label{lem:tropp-correct}
	Let $\matB$ be a fixed $d_1 \times d_2$ matrix. Construct a $d_1 \times d_2$ random matrix $\matR$ that satisfies
	$$\Expect{\matR} = \matB~~~~\textrm{and}~~~~\TNorm{\matR} \leq L.$$
	Let $\matM_1$ and $\matM_2$ be semidefinite upper bounds for the expected squares:
	$$\Expect{\matR \matR^\conj} \preceq \matM_1~~~~\textrm{and}~~~~\Expect{\matR^\conj \matR} \preceq \matM_2.$$
	Define the quantities
	$$
	m = \max(\TNorm{\matM_1}, \TNorm{\matM_2})~~~~\textrm{and}~~~~d = (\Trace{\matM_1} + \Trace{\matM_2}) / m.
	$$
	Form the matrix sampling estimator
	$$
	\bar{\matR}_n = \frac{1}{n}\sum_{k=1}^{n} \matR_k
	$$
	where each $\matR_k$ is an independent copy of $\matR$.
	Then, for all $t \geq \sqrt{m/n} + 2L/3n$,
	\begin{equation}
	\Pr(\TNorm{\bar{\matR}_n - \matB} \geq t) \leq 4d\exp\left( \frac{-nt^2/2}{m + 2Lt/3} \right).\nonumber
	\end{equation}
	
\end{lem}
The proof of Lemma~\ref{lem:tropp-correct}, which is essentially a restatement of Corollary 7.3.3 from~\cite{Tropp15} with slightly improved requirements, appears in appendix \ref{sec:matrix-concentration}
\begin{lem}
\label{lem:leverage-sampling}
Let $\tilde{\tau}:\RR^d \to \RR$ be a measurable function such that $\tilde{\tau}(\veta) \geq \tau_\lambda(\veta)$ for all $\veta \in \RR^d$,
 and furthermore assume that
$$
s_{\tilde{\tau}} \equiv \int_{\RR^d} \tilde{\tau}(\veta) d\veta
$$
is finite. Denote $p_{\tilde{\tau}}(\veta) = \tilde{\tau}(\veta) / s_{\tilde{\tau}}$. Let $\Delta \leq 1/2$ and $\rho\in(0,1)$.
Assume that $\|\matK\|_2\geq \lambda$.
% and that $s_{\tilde{\tau}} \geq 2$.
Suppose we take
$s \geq \frac{8}{3}\Delta^{-2}s_{\tilde{\tau}}\ln(16s_\lambda(\matK)/\rho)$ samples $\veta_1,\dots,\veta_s$ from
the distribution associated with the density $p_{\tilde{\tau}}(\cdot)$ and then construct the matrix $\matZ$ according to~\eqref{eq:Z-def} with $q=p_{\tilde{\tau}}$.
Then $\matZ\matZ^\conj + \lambda \matI_n$ is $\Delta$-spectral approximation of $\matK+\lambda\matI_n$ with probability of at least $1-\rho$.
\end{lem}
\begin{proof}

Let $\matK + \lambda \matI_n = \matV^\T \matSigma^2 \matV$ be an eigendecomposition of $\matK+\lambda \matI_n$.
Note that the $\Delta$-spectral approximation guarantee \eqref{eq:approx-bound-intro} is equivalent to
$$
\matK - \Delta (\matK + \lambda \matI_n) \preceq \matZ \matZ^\conj \preceq \matK + \Delta (\matK + \lambda \matI_n)\,,
$$
so by multiplying by $\matSigma^{-1} \matV$ on the left and $\matV^\T \matSigma^{-1}$ on the right we find that it
suffices to show that
\begin{equation}
\label{eq:norm-bound-over}
\TNorm{ \matSigma^{-1} \matV \matZ \matZ^\conj \matV^\T \matSigma^{-1}- \matSigma^{-1} \matV \matK \matV^\T \matSigma^{-1} } \leq \Delta
\end{equation}
holds with probability of at least $1-\rho$.
Let
$$
\matY_l = \frac{p(\veta_l)}{p_{\tilde{\tau}}(\veta_l)} \matSigma^{-1} \matV \z(\veta_l) \z(\veta_l)^\conj \matV^\T \matSigma^{-1}\,.
$$
Note that $\Expect{\matY_l} = \matSigma^{-1} \matV \matK \matV^\T \matSigma^{-1}$ and $\frac{1}{s}\sum^s_{l=1}\matY_l = \matSigma^{-1} \matV \matZ \matZ^\conj \matV^\T \matSigma^{-1}$.
Thus, we can use matrix concentration results to prove~\eqref{eq:norm-bound-over}.

To apply this bound we need to bound the norm of $\matY_l$ and the stable rank $\Expect{\matY^2_l}$. Since $\matY_l$ is always a rank one matrix we have
\begin{eqnarray*}
\TNorm{\matY_l} & = & \frac{p(\veta_l)}{p_{\tilde{\tau}}(\veta_l)}\Trace{\matSigma^{-1} \matV \z(\veta_l) \z(\veta_l)^\conj \matV^\T \matSigma^{-1}} \\
  & = & \frac{p(\veta_l)}{p_{\tilde{\tau}}(\veta_l)} \z(\veta_l)^\conj \matV^\T \matSigma^{-1} \matSigma^{-1} \matV \z(\veta_l) \\
  & = & \frac{p(\veta_l)}{p_{\tilde{\tau}}(\veta_l)} \z(\veta_l)^\conj (\matK + \lambda \matI_n)^{-1} \z(\veta_l) \\
  & = & \frac{s_{\tilde{\tau}}\cdot \tau_\lambda(\veta_l)}{\tilde{\tau}(\veta_l)} \leq s_{\tilde{\tau}}
\end{eqnarray*}
since $\tilde{\tau}_\lambda(\veta_l) \ge \tau(\veta_l)$ by assumption of the lemma.
We also have
\begin{eqnarray*}
  \matY^2_l & = & \frac{p(\veta_l)^2}{p_{\tilde{\tau}}(\veta_l)^2}\matSigma^{-1} \matV \z(\veta_l) \z(\veta_l)^\conj \matV^\T \matSigma^{-1} \matSigma^{-1} \matV \z(\veta_l) \z(\veta_l)^\conj \matV^\T \matSigma^{-1} \\
& = &  \frac{p(\veta_l)^2}{p_{\tilde{\tau}}(\veta_l)^2}\matSigma^{-1} \matV \z(\veta_l) \z(\veta_l)^\conj (\matK + \lambda\matI_n)^{-1} \z(\veta) \z(\veta_l)^\conj \matV^\T \matSigma^{-1} \\
& = &  \frac{p(\veta_l)\tau(\veta_l)}{p_{\tilde{\tau}}(\veta_l)^2}\matSigma^{-1} \matV \z(\veta_l) \z(\veta_l)^\conj \matV^\T \matSigma^{-1} \\
& = &  \frac{\tau(\veta_l)}{p_{\tilde{\tau}}(\veta_l)}\matY_l \\
& = & \frac{s_{\tilde{\tau}}\tau(\veta_l)}{\tilde{\tau}(\veta_l)}\matY_l \preceq s_{\tilde{\tau}} \matY_l.
\end{eqnarray*}
Let $\lambda_1\geq \dots\geq\lambda_n$ be the eigenvalues of $\matK$. We have
\begin{eqnarray*}
\Expect{s_{\tilde{\tau}} \matY_l} & = & s_{\tilde{\tau}} \matSigma^{-1} \matV \matK \matV^\T \matSigma^{-1} \\
  & = & s_{\tilde{\tau}} \left(\matI_n - \lambda \matSigma^{-2} \right) \\
& = & s_{\tilde{\tau}}\cdot \diag{\lambda_1/(\lambda_1 + \lambda),\dots,\lambda_n/(\lambda_n + \lambda)} := \matD\,.
\end{eqnarray*}
So,
\begin{eqnarray*}
\Pr\left(\LTNorm{\frac{1}{s}\sum^s_{l=1}\matY_l - \matSigma^{-1} \matV \matK \matV^\T \matSigma^{-1}}  \geq \Delta \right) & \leq &
\frac{8\Trace{\matD}}{\TNorm{\matD}}\exp\left(\frac{-s\Delta^2/2}{\TNorm{\matD} + 2s_{\tilde{\tau}} \Delta /3}\right) \\
& \leq & 8\frac{s_{\tilde{\tau}} \cdot s_\lambda(\matK)}{\lambda_1/(\lambda_1 + \lambda)} \exp\left(\frac{-s\Delta^2}{2s_{\tilde{\tau}}(1 + 2\Delta /3)}\right) \\
& \leq & 16s_\lambda(\matK) \exp\left(\frac{-s\Delta^2}{2s_{\tilde{\tau}}(1 + 2\Delta /3)}\right) \\
& \leq & 16s_\lambda(\matK) \exp\left(\frac{-3s\Delta^2}{8s_{\tilde{\tau}}}\right) \leq \rho
\end{eqnarray*}
where the third inequality is due to the assumption that $\lambda_1 = \|\matK\|_2\geq \lambda$ and the last inequality is due to the bound on $s$.
\end{proof}

Lemma \ref{lem:leverage-sampling} shows that if we could sample using the ridge leverage function, then $O(s_\lambda(\matK) \log(s_\lambda(\matK)))$ samples suffice for spectral approximation of $\matK$
(for a fixed $\Delta$ and failure probability). While there is no straightforward way to perform this sampling, we can consider how well the classic random Fourier features sampling distribution approximates the leverage function, obtaining a bound on its performance:
%
%In contrast, how many samples do we need when sampling using $p(\cdot)$, i.e. the classical random Fourier features method?
%The following theorem, which is one of our main contributions, establishes a bound.
\begin{thm}
\label{thm:rf-bound}
Let $\Delta \leq 1/2$ and $\rho\in(0,1)$. Assume that $\|\matK\|_2\geq \lambda$.
If we use $s\geq \frac{8}{3}\Delta^{-2}n_\lambda\ln(16s_\lambda(\matK)/\rho)$ random Fourier features (i.e., sampled according to $p(\cdot)$),
then $\matZ\matZ^\conj + \lambda \matI_n$ is $\Delta$-spectral approximation of $\matK+\lambda\matI_n$ with probability of at least $1-\rho$.
\end{thm}
\begin{proof}
Define $\tilde{\tau}(\veta) = p(\veta) \cdot n_\lambda$ and note that $\tilde{\tau}(\veta) \geq \tau_\lambda(\veta)$ by Proposition \ref{prop:simple-tau-bound} and that $s_{\tilde{\tau}} = n_\lambda$.
Finally, note that $p_{\tilde{\tau}}(\veta) = p(\veta)$, the classic Fourier features sampling probability.
\end{proof}

Theorem \ref{thm:rf-bound} establishes that if $\lambda = \omega(\log(n))$ and $\Delta$ is fixed, $o(n)$ random Fourier features suffice for spectral approximation, and so the method can provably speed up KRR. Nevertheless, the bound depends on $n_\lambda$ instead of $s_\lambda(\matK)$, as is possible with true leverage function
sampling (see Lemma \ref{lem:leverage-sampling}). This gap arises from our use
of the simple, often loose, leverage function upper bound given by Proposition \ref{prop:simple-tau-bound}.

Unfortunately, the bound in Theorem \ref{thm:rf-bound} cannot be improved. Even for the special case of a one-dimensional Gaussian kernel, the classic random Fourier features sampling distribution is far enough from the ridge leverage distribution that $\Omega(n_\lambda)$ features may be needed even when $s_\lambda(\matK) = o(n_\lambda)$. On the otherhand, a simple modified sampling approach \emph{does} closely approximate the true ridge leverage distribution and so yields significantly better bounds for the Gaussian kernel. We present these results in \S\ref{sec:lower} and \S\ref{sec:improved} respectively. We defer a discussion of their proofs to \S\ref{sec:actualBounds}, where we develop our main technical contribution: a sharper understanding of the ridge leverage function based on a formulation as the solution to two dual optimization problems which give corresponding upper and lower bounds on the distribution and, correspondingly, on sampling performance.

\section{Lower Bound for Classic Random Fourier Features}
\label{sec:lower}
Our lower bound shows that
the upper bound of Theorem \ref{thm:rf-bound} on the number of samples required by classic random Fourier features to obtain
a spectral approximation to $\matK+\lambda \matI_n$ is essentially best possible. The full proof is given in Appendix~\ref{appendix:rr-sampling}.
% for the popular Gaussian kernel
%even on datasets whose diameter grows very slowly.
\begin{thm}\label{thm:rr-samples}
	Consider the $d$-dimensional Gaussian kernel with $\sigma = (2\pi)^{-1}$ (so $p(\veta) =
	(2\pi)^{-d/2}e^{-\TNormS{\veta}/2}$).
	Suppose that $n\geq 17$ is any odd integer such that $m=n^{1/d} \geq \max(64\log n_\lambda, 3)$ is integer. Further, assume that $1\leq d \leq \frac{2\log n}{5\log\log n}$. For any $\lambda$ satisfying $\frac{10}{n} \leq \lambda \leq \min\left\{\left(\frac{1}{2}\right)^{2d} \cdot
	\frac{n}{1024}, n^{1-\frac{1}{128}}\right\}$, and every radius $\rad$ such that $2000 \log n_\lambda \leq \rad
	\leq \frac{n^{1/d}}{800\sqrt{\log(n_\lambda)}}$, there exists a dataset of $n$
	points $\{\x_j\}_{j=1}^n\subseteq [-\rad,\rad]^d$ such that if $s$ random
	Fourier features (i.e., sampled according to $p(\cdot)$) are sampled for some $s$ satisfying $s \leq
	\frac{n_\lambda}{13\cdot 2^{2d+4}}$, then with
	probability at least $0.5$, there exists a vector $\valpha\in \RR^n$ such that
	\begin{equation}\label{eq:rr-sampling}
	\valpha^\T (\matK+ \lambda \matI_n)\valpha < \frac{2}{3}\valpha^\T (\matZ\matZ^\conj+\lambda\matI_n)\valpha.
	\end{equation}
	Furthermore, for the said dataset is a uniformly spaced grid in $d$ dimensions, with $m$ points per dimension, and we have $s_\lambda(\matK) = O(\rad \cdot \poly{\log n_\lambda})$.
\end{thm}

\begin{rem}
Theorem~\ref{thm:rr-samples} gives a lower bound of $s=\Omega(n_\lambda / 2^{O(d)})$. However, since a lower dimensional dataset can be embedded in an higher dimension without affecting the kernel matrix or its approximation by adding zero coordinates, the stronger bound of $s=\Omega(n_\lambda)$ also holds. Nevertheless, we state a weaker version of the theorem since the certificate dataset is a uniform grid in $d$ dimensions (and not a one dimensional dataset embedded in an higher dimension).
\end{rem}	

Theorem~\ref{thm:rr-samples} shows that the number of samples $s$ required for $\matZ\matZ^\conj + \lambda \matI_n$ to be  a $1/2$-spectral
approximation to $\matK + \lambda \matI_n$ for a bounded dataset of points must depend at least
linearly on $n_\lambda$. So there is
an asymptotic gap between what is achieved with classical random Fourier features and what is achieved by modified
random Fourier features using leverage function sampling.

As we will see in \S\ref{sec:actualBounds}, the key idea behind the proof of Theorem \ref{thm:rr-samples} is to show that for a dataset contained in $[-\rad,\rad]^d$,
the ridge leverage function is large on a range of low frequencies. In contrast, the classic random Fourier features
distribution is very small at the edges of this frequency range, and so significantly undersamples some frequencies and
does not achieve spectral approximation.

We remark that it would have been preferable if Theorem \ref{thm:rr-samples} applied to bounded datasets (i.e. with $\rad$ fixed),
as the usual assumption in statistical learning theory is that data is sampled from a bounded domain. However, our
current techniques are unable to address this scenario. Nevertheless, our analysis allows $\rad$ to grow very slowly with $n$ and we conjecture that the upper bound is tight even for bounded domains.

\section{Improved Sampling for the Gaussian Kernel}
\label{sec:improved}

Contrasting with the lower bound of Theorem \ref{thm:rr-samples}, we now give a modified Fourier feature sampling distribution that does perform well for the Gaussian kernel on bounded input sets. Furthermore, unlike the true ridge leverage function, this distribution is simple and efficient to sample from.  To reduce clutter, we state the result for a fixed bandwidth $\sigma=(2\pi)^{-1}$. This is without loss of generality since we can rescale the points by $(2\pi\sigma)^{-1}$ and adjust the bounding interval.

Our modified distribution essentially corrects the classic distribution by ``capping'' the probability of sampling low frequencies near the origin. This allows it to allocate more samples to higher frequencies, which are undersampled by classical random Fourier features. See Figure \ref{distCompared} for a visual comparison of the two distributions.

%We now use the previous result to propose a new sampling distribution for the Gaussian kernel with
%$\sigma=(2\pi)^{-1}$ on datasets on a bounded interval $[-\rad, \rad]$ (again noting we can handle any $\sigma$ simply by scaling the datapoints).
%The distribution is specified using an upper-bound $\bar{\tau}_\rad$ on $\tau_\lambda$:
\begin{defn}[Improved Fourier Feature Distribution for the Gaussian Kernel] Define the function
	\begin{equation*}
	\bar{\tau}_\rad(\veta) \equiv \left\{
	\begin{array}{cc}
	\Big(12.4 \max(\rad, 2000 \log^{1.5} n_\lambda) \Big)^{d}+1 &\|\veta\|_\infty \leq 10\sqrt{\log(n_\lambda)}\\
	n_\lambda p(\veta) \prod^d_{j=1}\max(1,|\eta_j |)&\text{otherwise}\\
	\end{array}\right.
	\end{equation*}
	Let $s_{\bar{\tau}_\rad} = \int_{\mathbb{R}} \bar{\tau}_\rad(\veta) d\veta$ and define the probability density function $\bar{p}_\rad(\veta) = \bar{\tau}_\rad(\veta) / s_{\bar{\tau}_\rad}$.
\end{defn}

%% OLD VERSION:
\if0
\begin{defn}[Improved Fourier Feature Distribution for the Gaussian Kernel] Define the function
	\begin{equation*}
	\bar{\tau}_\rad(\veta) \equiv \left\{
	\begin{array}{cc}
	\Big(12.4 \max(\rad, 2000 \log^{1.5} n_\lambda) \Big)^{d}+1 &\|\veta\|_\infty \leq 10\sqrt{\log(n_\lambda)}\\
	p(\veta)n_\lambda&\text{otherwise}\\
	\end{array}\right.
	\end{equation*}
	Let $s_{\bar{\tau}_\rad} = \int_{\mathbb{R}} \bar{\tau}_\rad(\veta) d\veta$ and define the probability density function $\bar{p}_\rad(\veta) = \bar{\tau}_\rad(\veta) / s_{\bar{\tau}_\rad}$.
\end{defn}
\fi

Note that $\bar{p}_\rad(\veta)$ is just the uniform distribution for low frequencies with $\|\veta\|_\infty \le 10 \sqrt{\log(n_\lambda)}$, and a slightly modified classic Fourier features distribution, appropriately scaled, outside this range. As we show in \S\ref{sec:actualBounds}, $\bar{\tau}_\rad(\veta)$ upper bounds the true ridge leverage function $\tau_\lambda(\veta)$ for all $\veta$. Hence, simply applying Lemma \ref{lem:leverage-sampling}:

\begin{thm}
	\label{thm:improvedBound}
	Consider the d-dimensional Gaussian kernel with $\sigma = (2\pi)^{-1}$ (so $p(\veta) = (2\pi)^{-d/2} e^{-\|\veta\|_2^2/2}$) and any dataset of $n$ points $\{\x_j\}_{j=1}^n\subseteq\RR^d$ contained in a $\ell_\infty$-ball of radius $\rad$ (i.e $\InfNorm{\x_i -\x_j} \leq 2\rad$ for all $i,j\in[n]$). Suppose that $d \le 5\log(n_\lambda) + 1$.
	If we sample $s\geq \frac{8}{3}\Delta^{-2} s_{\bar{\tau}_\rad}\ln(16 s_\lambda(\matK) /\rho)$ random Fourier features according to $\bar{p}_\rad(\cdot)$ and construct $\bv{Z}$ according to~\eqref{eq:Z-def}, then with probability at least $1-\rho$, $\matZ\matZ^\conj + \lambda \matI_n$ is $\Delta$-spectral approximation of $\matK+\lambda\matI_n$. Furthermore, $s_{\bar{\tau}_{\rad}} = O\Big( (248\rad)^d \log(n_\lambda)^{d/2} + (200\log n_\lambda)^{2d}\Big)$ and $\bar{p}_\rad(\cdot)$ can be sampled from in $O(d)$ time.
\end{thm}
\begin{proof}
	The result follows from Lemma \ref{lem:leverage-sampling} and the fact that $\bar{\tau}_\rad(\cdot)$ upper bounds the true ridge leverage function, which is shown in Theorem~\ref{thm:lev-scores-ub} of \S\ref{sec:actualBounds}. The bound on $s_{\bar{\tau}_{\rad}}$ can be computed as follows.
	Let us denote $g_1(\eta) = (2\pi)^{-1/2} e^{-\eta^2/2} \max(1, |\eta|)$ and $g(\veta) = g_1(\eta_1)\cdot\ldots\cdot g_1(\eta_d)$. We calculate
	$$
	A\equiv \int_{-\infty}^{\infty} g_1(\eta) d\eta = \mathrm{erf}(1/\sqrt{2}) + \sqrt{2/e\pi} \approx 1.1663
	$$
	$$
	B \equiv  2 \int_{10\sqrt{\log n_\lambda}}^{\infty} g_1(\eta) d\eta = \sqrt{\frac{2}{\pi}} n^{-50}_\lambda\,.
	$$
	We now have (computed using a technique shown later in the proof)
	$$
	\int_{\|\veta\|_\infty > 10\sqrt{\log(n_\lambda)}} g(\veta) d\veta  = \sum^{d-1}_{j=0} (A-B)^{j}A^{d-1-j} B.
	$$
	The bound  $d \le 5\log(n_\lambda) + 1$ ensures that
	\begin{align*}
	s_{\bar{\tau}_\rad} = \int_{\RR^d} \bar{\tau}_\rad(\veta) d\veta
	&= \left(\big(12.4\max(\rad, 2000 \log^{1.5} n_\lambda) \big)^{d}+1\right) (20\sqrt{\log n_\lambda})^d + n_\lambda\cdot \int_{\|\veta\|_\infty > 10\sqrt{\log(n_\lambda)}} g(\veta) d\veta \\
	&= O\Big( (248\rad)^{d} \log(n_\lambda)^{d/2} + (200\log n_\lambda)^{2d}\Big).
	\end{align*}

	Sampling from $\bar{\tau}_\rad(\eta)$ amounts to sampling from a mixture of the uniform distribution on $[-10\sqrt{\log n_\lambda}, 10\sqrt{\log n_\lambda}]^d$ and the tail of the  distribution defined by $\bar{\tau}_R$: with probability	$\frac1{s_{\bar{\tau}_\rad}} (20\sqrt{\log n_\lambda})^d$ $\cdot \left(\big(12.4\max(\rad, 2000 \log^{1.5} n_\lambda) \big)^{d}+1\right)$ sample from the uniform distribution and with remaining probability sample from the tail.
	Above, we have an closed form expression for the total mass of the tail, which allows us to decide whether to sample from the uniform part or from the tail part using a single
	sample from a uniform distribution on $[0,1]$.

	Sampling from the uniform part, clearly takes $O(d)$ time. Sampling from the tail can be easily done via rejection sampling at $O(d)$ expected cost, as we now show.
	The density $p_t$ of the tail is:
	$$p_t(\veta) = \frac{g(\veta)\cdot \mathds{1} \big[ \|\veta\|_\infty \ge 10\sqrt{\log n_\lambda} \big] }{ \int_{\|\veta'\|_\infty \ge 10\sqrt{\log n_\lambda}} g(\veta') d\veta'}$$
	Now we write $\mathds{1} \big[ \|\veta\|_\infty \ge 10\sqrt{\log n_\lambda} \big]$ as a union of disjoint partitions as follows:
$$\mathds{1} \big[ \|\veta\|_\infty \ge 10\sqrt{\log n_\lambda} \big] = \sum_{j=1}^d \mathds{1} \big[ |\eta_j| \ge 10\sqrt{\log n_\lambda} \big] \mathds{1} \big[ |\eta_k| < 10\sqrt{\log n_\lambda} \,\, \forall k \in \{1, .., j-1\} \big]$$
	Let $R_j$ denote the $j$th region in the above partition:
	$$R_j = \Big\{ \veta \,\,:\,\, |\eta_j| \ge 10\sqrt{\log n_\lambda} \, , \, |\eta_k| < 10\sqrt{\log n_\lambda} \, \forall k \in \{1, .., j-1\} \Big\}$$
	Thus, the density $p_t$ can written as follows:
	$$ p_t(\veta) = \frac{g(\veta) \cdot \sum_{j=1}^d \mathds{1} \big[ \veta \in R_j \big] }{ \int_{\|\veta'\|_\infty \ge 10\sqrt{\log n_\lambda}} g(\veta') d\veta'} $$
	Now because $R_j$'s are disjoint sets we can do the following.
\begin{enumerate}
\item We first take a sample $j \in [d]$ with probability $\frac{\int_{ \veta \in R_j} g(\veta) d\veta}{ \int_{\|\veta'\|_\infty \ge 10\sqrt{\log n_\lambda}} g(\veta) d\veta'} $. In order to execute this step, we first compute:
$$
\int_{ \veta \in R_j} g(\veta) d\veta = A^{d-j} (A-B)^{j-1} B
$$
Then given the probabilities we can sample $j$ in $O(d)$ time.
\item Next, we need to take a sample from the distribution:
\begin{align}
p_{t,j}(\veta) &= \frac{g(\veta) \cdot \mathds{1} \big[ \veta \in R_j \big] }{ \int_{ \veta' \in R_j} g(\veta') d\veta' }\nonumber\\
& = \frac{g_1(\eta_j) \cdot \mathds{1} \big[ |\eta_j| \ge 10\sqrt{\log n_\lambda} \big] }{ \int_{ |\eta'| \ge 10\sqrt{\log n_\lambda} } g_1(\eta') d\eta' } \cdot \prod_{k=1}^{j-1} \frac{g_1(\eta_k) \cdot \mathds{1} \big[ |\eta_k| < 10\sqrt{\log n_\lambda} \big] }{ \int_{ |\eta'| < 10\sqrt{\log n_\lambda} } g_1(\eta') d\eta' } \cdot \prod_{k=j+1}^{d} g_1(\eta_k)\nonumber
\end{align}
We explain how to sample from this distribution in the subsequent paragraphs.
\end{enumerate}
	
 We now explain how to perform the sampling in the second step. It can be seen in the above expression that sampling from the distribution whose density is $p_{t,j}(\veta)$ amounts to sampling each of $d$ coordinates of $\veta$ independently from their corresponding distributions. There are three types of distributions that we need to sample from. Either we need to sample proportional to $g_1$ (coordinates whose index is higher than $j$) or we need to sample from the head of $g_1$ (rescaled) (coordinates $1,\dots,j-1$), or we sample from the tail (coordinate $j$).

 We start with sampling proportional to $g_1$. This distribution is a mixture of Gaussian on $[-1,1]$ and enlarged Gaussian outside. The total mass is $A$, and the relative mass of the Gaussian part is $\mathrm{erf}(1/\sqrt{2})/A$. First, we sample a uniform random variable $U$, which will decide which part of the mixture we sample. If $U$ is bigger than $\mathrm{erf}(1/\sqrt{2})/A$, then the sample comes from the tail. In that case, we generate the sample by computing $G^{-1}(U)$ where $G(\xi) \equiv A^{-1}\int_{-\xi}^{\xi} g_1(\eta) d\eta$ (i.e., we use inverse transform sampling). Note that $G$ has a simple invertible closed form for values larger than $1$, we have $G(1) = \mathrm{erf}(1/\sqrt{2})/A$. If $U \leq \mathrm{erf}(1/\sqrt{2})/A$, then the sample comes from the Gaussian part. To generate the sample from the head, we sample a standard Gaussian $X$, and test whether $X \leq 1$. If it is, then we use the sample, otherwise we reject and repeat. Obviously, the expected number of samples we need is $O(1)$.

 To sample proportional to the head of $g_1$, we repeat the above procedure and test whether the sample is smaller than $10\sqrt{\log n_\lambda}$. If it is not, we reject the sample and repeat.

 To sample proportional to the tail of $g_1$, we sample a uniform random variable $T$ on $[0, B/A]$, and return $G^{-1}(1-T)$, using the closed from expression for $G^{-1}$ for values close to $1$.

Thus, we can generate a sample in step 2 in $O(d)$ expected time, and overall the sampling procedure takes $O(d)$.
\end{proof}

Theorem \ref{thm:improvedBound} represents a possibly exponential improvement over the bound obtainable by classic random Fourier features.
Consider $d=1$ and $\rad \ge \log^{1.5}(n_\lambda)$. The bound on $s_{\bar{\tau}_\rad}$ shows that our modified distribution requires $O(\rad \sqrt{\log (n_\lambda)})$ samples, as compared to the lower bound of $\Omega(n_\lambda)$ given by Theorem \ref{thm:rr-samples}.

\begin{figure}[t]
	\centering
	\includegraphics[width=.60\textwidth]{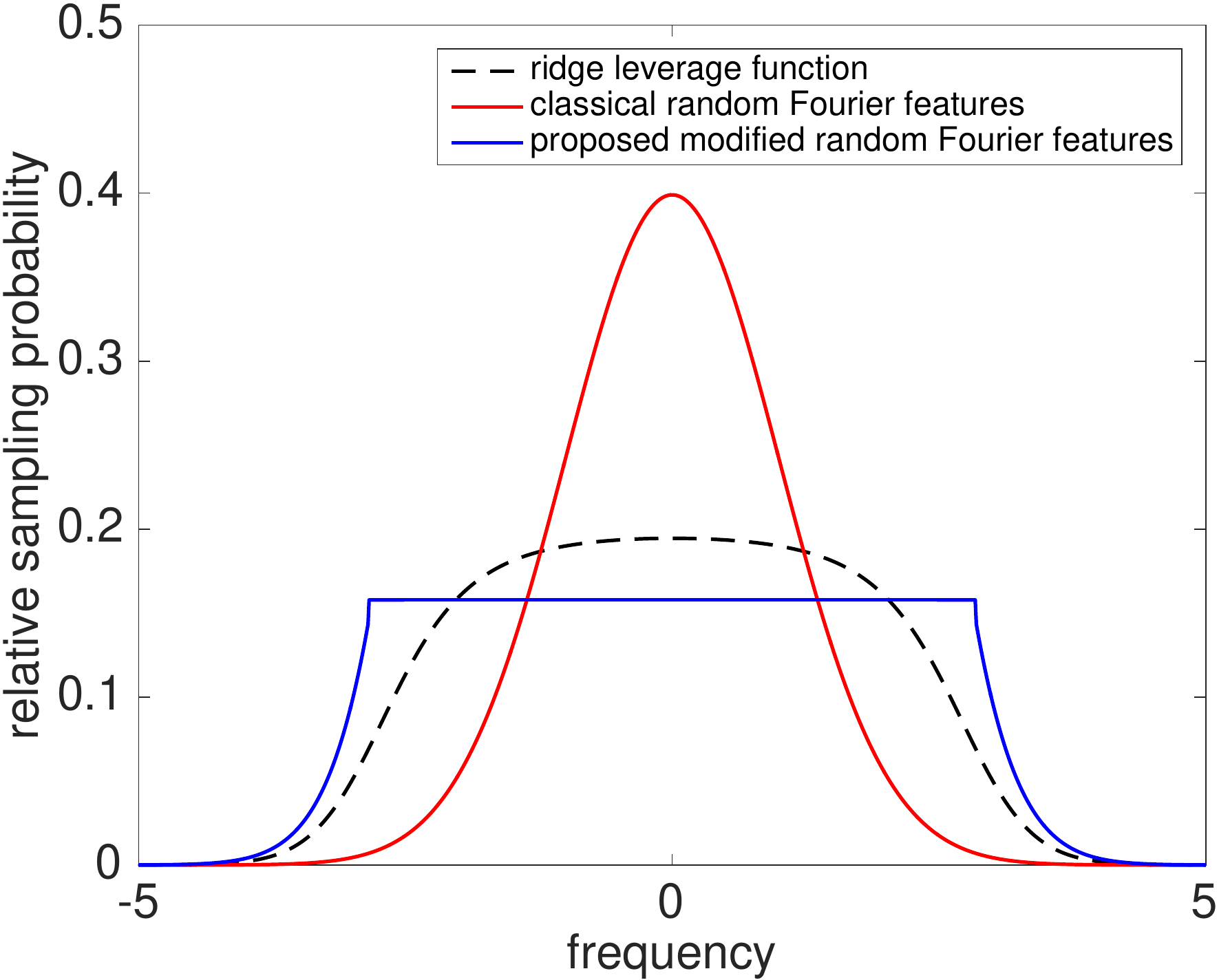}
	\caption{Plot of the true ridge leverage function vs. the classic random Fourier features distribution and our modified distribution, for a dataset of $n = 401$ equispaced points on the range $[-5,5]$. Our modified distribution closely matches the true leverage scores to within a small multiplicative factor. In contrast, the classical distribution oversamples low frequencies, at the expense of substantially undersampling higher frequencies.}
	\label{distCompared}
\end{figure}

\section{Bounding the Ridge Leverage Function}

\label{sec:actualBounds}

We now discuss our approach to bounding the ridge leverage function of the Gaussian kernel, which leads to Theorems \ref{thm:rr-samples} and \ref{thm:improvedBound}. The key idea is to reformulate the leverage function as the solution of
two dual optimization problems. By exhibiting suitable test functions for these optimization problems, we are able to give both upper and lower bounds on the ridge leverage function, and correspondingly on the sampling performance of classic and modified Fourier feature sampling.

% leading to our sampling lower bound for classic random Fourier features (Theorem \ref{thm:rr-samples}) and our improved method for the Gaussian kernel (Theorem \ref{thm:improvedBound}).

\subsection{Primal-Dual Characterization}
Before introducing our primal-dual characterization of the ridge leverage function, we give a few definitions.
%
%In this section we give two alternative characterizations of the ridge leverage function:
%one as a minimization, and the other as a maximization.  These characterizations are useful for
%bounding the leverage function, and thus designing an improved sampling distribution for the Gaussian kernel. Proofs can be found in the Appendix~\ref{appendix:lem:altlev}.
%
Define the operator $\matPhi : L_2(d\mu) \to \CC^n$ by
\begin{equation}\label{eq:phi-operator}
\matPhi y \equiv \int_{\RR^d} \z(\vxi) y(\vxi) d\mu(\vxi).
\end{equation}
 We first prove that the operator $\matPhi$ is defined on all $L_2(d\mu)$ and is a bounded linear operator. Indeed, for $y \in L_2(d\mu)$ we have:
\begin{eqnarray*}
	\TNormS{\matPhi y} &=& \LTNormS{\int_{\RR^d} \z(\vxi) y(\vxi) d\mu(\vxi)} \\
	&\leq & \int_{\RR^d} \TNormS{ \z(\vxi) y(\vxi) } d\mu(\vxi) \\
	&=&  \int_{\RR^d} |y(\vxi)|^2 \cdot \TNormS{ \z(\vxi) } d\mu(\vxi) \\
	&=& n \cdot \XNormS{y}{L_2(d\mu)}\,.
\end{eqnarray*}
Therefore, there is a unique adjoint operator
$\matPhi^\conj : \CC^n \to  L_2(d\mu)$,
such that $\langle \matPhi y, \x \rangle_{\CC^n} = \langle  y, \matPhi^\conj \x \rangle_{L_2(d\mu)}$ for every
$y \in L_2(d\mu)$ and $\x\in \CC^n$.
It is easy to verify that
$(\matPhi^\conj \x)(\veta) = \z(\veta)^\conj \x$.
We now have the following:
\begin{prop}
	\label{prop13}
	For every $\x \in \CC^n$:
	$$\matPhi \matPhi^\conj \x = \matK \x.$$
\end{prop}
\begin{proof}
	We have that for every $\x \in \CC^n$,
	\begin{eqnarray*}
		\matPhi \matPhi^\conj \x &=& \int_{\RR^d} \z(\vxi) (\matPhi^\conj \x)(\vxi) d\mu(\vxi) \\
		&=& \int_{\RR^d} \z(\vxi) \z(\vxi)^\conj \x d\mu(\vxi) \\
		& = & \left( \int_{\RR^d} \z(\vxi) \z(\vxi)^\conj d\mu(\vxi)\right) \x = \matK \x.
	\end{eqnarray*}
\end{proof}

We can now equivalently define the ridge leverage function $\tau_\lambda(\cdot)$ via the following optimization problems. Similar characterization are known for the finite dimensional case. Here we extend these results to an infinite dimensional case.
\begin{lem}
	\label{lem:altlev-ub}
	The ridge leverage function can alternatively be defined as:
	\begin{equation}
	\label{eq:altlev}
	\tau_\lambda(\veta) = \min_{y \in L_2(d\mu)} \lambda^{-1}\TNormS{\matPhi y - \sqrt{p(\veta)} \z(\veta)} + \XNormS{y}{L_2(d\mu)}.
	\end{equation}
\end{lem}
\begin{proof}
The minimizer of the right-hand side of~\eqref{eq:altlev} can be obtained from the usual normal
equations, and simplified using the matrix inversion lemma for operators~\cite{Ogawa88}:
\begin{eqnarray*}
	y^\star &=& \sqrt{p(\veta)} (\matPhi^\conj \matPhi + \lambda \matI_{L_2(d\mu)})^{-1} \matPhi^\conj \z(\veta) \\
	& = & \sqrt{p(\veta)} \matPhi^\conj (\matPhi \matPhi^\conj + \lambda \matI_n)^{-1}  \z(\veta) \\
	& = & \sqrt{p(\veta)} \matPhi^\conj (\matK + \lambda \matI_n)^{-1}  \z(\veta)
\end{eqnarray*}
where we used Proposition~\ref{prop13} to replace $\Phi \matPhi^\conj$ with $\matK$.
So, $y^\star(\vxi) = \sqrt{p(\veta)} \z(\vxi)^\conj (\matK + \lambda \matI_n)^{-1}  \z(\veta)$.
We now have
\begin{eqnarray*}
	\|y^\star\|^2_{L_2(d\mu)} & = &p(\veta) \int_{\RR^d} | \z(\vxi)^\conj (\matK + \lambda \matI_n)^{-1}  \z(\veta)|^2d\mu(\vxi) \\
	& = & p(\veta) \int_{\RR^d} \z(\veta)^\conj (\matK + \lambda \matI_n)^{-1} \z(\vxi) \z(\vxi)^\conj (\matK + \lambda \matI_n)^{-1}  \z(\veta)d\mu(\vxi) \\
	& = & p(\veta) \z(\veta)^\conj (\matK + \lambda \matI_n)^{-1} \left(\int_{\RR^d} \z(\vxi) \z(\vxi)^\conj d\mu(\vxi) \right) (\matK + \lambda \matI_n)^{-1}  \z(\veta) \\
	& = & p(\veta) \z(\veta)^* (\matK + \lambda \matI_n)^{-1} \matK (\matK + \lambda \matI_n)^{-1}  \z(\veta) \\
	& = & p(\veta) \z(\veta)^* (\matK + \lambda \matI_n)^{-1} (\matK + \lambda \matI_n - \lambda \matI_n) (\matK + \lambda \matI_n)^{-1}  \z(\veta)\\
	& = & p(\veta) \z(\veta)^* (\matK + \lambda \matI_n)^{-1}  \z(\veta) - \lambda p(\veta) \z(\veta)^* (\matK + \lambda \matI_n)^{-2}  \z(\veta)\\
\end{eqnarray*}
and
\begin{eqnarray*}
	\TNormS{\matPhi y^\star - \sqrt{p(\veta)} \z(\veta)} &=& p(\veta) \TNormS{ \matPhi  \matPhi^\conj (\matK + \lambda \matI_n)^{-1}  \z(\veta) - \z(\veta)} \\
	& = & p(\veta) \TNormS{ (\matK (\matK + \lambda \matI_n)^{-1} - \matI_n)\z(\veta)} \\
	& = & p(\veta) \TNormS{ \big((\matK +\lambda \matI_n - \lambda \matI_n) (\matK + \lambda \matI_n)^{-1} - \matI_n \big)\z(\veta)}\\
	& = & p(\veta) \TNormS{ \big( \lambda (\matK + \lambda \matI_n)^{-1} \big)\z(\veta)}\\
	& = & \lambda^2 p(\veta) \z(\veta)^\conj (\matK + \lambda \matI_n)^{-2} \z(\veta)\,.
\end{eqnarray*}

Now plugging these into \eqref{eq:altlev} gives:
\begin{align*}
&\|y^\star\|^2_{L_2(d\mu)} + \lambda^{-1} \TNormS{\matPhi y^\star - \sqrt{p(\veta)} \z(\veta)} \\
&\qquad = p(\veta) \z(\veta)^\conj (\matK + \lambda \matI_n)^{-1}  \z(\veta) - \lambda p(\veta) \z(\veta)^* (\matK + \lambda \matI_n)^{-2}  \z(\veta)\\
&\qquad \qquad+ \lambda p(\veta) \z(\veta)^* (\matK + \lambda \matI_n)^{-2} \z(\veta)\\
&\qquad = p(\veta) \z(\veta)^* (\matK + \lambda \matI_n)^{-1}  \z(\veta) \\
&\qquad = \tau_\lambda(\veta).
\end{align*}
\end{proof}

Recall that we define $\z(\veta)_j = e^{-2\pi i \x_j^\T \veta}$. So $\matPhi$ is just a $d$-dimensional Fourier transform of the function $y$ weighted by probability measure $d\mu(\vxi) = p(\vxi)d\vxi$, and evaluated at the frequencies given by the data points $\x_1,...,\x_n$.
Thus, the optimization problem of Lemma \ref{lem:altlev-ub} asks us to produce a function $y$ whose Fourier transform is close to the pure cosine wave $\sqrt{p(\veta)} \z(\veta)$ on our datapoints. At the same time, to keep the second term of \eqref{eq:altlev} small, $y$ should have bounded norm under the $\mu(\vxi)$ measure. So, the trivial solution of setting $y$ to be a Dirac delta function at $\veta$ (whose Fourier transform is a pure cosine with frequency $\veta$) fails. A more carefully chosen function must be constructed whose Fourier transform looks like the cosine at our datapoints but diverges elsewhere. Such a function certifies that, on our datapoints, the cosine of frequency $\veta$ can be approximately reconstructed with low energy using other frequencies. Hence $\veta$ is not a critical frequency for sampling, so $\tau_\lambda(\veta)$ is small.

Dual to minimization objective of Lemma \ref{lem:altlev-ub}, which allows us to certify upper bounds on the ridge leverage function, we have a maximization objective allowing us to certify lower bounds:
\begin{lem}
	\label{lem:altlev-lb}
	The ridge leverage function can alternatively be defined as:
	\begin{equation}
	\label{eq:altlev-lb}
	\tau_\lambda(\veta) = \max_{\valpha \in \CC^n} \frac{p(\veta)\cdot|\z(\veta)^{\conj} \valpha|^2}{\XNormS{\matPhi^\conj \valpha}{L_2(d\mu)}+\lambda \TNormS{\valpha}}.
	\end{equation}
\end{lem}
\begin{proof}
The optimization problem~\eqref{eq:altlev} can equivalently be reformulated as the following problem:
\begin{equation*}
\begin{aligned}
\tau_\lambda(\veta) = & {\text{ minimum}}
& & \XNormS{y}{L_2(d\mu)} + \TNormS{\u} \\
& y \in L_2(d\mu) ;
& & \u \in \CC^n\\
& \text{subject to:}
& & \matPhi y + \sqrt{\lambda}\u = \sqrt{p(\veta)} \z(\veta).
\end{aligned}
\end{equation*}

First we show that for any $\valpha \in \CC^n$, the argument of the minimization problem in \eqref{eq:altlev-lb} is no bigger than $\tau_\lambda(\veta)$. That is because for the optimal solution to above optimization, namely $\bar \u$ and $\bar y$, we have:
$$\matPhi \bar y + \sqrt{\lambda}\bar \u = \sqrt{p(\veta)} \z(\veta).$$
Hence,
\begin{align*}
|\sqrt{p(\veta)}\valpha^\conj \z(\veta)| &= |\valpha^\conj(\matPhi \bar y + \sqrt{\lambda}\bar \u)|\nonumber\\
&= |\valpha^\conj\matPhi \bar y + \valpha^\conj\sqrt{\lambda}\bar \u|\nonumber\\
&\le |\valpha^\conj\matPhi \bar y| + |\valpha^\conj\sqrt{\lambda}\bar \u|\nonumber\\
&= | \langle \valpha , \matPhi \bar y \rangle_{\CC^n} | + |\valpha^\conj\sqrt{\lambda}\bar \u|\nonumber\\
&= | \langle \matPhi^\conj \valpha , \bar y \rangle_{L_2(d\mu)} | + |\valpha^\conj\sqrt{\lambda}\bar \u| \nonumber\\
&\le \XNorm{\matPhi^\conj\valpha}{L_2(d\mu)}\cdot\XNorm{\bar y}{L_2(d\mu)} +
\sqrt{\lambda}\TNorm{\valpha^\conj}\cdot\TNorm{\bar \u}
\end{align*}
where the last inequality follows from Cauchy-Schwarz inequality
($|\valpha^\conj\matPhi \bar y| = |(\valpha^\conj\matPhi \bar y)^\conj| = |(\matPhi \bar y)^\conj\valpha| = |\langle\bar y, \matPhi^*\valpha\rangle_{L_2(d\mu)}|\leq  \XNorm{\matPhi^\conj\valpha}{L_2(d\mu)} \cdot \XNorm{\bar y}{L_2(d\mu)}$). By another use of Cauchy-Schwarz we have:
\begin{align*}
p(\veta)|\valpha^\conj  \z(\veta)|^2 &\le \Big( \XNorm{\matPhi^\conj\valpha}{L_2(d\mu)} \XNorm{\bar
	y}{L_2(d\mu)} + \sqrt{\lambda}\TNorm{\valpha^\conj} \cdot \TNorm{\bar \u} \Big)^2\nonumber\\
&\le  \Big( \XNormS{\matPhi^\conj\valpha}{L_2(d\mu)} + {\lambda}\TNormS{\valpha^\conj} \Big)
\cdot \Big( \XNormS{\bar y}{L_2(d\mu)} + \TNormS{\bar \u} \Big).
\end{align*}
Therefore, for every $\valpha \in \CC^n$,
\begin{equation}
\frac{p(\veta)|\valpha^\conj  \z(\veta)|^2}{\XNormS{\matPhi^\conj\valpha}{L_2(d\mu)} +
	\lambda \TNormS{\valpha}} \le \XNormS{\bar y}{L_2(d\mu)} + \TNormS{\bar \u} =
\tau_\lambda(\veta).
\end{equation}
Now it is enough to show that at the optimal $\valpha$ the dual problem gives the leverage scores. We show that $\bar \valpha = \sqrt{p(\veta)}( \matK + \lambda \matI_n )^{-1} \z(\veta)$ matches the leverage scores. First note that for any $\valpha \in \CC^n$ we have
\begin{align*}
\XNormS{\matPhi^\conj\valpha}{L_2(d\mu)} + \lambda \TNormS{\valpha} &=
\langle \matPhi^\conj\valpha, \matPhi^\conj\valpha\rangle_{L_2(d\mu)} + \lambda \valpha^\conj\valpha\nonumber\\
&= \langle \matPhi \matPhi^\conj\valpha, \valpha\rangle_{\CC^n} + \lambda \valpha^\conj\valpha\nonumber\\
&= \langle \matK\valpha, \valpha\rangle_{\CC^n} + \lambda \valpha^\conj\valpha\nonumber\\
&=  \valpha^\conj( \matK + \lambda \matI_n)\valpha.
\end{align*}

Now by substituting $\bar \valpha = \sqrt{p(\veta)} ( \matK + \lambda \matI_n )^{-1} \z(\veta)$ we have:
\begin{align}
\frac{p(\veta)|\bar \valpha^\conj  \z(\veta)|^2}{\XNormS{\matPhi^\conj\bar\valpha}{L_2(d\mu)}
	+ \lambda \TNormS{\bar\valpha}} &= \frac{p(\veta)^2| \z(\veta)^\conj ( \matK + \lambda \matI_n
	)^{-1}  \z(\veta)|^2}{p(\veta) \z(\veta)^\conj ( \matK + \lambda \matI_n )^{-1} ( \matK + \lambda
	\matI_n)( \matK + \lambda \matI_n )^{-1}  \z(\veta)}\nonumber\\
&= p(\veta)| \z(\veta)^\conj ( \matK + \lambda \matI_n )^{-1}  \z(\veta)| \nonumber\\
& = \tau_\lambda(\veta).
\end{align}
\end{proof}
The optimization problem of Lemma \ref{lem:altlev-lb} asks us to exhibit a set of coefficients $\valpha \in \CC^n$, such that the Fourier domain representation of our point set weighted by these coefficients (i.e. $\matPhi^{\conj} \valpha$) is concentrated at frequency $\veta$ and hence $\frac{p(\veta)\cdot|\valpha^\conj \z(\veta)|^2}{\XNormS{\matPhi^\conj \valpha}{L_2(d\mu)}}$ is large. $\valpha$ certifies that $\veta$ is a critical frequency for representing our point set and so $\tau_\lambda(\veta)$ must be large. $\lambda \norm{\valpha}^2$ is a regularization term, decreasing the ridge leverage function when $p(\veta)$ is very small, i.e. when $\veta$ has small weight in the Fourier transform of our kernel.

%Motivated by the lower bound of Theorem \ref{thm:rr-samples}, we propose a new sampling distribution for the Gaussian kernel
%which more closely approximates the true leverage function distribution and is thus capable of spectrally approximating the kernel matrix using considerably fewer samples.
%{\bf \color{red} For simplicity we consider only one dimensional data.}

\subsection{Bounding the Gaussian Kernel Leverage Function: Upper Bound}\label{sec:upperBoundSketch}

%In this section we prove nearly matching bounds on the leverage score function for the one-dimensional Gaussian kernel on bounded datasets.
%For simplicity of presentation we focus on the one-dimensional setting. Our results extend to higher dimensions, albeit with an exponential
%in the dimension loss in the gap between upper and lower bounds.

We start by applying Lemma \ref{lem:altlev-ub} to prove a ridge leverage function upper bound for the Gaussian kernel. Again, to reduce clutter, we state the result for a fixed bandwidth $\sigma=(2\pi)^{-1}$.
%Our bounds are parameterized by the width of the point set, which we denote by $\rad$.
%To reduce clutter, we present all results for fixed $\sigma=(2\pi)^{-1}$. This is without loss of generality since
%we can rescale the points by $(2\pi\sigma)^{-1}$ and adjust the interval to $[-\rad/(2\pi\sigma),\rad/(2\pi\sigma)]$.
%
%Our bounds are parameterized by the width of the point set, which we denote by $\rad$. The operator $\matPhi$ becomes
%\begin{equation}\label{eq:phi-gaussian}
%\matPhi y = \int_{\RR^d} e^{-\vxi^2/2} \z(\vxi) y(\vxi) d\vxi\,
%\end{equation}
%where $\z(\vxi)=(e^{2\pi i x_j \vxi})_{j=1}^n\in \CC^n$.
%
\begin{thm}\label{thm:lev-scores-ub}
	Consider the d-dimensional Gaussian kernel with $\sigma=(2\pi)^{-1}$. For any integer $n$ and parameter $0 < \lambda \le \frac{n}{2}$ such that $d \le n_\lambda/4$, and any radius $\rad > 0$,
	if $\x_1,...,\x_n \in \RR^d$  is contained in a $\ell_\infty$-ball of radius $\rad$ (i.e $\InfNorm{\x_i -\x_j} \leq 2\rad$ for all $i,j\in[n]$), then for every $\|\veta\|_\infty \le 10 \sqrt{\log n_\lambda}$ we have:
	$$\tau_\lambda(\veta) \le \Big( 12.4\max(\rad, 2000 \log^{1.5} n_\lambda) \Big)^{d} + 1\,.$$
\end{thm}

Applying Theorem \ref{thm:lev-scores-ub} for $\veta$ with $\|\veta\|_\infty < 10 \sqrt{\log n_\lambda}$ and Proposition \ref{prop:simple-tau-bound} for $\veta$ outside this range immediately implies our improved sampling bound Theorem \ref{thm:improvedBound}.

\paragraph{Theorem~\ref{thm:lev-scores-ub} Proof Outline (Details and a full proof in Appendix~\ref{sec:lev-scores-ub}).}

For simplicity we focus on the case of $d = 1$. Our proof for higher dimensions uses similar ideas.
To upper bound $\tau_\lambda(\eta)$ using Lemma~\ref{lem:altlev-ub} it suffices to exhibit any function $y_\eta \in L_2(d\mu)$ (i.e. with bounded norm $\XNormS{y_\eta}{L_2(d\mu)}$) such that, when reweighted by $\mu(\xi) = p(\xi)d\xi$, $y_\eta$'s Fourier transform is close to the pure cosine target function $\bv{z}(\eta)$ on our datapoints. In general the test function depends on $\eta$ and hence our subscript notation $y_\eta(\cdot)$.

%Lemma~\ref{lem:altlev-ub} allows us to bound $\tau_\lambda(\eta)$ simply by exhibiting any $y(\cdot)$ which makes the cost function
%small.
One simple attempt is  $y_\eta(\xi) = \frac{1}{\sqrt{p(\eta)}} \delta(\eta - \xi)$ where $\delta(\cdot)$ is the Dirac delta function. This
choice zeros out the first term of \eqref{eq:altlev}. However $\delta(\cdot)$ is not square integrable,
$y_\eta \not\in L_2(d\mu)$, so the lemma cannot be used (the norm is unbounded). Another attempt is $y_{\eta}(\xi) = 0$, which zeros out the
second term and recovers the trivial bound $\tau_\lambda(\eta) \leq \lambda^{-1} \norm{\sqrt{p(\eta)} \z(\eta)}_2^2 = p(\eta)n_\lambda$ of Proposition~\ref{prop:simple-tau-bound}.

We improve this bound by replacing the Dirac delta function at $\eta$ with a `soft spike' whose Fourier transform still looks approximately like
a cosine wave on $[-\rad,\rad]$, and hence at our data points, which are bounded on this range. The smaller $\rad$ is, the more spread out this function
can be, and hence the smaller its norm $\XNormS{y_\eta}{L_2(d\mu)}$, and the better the leverage function bound.

A natural idea is to consider the inverse Fourier transform of the cosine with frequency $\eta$ restricted to the range $[-\rad,\rad]$ -- i.e. multiplied by the box function on this range. It is well known that this is a sinc function with width $1/2\rad$, centered at $\eta$: $g_\eta(\xi) = 2\rad\cdot \sinc{2\rad (\xi-\eta)}$, where $\sinc{x} = \frac{\sin x}{x}$ (see Figure \ref{sincFigure}). If we set $y_\eta(\xi) = g_\eta(\xi) \cdot \frac{\sqrt{p(\eta)}}{p(\xi)}$, the $d\mu$ weighted Fourier transform at $x_j \in [-\rad,\rad]$, $(\matPhi y_\eta)_j$, will be identical to the target $\bv{z}(\eta)_j$ and so again the first term of \eqref{eq:altlev} will be $0$. Unfortunately, $ \XNormS{y_\eta}{L_2(d\mu)}$ will still be too large. The reweighting function $1/p(\xi) = 2\pi e^{\xi/2}$ grows exponentially in $\xi$, while $\sinc{2\rad(\xi-\eta)}$ only falls off linearly, so $y_\eta$ will have unbounded energy in the high frequencies.

\begin{figure}[t]
	\centering
	\begin{subfigure}{0.5\textwidth}
		\centering
		\includegraphics[width=.9\textwidth]{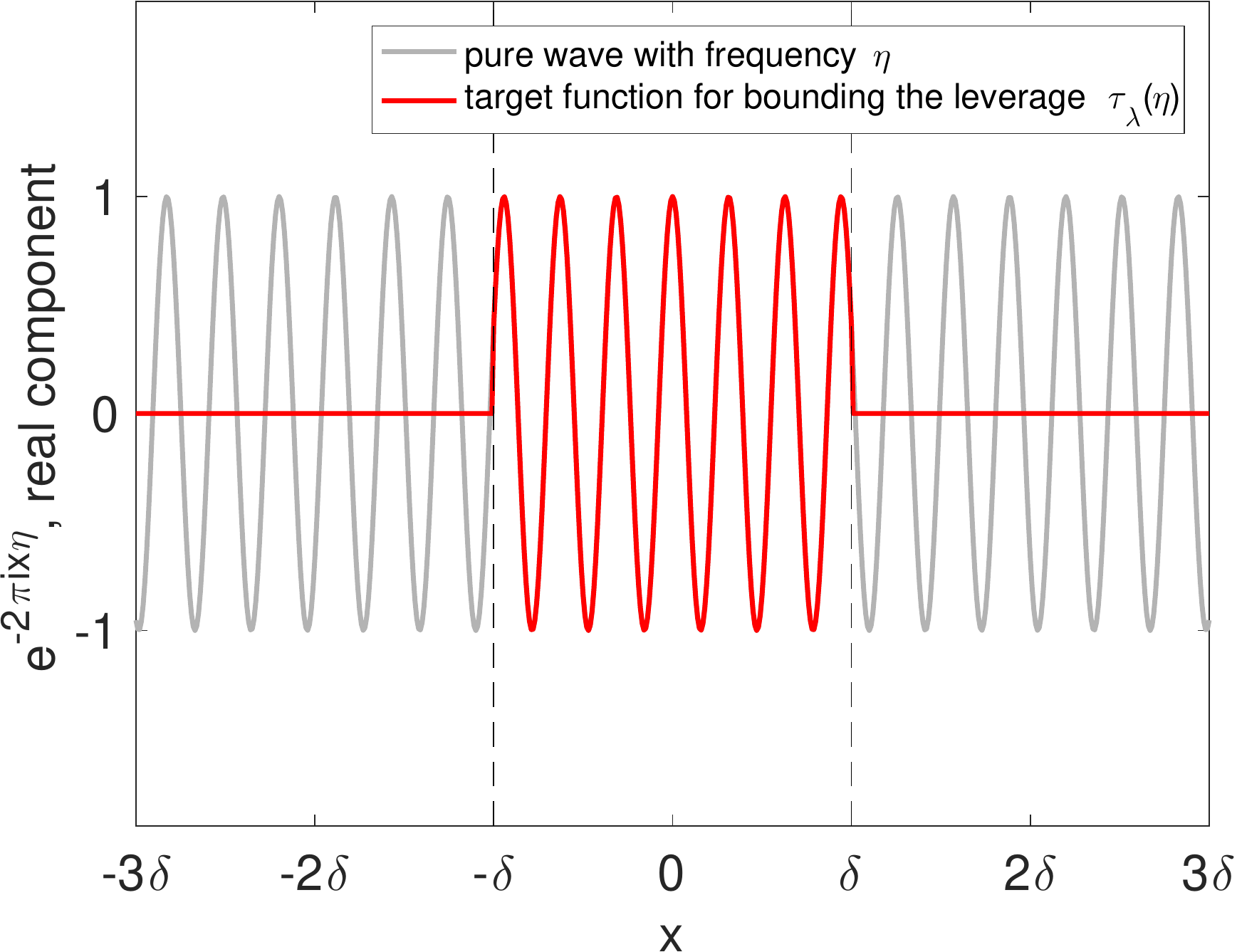}
		%\caption{Function.}
	\end{subfigure}%
	\begin{subfigure}{0.5\textwidth}
		\centering
		\includegraphics[width=.9\textwidth]{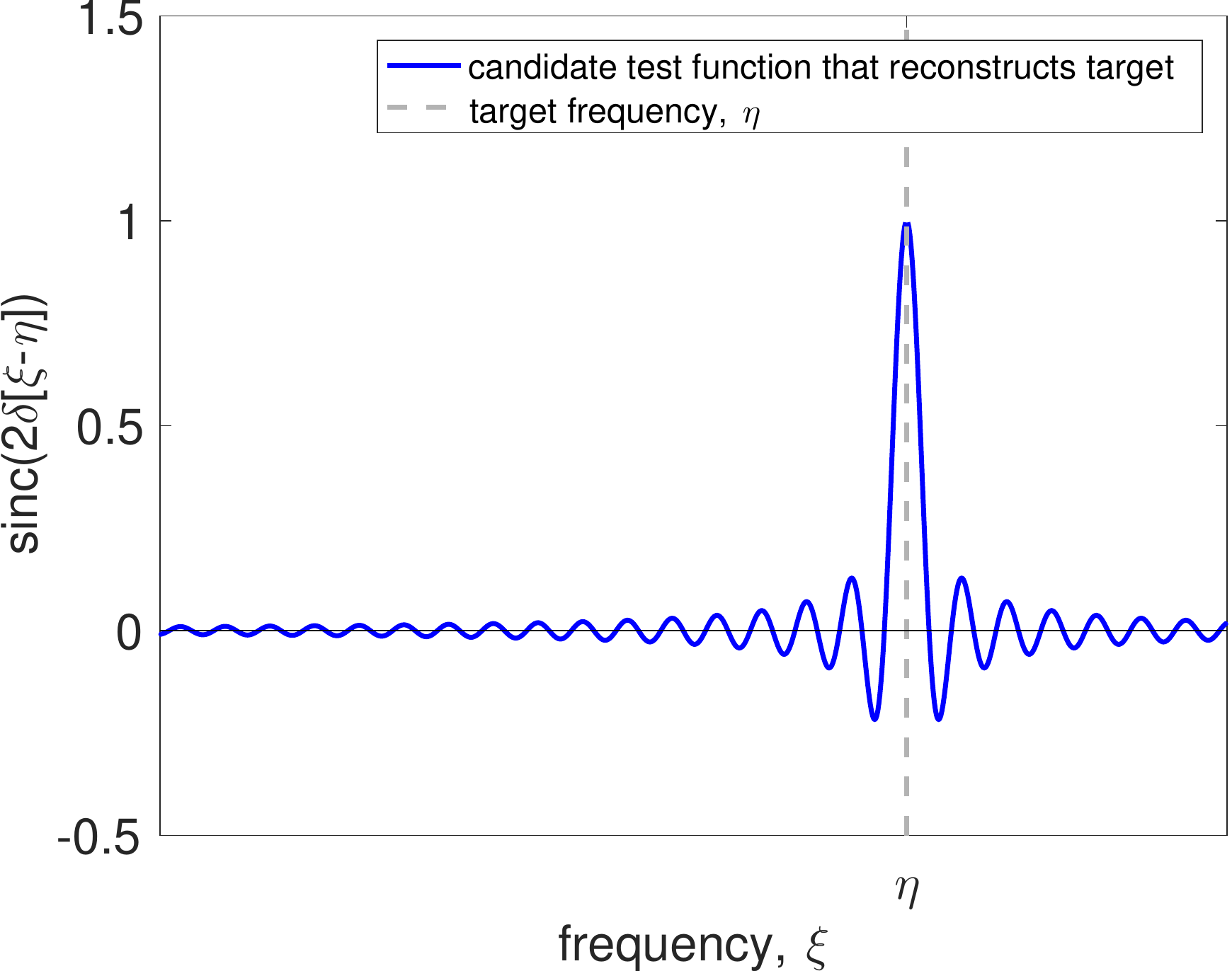}
		%\caption{Improved landmark sampling.}
	\end{subfigure}
	\caption{To minimize $\norm{\matPhi y_\eta - \bv{z}(\eta)}_2^2$, we can choose a test function $y_\eta(\xi)$ whose ($d\mu$ weighted) Fourier transform $\matPhi y_\eta$ is the pure cosine $e^{-2\pi i x \eta}$ multiplied by the box function on $[-\rad,\rad]$. Specifically, $y_\eta(\xi)p(\xi)$ is a sinc function centered at $\eta$. Unfortunately, $\XNormS{y_\eta}{L_2(d\mu)}$ is too large to get a good leverage function bound from Lemma \ref{lem:altlev-ub}. However, this construction is the starting point for our final test function, pictured in Figure \ref{softSincFigure}.}
	\label{sincFigure}
\end{figure}

To correct this issue, we dampen the sinc at higher frequencies by multiplying with a Gaussian, which decreases $ \XNormS{y_\eta}{L_2(d\mu)}$, but does not significantly affect the Fourier transform on $[-\rad,\rad]$.

Specifically, for some parameters $u,v$ set $g_\eta(\xi)$ to be product of a Gaussian with standard deviation $1/u$ with a sinc function with width $1/v$, both centered at $\eta$. The corresponding Fourier transform $\hat{g}_\eta(x)$ is the convolution of a Gaussian with standard deviation $u$ with a box of width $v$ -- i.e. a blurred box.

If we set $v = \Theta(\rad  + u\sqrt{\log n_\lambda})$ then the box, when centered at $x \in [-\rad, \rad]$ nearly covers the full mass of the Gaussian. Specifically, we have $1 - 1/n_\lambda^c \le |\hat{g}_\eta(x) | \le 1$ for $x \in [-\rad, \rad]$ and some large constant $c$. Since $g_\eta(\xi)$ is centered at $\eta$, $\hat{g}_\eta(x)$ is multiplied by the cosine wave $e^{-2\pi i x \eta}$, and so we have $(\matPhi y_{\eta})_j = \sqrt{p(\eta)} \hat{g}_\eta(x_j) \approx \bv{z}(\eta)_j$. Thus, when applying Lemma \ref{lem:altlev-ub} to bound the leverage function, the first term of ~\eqref{eq:altlev} will be negligible (see Figure \ref{softSincFigure}).

Theorem \ref{thm:lev-scores-ub} then follows from setting $u$ to minimize $\XNormS{y_\eta}{L_2(d\mu)}$ -- balancing increased damping for large $\eta$ with increased energy due to a more concentrated Gaussian. We eventually choose $u = \Theta(\log n_\lambda)$. Obtaining tight bounds and in particular achieving the right dependence on $\log n_\lambda$ requires several modifications, but the general intuition described above works!

\begin{figure}[t]
	\begin{subfigure}{0.5\textwidth}
		\centering
		\includegraphics[width=.9\textwidth]{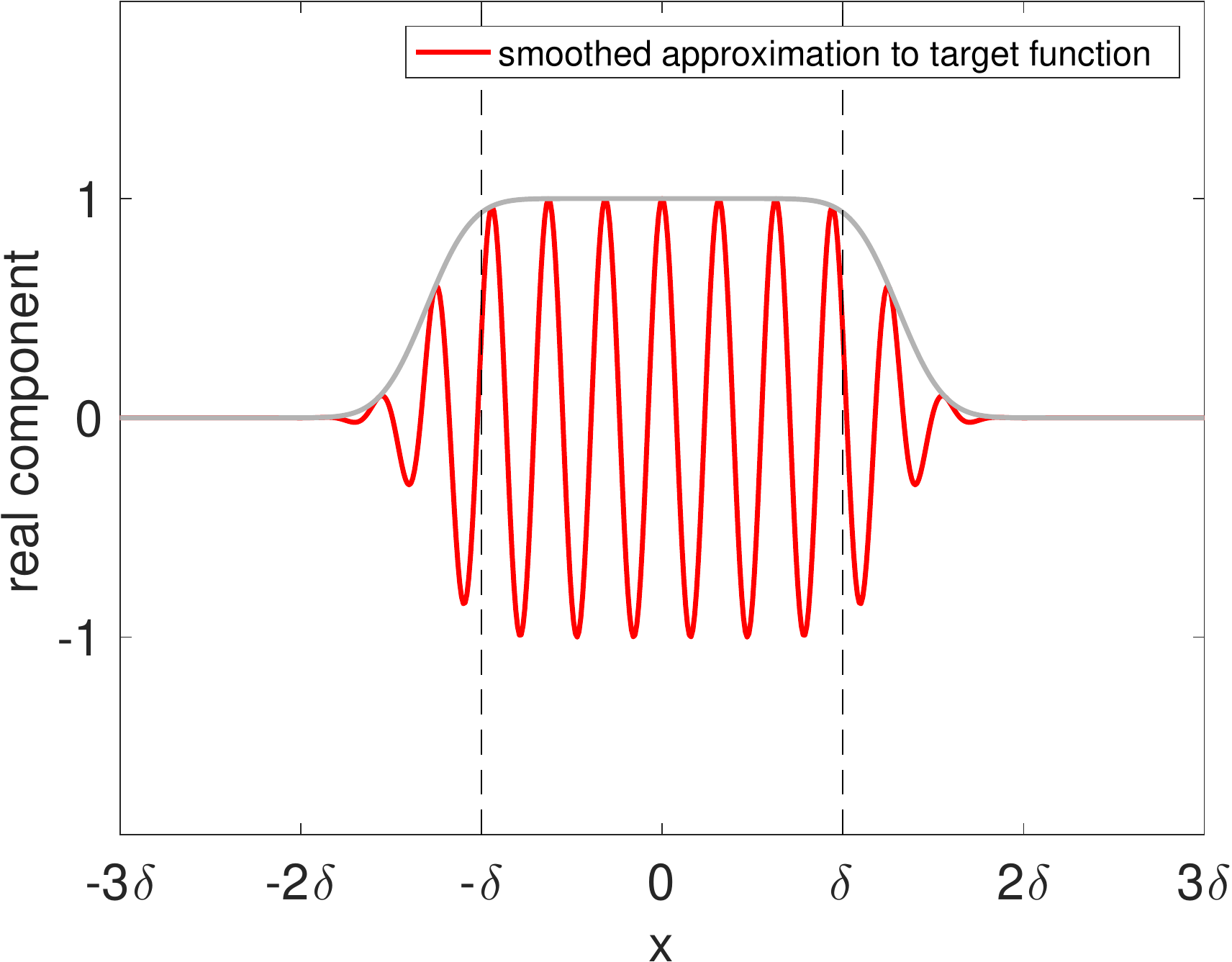}
		%\caption{Function.}
	\end{subfigure}%
	\begin{subfigure}{0.5\textwidth}
		\centering
		\includegraphics[width=.9\textwidth]{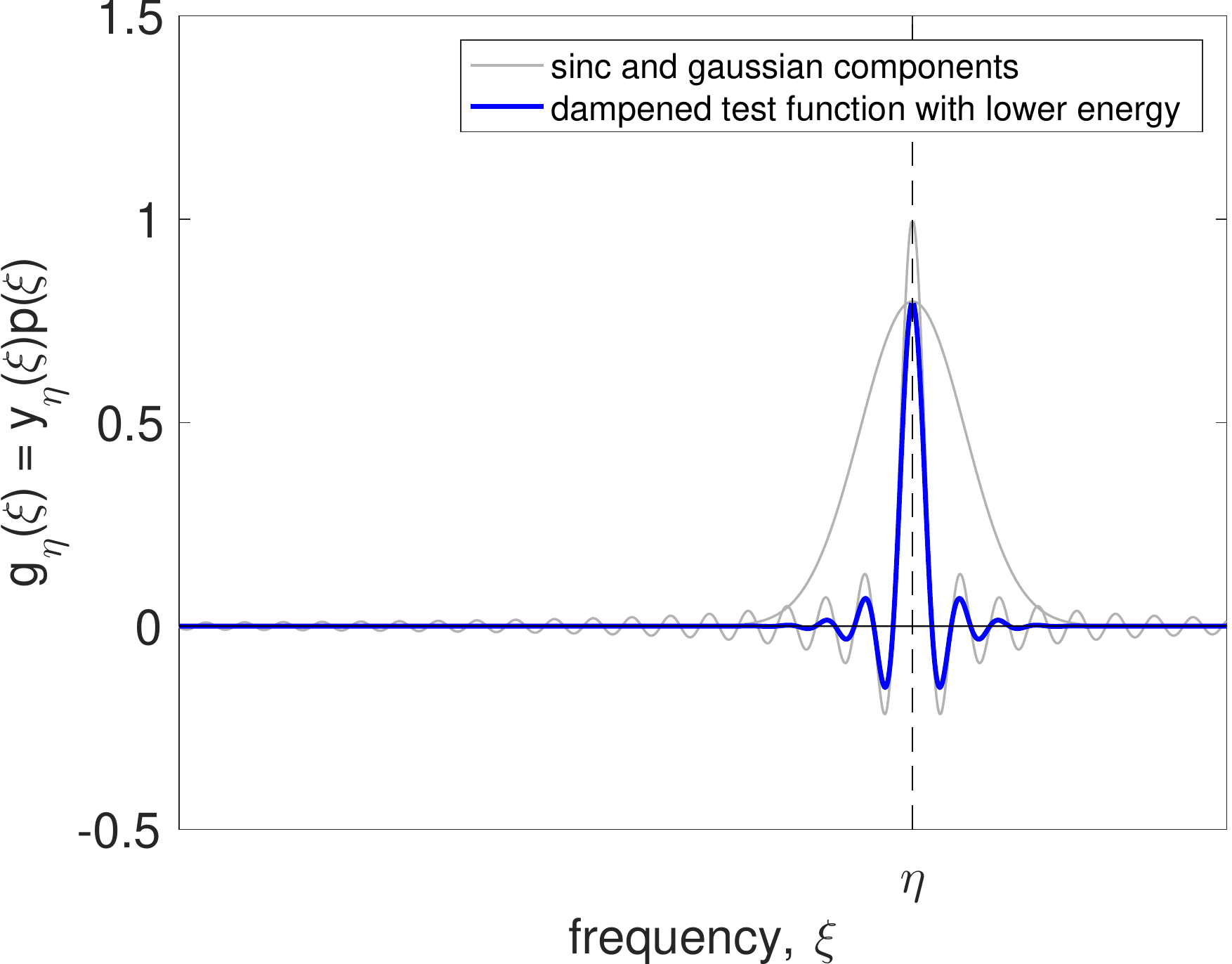}
		%\caption{Improved landmark sampling.}
	\end{subfigure}
	\caption{In comparison to Figure \ref{sincFigure}, damping the sinc function with a Gaussian decreases the energy $ \XNormS{y_\eta}{L_2(d\mu)}$ but does not significantly affect the Fourier transform on $[-\rad,\rad]$. $\matPhi y_{\eta}$ is a pure cosine with frequency $\eta$ multiplied by a blurred box function and thus $(\matPhi y_{\eta})_j \approx \bv{z}(\eta)_j$ for $x_j \in [-\rad,\rad]$. Accordingly, $y_{\eta}$ is ideal for bounding the leverage function via  Lemma \ref{lem:altlev-ub}.}
	\label{softSincFigure}
\end{figure}

\subsection{Bounding the Gaussian Kernel Leverage Function: Lower Bound}

Using the dual leverage function characterization of Lemma \ref{lem:altlev-lb}, we can give a near matching leverage function lower bound for the Gaussian kernel. We have:

\begin{thm}\label{thm:main-lev-lb}
	Consider the $d$-dimensional Gaussian kernel with $\sigma=(2\pi)^{-1}$. For any integer $n = m^d \ge 55$ with integer	$m\geq\max( 64\log(n)\sqrt{\log n_\lambda}, 64\log(n_\lambda), 3)$
and $1\leq d \leq \min\left(\frac{\log n}{18\log\log n},64n^{5/2}_\lambda\log^{3/2}n_\lambda\right)$, any parameter $\frac{10}{n} \le \lambda \le
	\min\left\{\left(\frac{1}{2}\right)^{2d}\cdot\frac{n}{1024}, n^{1-\frac{1}{128}}\right\}$, and every radius
	$2000\log n_\lambda \le \rad \le
	\frac{m}{500\sqrt{\log(n_\lambda)}}$,
	there exist $\x_1,\x_2,\dots,\x_n \in [-\rad,\rad]^d$ such that for every $\veta \in [-50\sqrt{\log n_\lambda} ,
	50\sqrt{\log n_\lambda}]^d$ we have
	\[ \tau_\lambda(\veta) \ge \frac{1}{128}\left( \frac{\rad}{3}\right)^d \cdot \frac{p(\veta)}{p(\veta) +
		(4\rad/3)^d n_\lambda^{-1}}. \]
\end{thm}

\paragraph{Theorem~\ref{thm:main-lev-lb} Proof Outline (Details and a full proof are given in Appendix~\ref{sec:main-lev-lb}).}
 The main idea of the proof is to use  Lemma~\ref{lem:altlev-lb} to get a lower bound on $\tau_\lambda(\veta)$. Note that the expression given under the
maximum in \eqref{eq:altlev-lb} provides a lower bound for any choice of $\valpha$. However, we provide a judiciously
chosen $\valpha$ that is related to the test function $y_{\veta} \in L_2(d\mu)$ used in the proof of
Theorem~\ref{thm:lev-scores-ub} which provides an upper bound on $\tau_\lambda(\veta)$. The choice of
$y_{\veta}$ in the proof of the upper bound is essentially a sinc function that is dampened by a Gaussian centered at
$\veta$. Due to the duality of the corresponding minimization and maximization problems in Lemma~\ref{lem:altlev-ub}
and Lemma~\ref{lem:altlev-lb}, respectively, the optimal $\bs\alpha$ must essentially be a scalar multiple of $\matPhi
y_{\veta}$, which is a (weighted) Fourier transform of $y_{\veta}$ evaluated on the data points $\x_1, \x_2,\dots,
\x_n$. Hence, we should intuitively choose $\bs\alpha$ to be the samples of $y_{\veta}$ on the data points.
Moreoever, to provide the tightest possible lower bound, we wish to choose our data points $\x_1, \x_2, \dots, \x_n$ to
be as spread apart as possible, as this corresponds to a higher statistical dimension (which corresponds to higher
leverage scores on average). Thus, we choose our points to be evenly spaced points on a $d$-dimensional grid located
inside an $L_\infty$ ball of radius $\rad$ around the origin.

\subsection{Bounding the Statistical Dimension of Gaussian Kernel Matrices}
\label{sec:stat-bound}

Theorems~\ref{thm:lev-scores-ub} and~\ref{thm:main-lev-lb} together imply a
tight bound on the statistical dimension of Gaussian kernel matrices corresponding
to bounded points sets (the proof appears in Appendix~\ref{sec:cor-sd-bound}):
\begin{cor} \label{cor:sd-bound}
	Consider the $d$-dimensional Gaussian kernel with $\sigma=(2\pi)^{-1}$. For any integer $n = m^d \ge 17$ with integer	$m\geq 3$, parameter $0 < \lambda \le \frac{n}{2}$, $1\leq d \leq \frac{5 \log n_\lambda}{\log\log n_\lambda}$, and $\rad>0$, if $\x_1,...,\x_n \in [-\rad,\rad]^d$:
	\begin{eqnarray*}
		s_\lambda(\matK) & \le & \Big( 20 \sqrt{\log n_\lambda} \Big)^d \left(\Big( 12.4\max(\rad, 2000 \log^{1.5} n_\lambda) \Big)^{d} + 1\right)\Big/ \Gamma(d/2+1) + 1 \\
		& = & O\Big( \frac{(248\rad)^{d} \log(n_\lambda)^{d/2} + (200\log n_\lambda)^{2d}}{\Gamma(d/2+1)} \Big)
	\end{eqnarray*}
	
	Furthermore, if $2000\log n_\lambda \le \rad \le
	\frac{m}{500\sqrt{\log(n_\lambda)}}$, $1\leq d \leq \frac{\log n}{2\log\log n}$ and $m\geq 64\log(n)\sqrt{\log n_\lambda}$ there exists a set of points $\x_1,\ldots, \x_n\subseteq [-\rad, \rad]^d$ such that:
	$$s_\lambda(\matK) = \Omega\left( \frac{\left( \frac{\sqrt{\pi}\rad}{18} \sqrt{ \log \frac{n_\lambda}{\rad^d}} \right)^d}{\Gamma(d/2+1)} \right).$$
%	The bounds above match up to constant factors if $1000 \log^{1.5} n_\lambda \le\rad\leq n_{\lambda}^{0.99}$. For any $1000 \log^{1.5} n_\lambda \le \rad\leq \frac{n}{500\sqrt{\log(n_\lambda)}}$ they match up to a $\sqrt{\log n_\lambda}$ factor.
\end{cor}

\newcommand{\Ic}{{\mathcal I}}

\section{Numerical Experiments}
We now report experiments on synthetic low-dimensional datasets. These experiments are designed to illustrate various points
made in the previous sections. The datasets are not designed to be realistic.

\begin{figure}[t]
	\begin{centering}
		\begin{tabular}{cc}
			\includegraphics[width=0.45\textwidth]{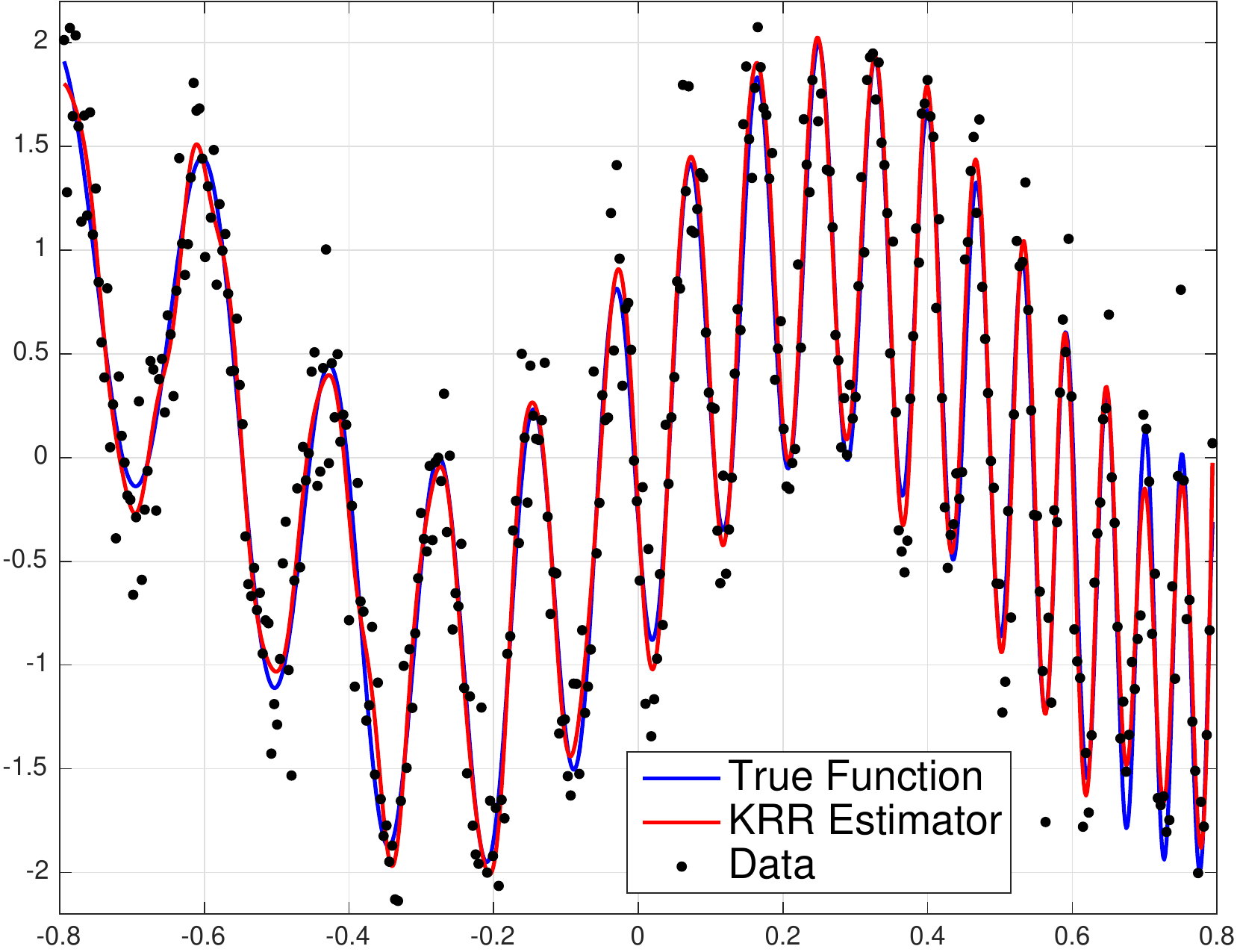} & \includegraphics[width=0.45\textwidth]{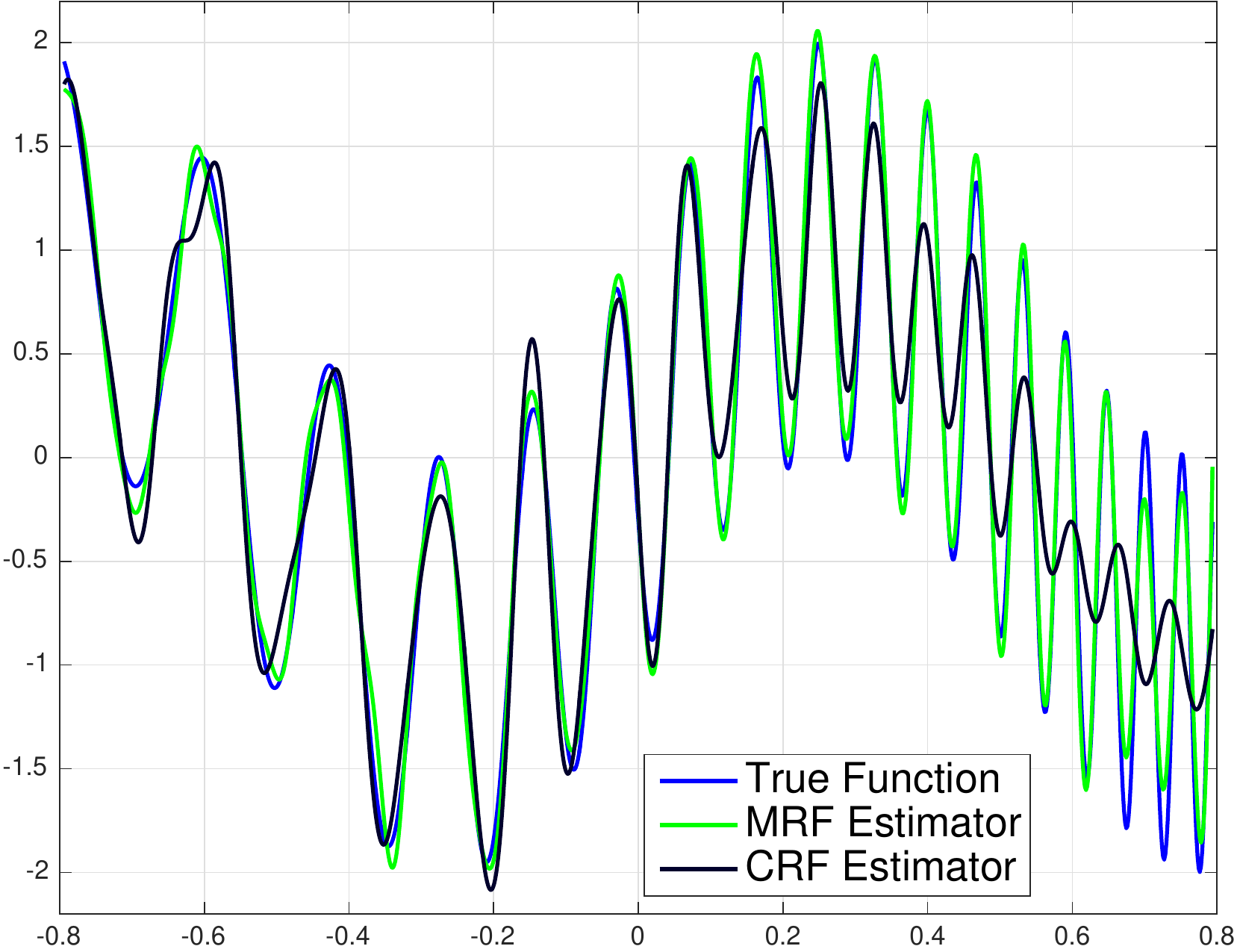}\\
		\end{tabular}
		\par\end{centering}
	\protect\caption{\label{fig:wiggly1} Results on the wiggly function~\eqref{eq:wiggly}. Left graph shows the function itself, the noisy samples and the KRR estimator. Right graph shows both the classical random Fourier features estimator (labeled {\em CRF}) and a modified random Fourier features estimator (labeled {\em MRF}).}
\end{figure}

In the first experiment, we noisily sample from the function\footnote{This function was taken from Trefethen's book on approximation theory~\cite{Trefethen12}.}
\begin{eqnarray}
\label{eq:wiggly}
f^\star(x) = \sin(6x) + \sin(60\exp(x))\,.
\end{eqnarray}
The function is sampled on a fine 400-point uniform grid spanning $[-5/2\pi, +5/2\pi]$. Samples are generated using the formula
$$
y_i = f^\star(x_i) + \nu_i\,.
$$
In the above, $x_i$ is a grid point, $y_i$ is the corresponding noisy sample, and $\{\nu_i\}$'s are i.i.d noise terms, distributed as normal variables with variance $\sigma^2_\nu = 0.3^2$. Figure~\ref{fig:wiggly1} (left) shows $f^\star$ and the noisy samples.

Figure~\ref{fig:wiggly1} (left) also shows the KRR estimator, obtained using the Gaussian kernel with $\sigma=0.0280443$ and regularization parameter $\lambda=0.00618936$. These values where obtained by optimizing the estimator's risk, which we can compute due to our knowledge of $f^\star$ and the noise distribution, using MATLAB's \texttt{fminsearch} function starting from $\sigma_0 = 1$ and $\lambda_0=1$.

Figure~\ref{fig:wiggly1} (right) shows the estimator obtained using $s=200$ classical random Fourier features (labeled {\em CRF}) and $s=200$ modified random Fourier features (labeled {\em MRF}). For modified random Fourier features, we did not use the analytical construction in \S\ref{sec:improved}, but rather use a uniform distribution on $[-\gamma/\sigma, \gamma/\sigma]$, treating $\gamma$ as a parameter (we use $\gamma = 4$). Technically, the support of the distribution is not the entire real line (as required), so the expected value of the substitute kernel is not identical to that of the true kernel, however the weight of values which are not in the support is negligible for large enough values of $\gamma$.  We clearly see that while classical random Fourier features fails to estimate the higher frequency areas of $f^\star$, modified random Fourier features approximates them well (close to the quality of the KRR estimator).

\begin{table*}[t]
	\protect\caption{\label{tab:wiggly} Comparison of estimators for the wiggly function~\eqref{eq:wiggly}.}
	\begin{center}
		\begin{tabular}{c|ccccc}
			{\bf Estimator} &  {\bf ${\cal R}(f)$} & {\bf $\sum^n_{i=1}(f^\star(x_i) - f(x_i))^2$} & {\bf $s_\lambda$} & {\bf $\frac{\FNormS{\matK - \matZ\matZ^\conj}}{\FNormS{\matK}}$} & {\bf $\kappa(\matK + \lambda \matI, \matZ \matZ^\conj + \lambda \matI)$ } \tabularnewline
			\hline
			KRR & 0.0164 & 0.0116 & $s_\lambda(\matK)=73.1$ & &   \tabularnewline
			CRF & 0.1474 & 0.1511 & $s_\lambda(\matZ\matZ^\conj)=46.2$ & 0.17 &  1458.6 \tabularnewline
			MRF & 0.0178 & 0.0120 & $s_\lambda(\matZ\matZ^\conj)=68.8$ & 0.31 &  56.2 \tabularnewline
		\end{tabular}
	\end{center}
\end{table*}

Table~\ref{tab:wiggly} compares the estimators quantitatively. We clearly see that the MRF estimator enjoys both a lower risk and
a lower actual in-sample error, when compared to the CRF estimator. MRF's risk is close to the KRR's risk. It is important to note that while the $\matZ$ produced by MRF leads to a better estimator, when it comes to approximating the kernel matrix entry-wise (measured by
$\FNormS{\matK - \matZ\matZ^\conj}/\FNormS{\matK}$), CRF produces a better approximation. This illustrates that entrywise error rates are not predictive of approximation quality. In contrast, the generalized condition number (ratio between largest and smallest generalized eigenvalues) of $(\matK + \lambda \matI, \matZ\matZ^\conj + \lambda \matI)$, closely related to spectral approximation guarantees, is much more predictive  of estimator quality (although additional experiments reveal that it is not completely predictive).

\begin{figure}
	\begin{centering}
		\begin{tabular}{ccc}
			\includegraphics[width=0.30\textwidth]{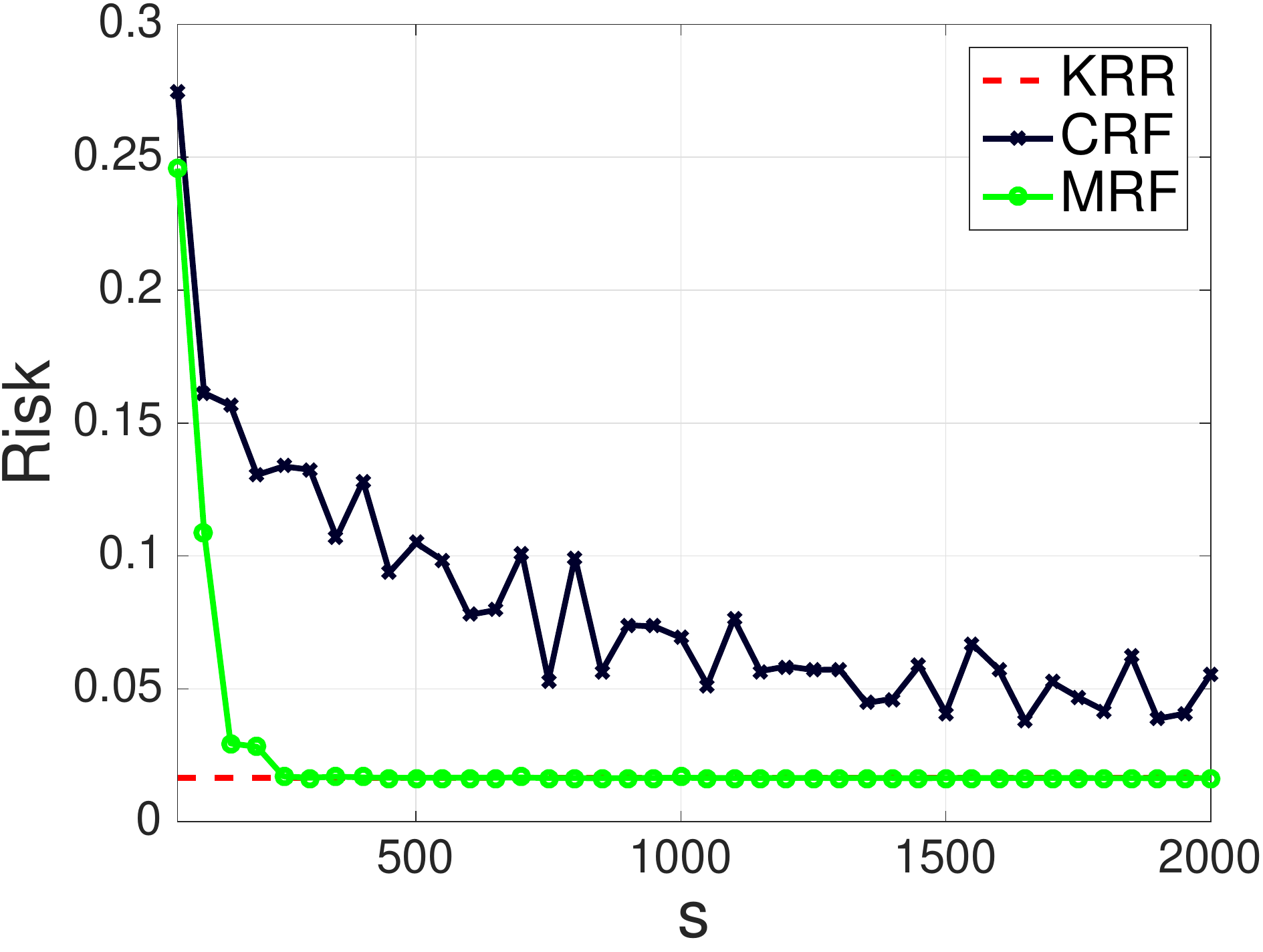} &
			\includegraphics[width=0.30\textwidth]{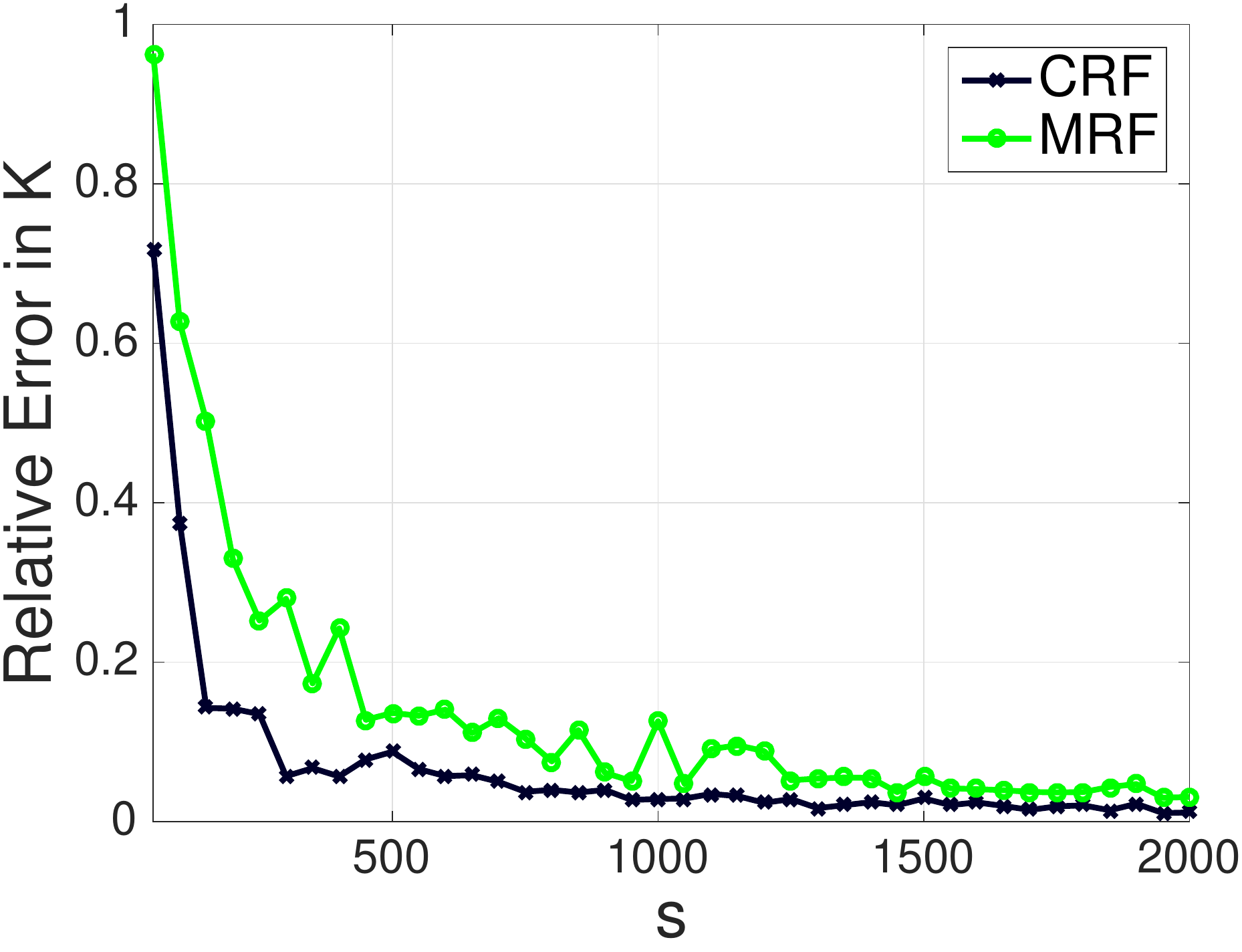} & \includegraphics[width=0.30\textwidth]{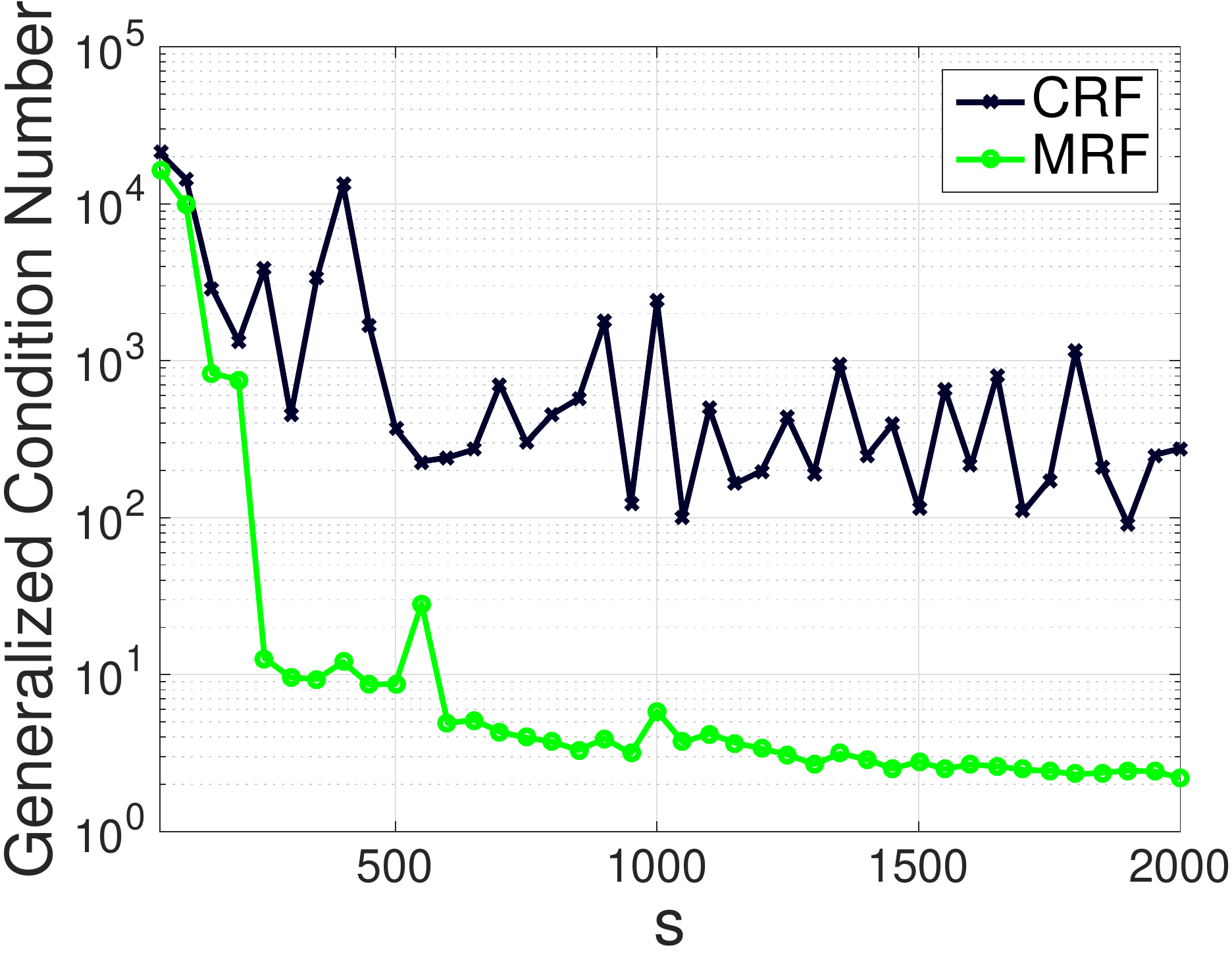}\\
		\end{tabular}
		\par\end{centering}
	\protect\caption{\label{fig:wiggly2} Assessing estimator's quality when varying $s$.}
\end{figure}

This is further examined in Figure~\ref{fig:wiggly2}, where we vary $s$ and assess the estimator's quality. The leftmost graph shows the risk. While the MRF's risk quickly converges to the KRR risk, CRF's risk reduces very slowly, practically stagnating for higher $s$. Note that even when $s>n$ CRF's risk is larger than KRR's risk! This is while the entry-wise error of CRF consistently continues to reduce and is consistently better than MRF's (middle figure). In contrast, MRF's generalized condition number is consistently lower than CRF's (rightmost figure). MRF's generalized condition number continues to reduce when $s$ grows, while CRF's stagnates.

In Figure~\ref{fig:wiggly2D} we report experiments with the two dimensional function
\begin{eqnarray}
\label{eq:wiggly2D}
f^\star(x,z) = (\sin(x) + \sin(10\exp(x)))(\sin(z) + \sin(10\exp(z)))\,.
\end{eqnarray}
We sample points on a $40\times40$ uniform grid (total of $n=1600$ points), and use $\sigma = 0.181167,\lambda = 0.00106475$. We use a fixed
$s=400$.  The MRF estimator is very close to the KRR estimator, while the CRF estimator misses or distorts some of the features of the
function.

\begin{figure}[H]
	\begin{centering}
		\begin{tabular}{c}
			\includegraphics[width=0.9\textwidth]{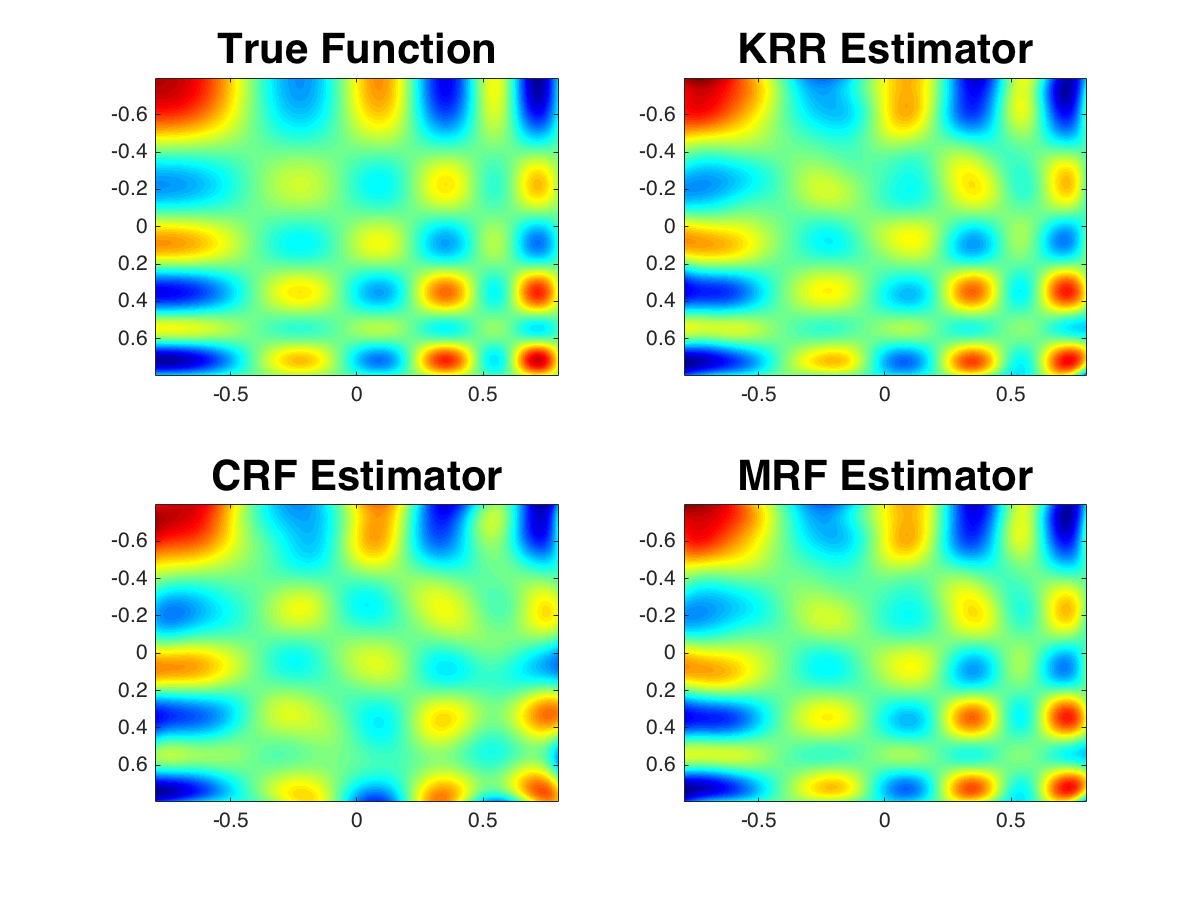}
		\end{tabular}	
		\par\end{centering}
	\protect\caption{\label{fig:wiggly2D} Approximation of the two dimensional wiggly function~\eqref{eq:wiggly2D}.}
\end{figure}

\section{Conclusions}

We have analyzed random Fourier features from a spectral matrix approximation point of view. We show both
positive and negative results regarding the use of random Fourier features to obtain spectral
approximation of the kernel matrix. Our study is well motivated by the fact that spectral approximation
bounds lead to statistical guarantees for KRR. Althouhgh we do not discuss in detail, our results can also be extended to bounds for other kernel-based methods such as kernel $k$-means and kernel PCA via recent results~\cite{CohenMuscoMusco17, MuscoMusco16}.

Our results expose a potential sub-optimality of random Fourier features, and also show that a
variant which uses a specially crafted feature sampling distribution can achieve better theoretical properties.
However, our construction is mostly theoretical due to an exponential dependence on the data dimension.
Nevertheless, our results motivate further efforts to improve random Fourier features by devising improved sampling
distributions.

%From a conceptual point of view, our results are based on worst-case analysis of the leverage scores with respect to the data points.
%It is natural to try to replace the worst-case analysis with an analysis that assumes the data points are sampled from
%some distribution (e.g., as was recently done by~\cite{Bach15}). This is a common assumption in learning theory. We leave this for
%future work as well.

%From a conceptual point of view, our results are based on worst-case analysis with respect to the data points.
%It can be informative to perform the analysis under the assumption that the points are sampled from
%some distribution (e.g., as was recently done by~\cite{Bach15}). This is a common assumption in learning theory. We leave this for
%future work as well.

From a conceptual point of view, our results are based on worst-case analysis of the leverage scores with respect to the data points.
It is natural to try to replace the worst-case analysis with an analysis that assumes the data points are sampled from
some distribution (e.g., as was recently done by Bach~\cite{Bach15}). We leave this for
future work as well.

\section*{Acknowledgements}

The authors thank Arturs Backurs helpful discussions at early stages of this project, and Jarek Blasiok for pointing out a typo in an earlier version of the manuscript.
Haim Avron acknowledges the support from the XDATA program of the Defense Advanced Research Projects
Agency (DARPA), administered through Air Force Research Laboratory contract FA8750-12-C-0323 and an IBM Faculty Award.
Cameron Musco acknowledges the support by NSF Graduate Research Fellowship, AFOSR grant FA9550-13-1-0042 and the NSF Center for Science of Information.

%\textbf{Do not} include acknowledgements in the initial version of
%the paper submitted for blind review.


\bibliography{paper_arxiv}

\newcommand{\etalchar}[1]{$^{#1}$}
\begin{thebibliography}{AKM{\etalchar{+}}17}

\bibitem[ACW17]{ACW16}
Haim Avron, Kenneth~L. Clarkson, and David~P. Woodruff.
\newblock Faster kernel ridge regression using sketching and preconditioning.
\newblock {\em SIAM Journal on Matrix Analysis and Applications}, to appear,
  2017.

\bibitem[AKM{\etalchar{+}}17]{ICML_paper}
Haim Avron, Michael Kapralov, Cameron Musco, Christopher Musco, Ameya
  Velingker, and Amir Zandieh.
\newblock Random fourier features for kernel ridge regression: Approximation
  bounds and statistical guarantees.
\newblock In {\em International Conference on Machine Learning (ICML)}, 2017.

\bibitem[AM15]{AlaouiMahoney15}
Ahmed~El Alaoui and Michael~W. Mahoney.
\newblock Fast randomized kernel ridge regression with statistical guarantees.
\newblock In {\em Neural Information Processing Systems (NIPS)}, 2015.

\bibitem[ANW14]{ANW14}
Haim Avron, Huy Nguyen, and David Woodruff.
\newblock Subspace embeddings for the polynomial kernel.
\newblock In {\em Neural Information Processing Systems (NIPS)}, 2014.

\bibitem[Bac13]{Bach13}
Francis~R. Bach.
\newblock Sharp analysis of low-rank kernel matrix approximations.
\newblock In {\em Conference on Learning Theory (COLT)}, 2013.

\bibitem[Bac17]{Bach15}
Francis Bach.
\newblock On the equivalence between kernel quadrature rules and random feature
  expansions.
\newblock {\em Journal of Machine Learning Research}, 18(21):1--38, 2017.

\bibitem[CDV07]{CaponnettoDeVito2007}
A.~Caponnetto and E.~De~Vito.
\newblock Optimal rates for the regularized least-squares algorithm.
\newblock {\em Foundations of Computational Mathematics}, 7(3):331--368, 2007.

\bibitem[CMM17]{CohenMuscoMusco17}
Michael~B. Cohen, Cameron Musco, and Christopher Musco.
\newblock Input sparsity time low-rank approximation via ridge leverage score
  sampling.
\newblock In {\em Proceedings of the Twenty-Eighth Annual ACM-SIAM Symposium on
  Discrete Algorithms}, SODA '17, pages 1758--1777, Philadelphia, PA, USA,
  2017. Society for Industrial and Applied Mathematics.

\bibitem[COCF16]{CutajarEtAl16}
Kurt Cutajar, Michael Osborne, John Cunningham, and Maurizio Filippone.
\newblock Preconditioning kernel matrices.
\newblock In {\em International Conference on Machine Learning (ICML)}, 2016.

\bibitem[Fel68]{feller}
William Feller.
\newblock {\em An introduction to probability theory and its applications.
  Volume 1}.
\newblock Wiley series in probability and mathematical statistics. John Wiley
  \& sons, New York, Chichester, Brisbane, 1968.

\bibitem[MD09]{MahoneyDrineas09}
Michael~W. Mahoney and Petros Drineas.
\newblock {CUR} matrix decompositions for improved data analysis.
\newblock {\em Proceedings of the National Academy of Sciences},
  106(3):697--702, 2009.

\bibitem[MM17]{MuscoMusco16}
Cameron Musco and Christopher Musco.
\newblock Recursive sampling for the {N}ystr\"{o}m method.
\newblock In {\em Neural Information Processing Systems (NIPS)}, 2017.

\bibitem[Oga88]{Ogawa88}
Hidemitsu Ogawa.
\newblock An operator pseudo-inversion lemma.
\newblock {\em SIAM Journal on Applied Mathematics}, 48(6):1527--1531, 1988.

\bibitem[RCR15]{rudi2015less}
Alessandro Rudi, Raffaello Camoriano, and Lorenzo Rosasco.
\newblock Less is more: Nystr{\"o}m computational regularization.
\newblock In {\em Neural Information Processing Systems (NIPS)}, 2015.

\bibitem[RCR17]{RudiEtAl16}
Alessandro Rudi, Raffaello Camoriano, and Lorenzo Rosasco.
\newblock Generalization properties of learning with random features.
\newblock In {\em Neural Information Processing Systems (NIPS)}, 2017.

\bibitem[RR07]{RahimiRecht07}
A.~Rahimi and B.~Recht.
\newblock Random features for large-scale kernel machines.
\newblock In {\em Neural Information Processing Systems (NIPS)}, 2007.

\bibitem[RR08]{RahimiRecht09}
Ali Rahimi and Benjamin Recht.
\newblock Weighted sums of random kitchen sinks: Replacing minimization with
  randomization in learning.
\newblock In {\em Neural Information Processing Systems (NIPS)}, 2008.

\bibitem[Tre12]{Trefethen12}
Lloyd~N. Trefethen.
\newblock {\em Approximation Theory and Approximation Practice}.
\newblock Society for Industrial and Applied Mathematics, Philadelphia, PA,
  USA, 2012.

\bibitem[Tro15]{Tropp15}
Joel~A. Tropp.
\newblock An introduction to matrix concentration inequalities.
\newblock {\em Foundations and Trends® in Machine Learning}, 8(1-2):1--230,
  2015.

\bibitem[Woo14]{Woodruff14}
David~P. Woodruff.
\newblock Sketching as a tool for numerical linear algebra.
\newblock {\em Foundations and Trends in Theoretical Computer Science},
  10(1--2):1--157, October 2014.

\bibitem[ZDW15]{ZhangDuchiWainwright15}
Yuchen Zhang, John Duchi, and Martin Wainwright.
\newblock Divide and conquer kernel ridge regression: A distributed algorithm
  with minimax optimal rates.
\newblock {\em J. Mach. Learn. Res.}, 16(1):3299--3340, January 2015.

\end{thebibliography}
\bibliographystyle{alpha}

%\pagebreak{}

\appendix

\section{Matrix Approximation by Random Sampling:\\An Intrinsic Dimension Bound}\label{sec:matrix-concentration}

The following Corollary is essentially a restatement of Corollary 7.3.3 from~\cite{Tropp15}. However,
the minimum $t$ in the following statement is much lower than the bound that appears in~\cite{Tropp15} which is unnecessarily loose (possibly, a typo
in~\cite{Tropp15}). For completeness, we include a proof.

\noindent \textbf{Lemma \ref{lem:tropp-correct}} (Restated)
\emph{
Let $\matB$ be a fixed $d_1 \times d_2$ matrix. Construct a $d_1 \times d_2$ random matrix $\matR$ that satisfies
$$\Expect{\matR} = \matB~~~~\textrm{and}~~~~\TNorm{\matR} \leq L.$$
Let $\matM_1$ and $\matM_2$ be semidefinite upper bounds for the expected squares:
$$\Expect{\matR \matR^\conj} \preceq \matM_1~~~~\textrm{and}~~~~\Expect{\matR^\conj \matR} \preceq \matM_2.$$
Define the quantities
$$
m = \max(\TNorm{\matM_1}, \TNorm{\matM_2})~~~~\textrm{and}~~~~d = (\Trace{\matM_1} + \Trace{\matM_2}) / m.
$$
Form the matrix sampling estimator
$$
\bar{\matR}_n = \frac{1}{n}\sum_{k=1}^{n} \matR_k
$$
where each $\matR_k$ is an independent copy of $\matR$.
Then, for all $t \geq \sqrt{m/n} + 2L/3n$,
\begin{equation}
\label{eq:tropp-concentration}
\Pr(\TNorm{\bar{\matR}_n - \matB} \geq t) \leq 4d\exp\left( \frac{-nt^2/2}{m + 2Lt/3} \right).
\end{equation}
}
\begin{proof}
	The proof mirrors the proof of Corollary 6.2.1 in~\cite{Tropp15}, using Theorem 7.3.1 instead of Theorem 6.1.1 (both
	from \cite{Tropp15}).
	Since $\Expect{\matR}=\matB$, we can write
	$$
	\matZ \equiv \bar{\matR}_n - \matB = \frac{1}{n}\sum^n_{k=1} (\matR_k - \Expect{\matR}) = \sum^n_{k=1}\matS_k,
	$$
	where we have define $\matS_k \equiv n^{-1}(\matR_k - \Expect{\matR})$. These random matrices are i.i.d and each
	has zero mean. 	
	Now, we can bound each of the summands:
	$$ \TNorm{\matS_k} \leq \frac{1}{n}(\TNorm{\matR_k} + \TNorm{\Expect{\matR}}) \leq \frac{1}{n}(\TNorm{\matR_k} + \Expect{\TNorm{\matR}}) \leq \frac{2L}{n},$$
	where the first inequality is the triangle inequality and the second is Jensen's inequality.
	
	To find semidefinite upper bounds $\matV_1$ and $\matV_2$ on the matrix-valued variances we note that
	\begin{eqnarray*}
		\Expect{\matS_1 \matS^\conj_1} & = & n^{-2} \Expect{(\matR - \Expect{\matR})(\matR - \Expect{\matR})^\conj}\\
		& = & n^{-2}\left( \Expect{\matR \matR^\conj} - \Expect{\matR}\Expect{\matR}^\conj \right) \\
		& \preceq & n^{-2} \Expect{\matR\matR^\conj}.
	\end{eqnarray*}
	Likewise, $\Expect{\matS^\conj_1 \matS_1} \preceq n^{-2} \Expect{\matR^\conj\matR}$. Since the summands are i.i.d, if we define
	$\matV_1 \equiv n^{-1} \matM_1$ and $\matV_2 \equiv n^{-1} \matM_2$,  we have $\Expect{\matZ\matZ^\conj} \preceq \matV_1$
	and $\Expect{\matZ^\conj\matZ} \preceq \matV_2$.
	
	We now calculate,
	$$
	\nu \equiv \max(\TNorm{\matV_1}, \TNorm{\matV_2}) = \frac{m}{n}
	$$
	and
	$$\frac{\Trace{\matV_1} + \Trace{\matV_2}}{\max(\TNorm{\matV_1}, \TNorm{\matV_2})} = d\,.$$
	Noticing, that the condition $t \geq \sqrt{m/n} + 2L/3n$ meets the required lower bound in Theorem 7.3.1 in \cite{Tropp15}
	we can now apply this theorem, which along with the above calculations translates to~\eqref{eq:tropp-concentration}.
\end{proof}

\section{Fourier Transforms and Gaussian Distributions}

Our upper and lower bound analysis relies predominantly on Fourier analysis and properties of the
Gaussian distribution.
In this section we introduce some additional notation and state some useful facts about these.

\subsection{Properties of Fourier Transforms}
\begin{defn}[Fourier Transform]
	The \emph{Fourier transform} of a continuous function $f : \RR^d \to \CC$ in
	$L_1(\RR^n)$ is defined to be the function $\mathcal{F}f: \RR^d \to \CC$ as
	follows:
	\[
	(\mathcal{F}f)(\bs{\xi}) =  \int_{\RR^d} f(\bv{t}) e^{-2\pi i \bv{t}^T \bs{\xi}}\,d\bv{t}.
	\]
	We also sometimes use the notation $\hat{f}$ for the Fourier transform of $f$. We often informally refer to $f$ as representing the function in \emph{time domain} and $\hat{f}$ as representing the function in \emph{frequency domain}.
\end{defn}

The original function $f$ can also be obtained from $\hat{f}$ by the
\emph{inverse Fourier transform}:
\begin{equation*}
f(\bv{t}) = \int_{\RR^d} \hat{f}(\bs{\xi}) e^{2\pi i \bs{\xi}^T \bv{t}}\,d\bs{\xi}
\end{equation*}
\begin{defn}[Convolution]
	The \emph{convolution} of two functions $f:\RR^d\to\CC$ and $g:\RR^d\to\CC$ is
	defined to be the function $(f*g):\RR^d\to\CC$ given by
	\begin{equation*}
	(f*g)(\veta) = \int_{\RR^d} f(\bv{t})g(\veta-\bv{t})\,d\bv{t}.
	\end{equation*}
\end{defn}
The convolution theorem shows that the Fourier transform of the convolution of
two functions is simply the product of the individual Fourier transforms:
\begin{claim}[Convolution Theorem] \label{claim:convthm}
	Given functions $f:\RR^d\to\CC$ and $g:\RR^d\to\CC$ whose convolution is $h =
	f*g$, we have
	\[
	\hat{h}(\bs{\xi}) = \hat{f}(\bs{\xi})\cdot\hat{g}(\bs{\xi})
	\]
	for all $\bs{\xi}\in\RR^d$.
\end{claim}
% Next, we present the \emph{Plancherel theorem}, which states that the inner
% product of two functions in time domain is equal to the inner product of the
% same functions in frequency domain:
% \begin{claim}[Plancherel Theorem]
%  If $f:\RR^d\to\CC$ and $g:\RR^d\to\CC$ are functions in $L_1(\RR^d)\cap
% L_2(\RR^d)$, then
%  \[
%   \int_{\RR^d} f(\bv{t}){g(\bv{t})}^\conj\,d\bv{t} = \int_{\RR^d}
% \hat{f}(\bs{\xi}) {\hat{g}(\bs{\xi})}^\conj \,d\bs{\xi}
%  \]
% \end{claim}
% If we set $g = f$ in the Plancherel theorem, then we obtain Parseval's theorem
% as a corollary.
% \begin{claim}[Parseval's Theorem]
%  If $f:\RR^d\to\CC$ is a function in $L_1(\RR^d)\cap L_2(\RR^d)$, then
%  \[
%   \int_{\RR^d} |f(\bv{t})|^2\,d\bv{t} = \int_{\RR^d} |\hat{f}(\bs{\xi})|^2\,d\bs{\xi}.
%  \]
% \end{claim}

We now define the \emph{rectangle function} and \emph{normalized sinc function}, which
we use extensively in our analysis.
\begin{defn}[Rectangle Function] \label{def:rect}
	We define the 1-dimensional \emph{rectangle function} $\rect_{1,a}:\RR\to\CC$ as
	\[
	\rect_{1,a}(x) = \begin{cases} 0\qquad &\text{if $|x| > a/2$}\\ \frac{1}{2}
	&\text{if $|x| = a/2$}\\ 1\qquad &\text{if $|x| < a/2$} \end{cases}.
	\]
	For any $d > 1$, we define the $d$-dimensional \emph{rectangle function} $\rect_{d,a}:
	\RR^d\to\CC$ as
	\[
	\rect_{d,a}(\x) = \prod_{j=1}^d \rect_{1,a}(x_j).
	\]
	If $d$ is understood from context, we often omit $d$ and write $\rect_a$. Moreover, if $a=1$
	(and $d$ is understood from context), we often omit all subscripts and simply write $\rect$.
\end{defn}
\begin{defn}[Normalized Sinc Function] \label{def:sinc}
	We define the $d$-dimensional \emph{normalized sinc function} $\mathrm{sinc}_d:\RR^d\to\CC$ as
	\[
	\mathrm{sinc}_d(\x) = \prod_{j=1}^d \frac{\sin(\pi x_j)}{\pi x_j}.
	\]
	We often omit the subscript and simply write $\mathrm{sinc}$.
\end{defn}
It is well known that the Fourier transform of the rectangle
function (with $a=1$) is the normalized sinc function:
\[
\mathcal{F}(\rect_d) = \mathrm{sinc}_d.
\]

We use $\delta_d$ to denote the d-dimensional \emph{Dirac delta function}.
The Dirac delta function satisfies the following useful property for any function
$f$:
\[
\int_{\RR^d} f(\x) \delta_d(\x-\a)\,d\x = f(\a),
\]
i.e. the integral of a function multiplied by a shifted Dirac delta functions
picks out the value of the function at a particular point. Thus, it is not hard to see that the Fourier transform of a $\delta_d$ is the constant function which is $1$ everywhere:
\[
(\Fc\delta_d)(\vxi) = \int_{\RR^d} e^{-2\pi i \t^T \vxi}\cdot \delta_d(\t)\,d\t
= e^{-2\pi i \cdot 0^T \cdot \vxi} = 1
\]
for all $\vxi$. Similarly, the Fourier transform of a shifted delta function is as follows:
\[
(\Fc\delta(\cdot - \a))(\vxi) = \int_{\RR^d} e^{-2\pi i \t^T \vxi}\cdot \delta_d(\t-\a)\,d\t = e^{-2\pi i \a^T\vxi}.
\]
Moreover, it is not hard to see that convolving a function by a
shifted delta function results in a shift of the original function:
\[
(f * \delta_d(\cdot - \a))(\x) = f(\x-\a).
\]
Thus, by the convolution theorem, we obtain the following identity:
\begin{claim}
	Given a function $f:\RR^d\to\CC$, we have
	\[
	(\Fc f(\cdot - \a))(\vxi) = (\Fc (f * \delta_d(\cdot - \a)))(\vxi) = \hat{f}(\vxi)
	\cdot e^{-2\pi i \a^T \vxi}.
	\]
\end{claim}
Similarly,
\begin{claim}
	Given a function $f:\RR^d\to\CC$, we have
	\[
	(\Fc (f(\x)\cdot e^{2\pi i \a^T \x}))(\vxi) = \hat{f} (\vxi - \a).
	\]
\end{claim}

Finally, we introduce a useful function known as the \emph{Dirac comb function}:
\begin{defn}
	The d-dimensional \emph{Dirac comb function} with period $T$ is defined as
	$f$ satisfying
	\[
	f(\x) =  \sum_{\mathbf{j} \in \ZZ^d} \delta(x - \mathbf{j}T).
	\]
\end{defn}
It is a standard fact that the Fourier transform of a Dirac comb function is
another Dirac comb function which is scaled and has the inverse period:
\begin{claim}\label{diracF}
	Let
	\[
	f(\x) = \sum_{\mathbf{j} \in \ZZ^d} \delta(x - \mathbf{j}T)
	\]
	be the d-dimensional Dirac comb function with period $T$. Then,
	\[
	(\Fc f)(\vxi) = \frac{1}{T^d}  \sum_{\mathbf{j} \in \ZZ^d} \delta\left(\xi - \frac{\mathbf{j}}{T}\right).
	\]
\end{claim}
We use the Dirac comb function in our lower bound constructions.

\begin{claim}\label{claim:nyquist}
Given a function $f:\RR^d\to\CC$, we have:
\begin{align}	
\Fc \left( f(\cdot)  \sum_{\mathbf{j}\in\ZZ^d}
\delta_d(\cdot-T\mathbf{j}) \right)(\bs{\xi}) &= \sum_{\mathbf{j}\in\ZZ^d} T^{-d} \Fc(f)(\bs{\xi}-T^{-1}\mathbf{j}).
\end{align}
\end{claim}

\subsection{Properties of Gaussian Distributions}
The following is a standard fact about the cumulative distribution function of the standard
Gaussian distribution:
\begin{claim}[\cite{feller}] \label{claim:cdfnormal}
	For any $x > 0$, we have
	\[
	\frac{1}{\sqrt{2\pi}}\int_{x}^\infty e^{-t^2/2}\,dt \leq
	\frac{e^{-x^2/2}}{x\sqrt{2\pi}}.
	\]
	Moreover, as a direct consequence, for any $\sigma, x > 0$, we have that
	\[
	\frac{1}{\sqrt{2\pi}\sigma} \int_{x}^\infty e^{-t^2/2\sigma^2}\,dt \leq
	\frac{\sigma e^{-x^2/2\sigma^2}}{x\sqrt{2\pi}}.
	\]
	Also, if $x\geq 1$, then
	\[
	\left(\frac{1}{x}-\frac{1}{x^3}\right)\cdot\frac{1}{\sqrt{2\pi}}e^{-x^2/2}
	\leq \frac{1}{\sqrt{2\pi}} \int_x^\infty e^{-t^2}\,dt.
	\]
\end{claim}
Next, we prove the following claim, which provides tail bounds for modified Gaussians:
\begin{claim} \label{claim:cdfnormalhighdim}
	We have the following results:
	\begin{enumerate}
		\item For any $x > 0$ and $d=1$, we have
		\[
		\int_{x}^\infty t^d e^{-t^2/2}\,dt = e^{-x^2/2}.
		\]
		\item For any $x > 0$ and odd integer $d > 1$, we have
		\[
		\int_{x}^\infty t^d e^{-t^2/2}\,dt \geq  (d-1)(d-3)\cdots 2 \cdot e^{-x^2/2}.
		\]
		\item For any $x>0$ and even integer $d >  1$, we have
		\[
		\int_{x}^\infty t^d e^{-t^2/2}\,dt \geq (d-1)(d-3)\cdots 3 \cdot xe^{-x^2/2}.
		\]
		\item For any $x>0$ and integer $d\geq 1$, we have
		\[
		\int_{x}^\infty t^d e^{-t^2/2}\,dt \geq x^{d-1} e^{-x^2/2}.
		\]
	\end{enumerate}
\end{claim}
\begin{proof}
	Part (1) is simple calculation.
	
	If $d$ is odd, say $d = 2a+1$, then by repeated use of integration by parts,
	\begin{align}
	\int_{x}^\infty t^d e^{-t^2/2}\,dt &= \sum_{j=0}^{a-1} \left( \prod_{k=1}^j (d - (2k-1)) \right) x^{d-(2j+1)}
	e^{-x^2/2} + (d-1)(d-3)\cdots 2 \int_{x}^\infty t e^{-t^2/2}\,dt \label{eq:oddlbd}\\
	&\geq (d-1)(d-3)\cdots 2 \int_{x}^\infty t e^{-t^2/2}\,dt \nonumber\\
	&= (d-1)(d-3)\cdots 2 \cdot e^{-x^2/2}, \nonumber
	\end{align}
	which establishes part (2).
	
	On the other hand, if $d$ is even, say $d = 2a$, then we have
	\begin{align}
	\int_{x}^\infty t^d e^{-t^2/2}\,dt &= \sum_{j=0}^{a-1} \left( \prod_{k=1}^j (d - (2k-1)) \right) x^{d-(2j+1)}
	e^{-x^2/2} + (d-1)(d-3)\cdots 1 \int_{x}^\infty e^{-t^2/2}\,dt \label{eq:evenlbd}\\
	&\geq (d-1)(d-3)\cdots 3 \cdot xe^{-x^2/2}, \nonumber
	\end{align}
	which establishes part (3) of the claim.
	
	Finally, note that \eqref{eq:oddlbd} and \eqref{eq:evenlbd} are both bounded
	from below by $x^{d-1} e^{-x^2/2}$ (since this is the first term of the summation in both expressions), which
	establishes part (4).
\end{proof}

We also need the following property about Gaussian samples.
\begin{claim}\label{claim:gaussiansamp}
	Let $t\geq 10$, and $a_1, a_2, \dots, a_t$ be sampled according to the
	Gaussian distribution
	given by probability density function $\frac{1}{\sqrt{2\pi}}e^{-x^2/2}$.
	Let $a^* = \max_{1\leq j\leq t} |a_j|$. Then,
	\[
	\Pr\left[\frac{1}{\sqrt{2\pi}} e^{-{a^*}^2/2} \leq
	\frac{8\sqrt{\log
			t}}{t}\right]
	\geq 1 - e^{-1} \geq \frac{1}{2}.
	\]
\end{claim}
\begin{proof}
	Choose $q_1$ such that
	\begin{equation}
	\int_{q_1}^\infty \frac{1}{\sqrt{2\pi}} e^{-x^2/2}\,dx = \frac{1}{t}.
	\label{eq:cdf1}
	\end{equation}
	Note that by Claim~\ref{claim:cdfnormal}, we have
	\[
	\frac{1}{\sqrt{2\pi}} \int_{2\sqrt{\log t}}^\infty  e^{-x^2/2}\,dx \leq
	\frac{1}{2\sqrt{2\pi} t^2\sqrt{\log t}} \leq \frac{1}{t}.
	\]
	Thus, $q_1 \leq 2\sqrt{\log t}$.
	
	Also, since $\frac{1}{t} \leq \frac{1}{4}$, we have that $q_1 \geq
	\frac{6}{5}$. Thus, by another application of Claim~\ref{claim:cdfnormal},
	\[
	\frac{1}{t} = \frac{1}{\sqrt{2\pi}} \int_{q_1}^\infty  e^{-x^2/2}\,dx \geq
	\left(\frac{1}{q_1}-\frac{1}{q_1^3}\right) \frac{1}{\sqrt{2\pi}} e^{-q_1^2/2}
	\geq \frac{1}{4q_1}\cdot\frac{1}{\sqrt{2\pi}}e^{-q_1^2/2},
	\]
	and so,
	\[
	\frac{1}{\sqrt{2\pi}}e^{-q_1^2/2}\leq \frac{4q_1}{t} \leq \frac{8\sqrt{\log
			t}}{t}.
	\]
	Therefore,
	\begin{align*}
	\Pr\left[\frac{1}{\sqrt{2\pi}}e^{-{a^*}^2/2} \leq \frac{8\sqrt{\log
			t}}{t} \right] & \geq \Pr\left[\frac{1}{\sqrt{2\pi}}e^{-{a^*}^2/2} \leq \frac{1}{\sqrt{2\pi}}e^{-q_1^2/2} \right] \\
	& = \Pr[{a^*} \geq q_1]\\
	&= 1 - \left(1-\frac{1}{t}\right)^t\\
	&\geq 1 - \frac{1}{e}\\
	&\geq \frac{1}{2},
	\end{align*}
	as desired.
\end{proof}
We extend the above claim to an analogous claim for $d$-dimensional Gaussians, where $d > 1$:
\begin{claim} \label{claim:gaussiansamphighdim}
	Let $d\geq 2$, $t\geq 3$, and $\a_1, \a_2, \dots, \a_t \in \RR^d$ be sampled according to the
	$d$-dimensional Gaussian distribution given by the probability density function $\frac{1}{(\sqrt{2\pi})^d}
	e^{-\|\x\|_2^2/2}$. Let $\a^* = {\arg\max}_{\a \in \{\a_1, \a_2, \dots, \a_t\}} \|\a\|_2$.
	Then,
	\[
	\Pr\left[\frac{1}{(\sqrt{2\pi})^d} e^{-\|\a^*\|_2^2/2} \leq
	\frac{(d-1)^{\frac{d-1}{2}}}{(2\pi)^{d/2} (\log t)^{\frac{d-2}{2}} t}\right]
	\geq 1-e^{-1} \geq \frac{1}{2}.
	\]
\end{claim}
\begin{proof}
	Choose $q$ such that
	\begin{align*}
	\int_{\substack{\x \in \RR^d\\ \|\x\|_2 \geq q}} \frac{1}{(\sqrt{2\pi})^d} e^{-\|\x\|_2^2/2}\,d\x =
	\frac{1}{t}.
	\end{align*}
	Note that we have
	\begin{align}
	\frac{1}{t} = \int_{\substack{\x \in \RR^d\\ \|\x\|_2 \geq q}} \frac{1}{(\sqrt{2\pi})^d} e^{-\|\x\|_2^2/2}\,d\x
	&= \int_q^\infty \frac{1}{(\sqrt{2\pi})^d} e^{-r^2/2} \cdot d V_d r^{d-1}\,dr \nonumber\\
	&= \frac{dV_d}{(\sqrt{2\pi})^d} \int_q^\infty r^{d-1} e^{-r^2/2}\,dr, \label{eq:qtint}
	\end{align}
	where $V_d$ is  the volume of a $d$-sphere of radius 1. Note that if $d$ is even, then $V_d =
	\frac{\pi^{d/2}}{(d/2)!}$, and so, by part (2) of Claim~\ref{claim:cdfnormalhighdim}, we have
	\begin{align*}
	\int_{\substack{\x \in \RR^d\\ \|\x\|_2 \geq \sqrt{2 \log t}}} \frac{1}{(\sqrt{2\pi})^d} e^{-\|\x\|_2^2/2}\,d\x &= \frac{1}{(\sqrt{2\pi})^d} dV_d \int_{\sqrt{2 \log t}}^\infty r^{d-1} e^{-r^2/2} \,dr\\
	&\geq \frac{1}{(\sqrt{2\pi})^d} dV_d \left((d-2)(d-4)\cdots 2 \cdot e^{-(\sqrt{ 2 \log t})^2/2}\right)\\
	&= \frac{1}{(\sqrt{2\pi})^d} d\left(\frac{\pi^{d/2}}{(d/2)!}\right) \left((d-2)(d-4)\cdots 2 \cdot e^{-(\sqrt{2\log
			t})^2/2}\right)\\
	&= \frac{1}{t}.
	\end{align*}
	On the other hand, if $d$ is odd, then $V_d = \frac{2^{\frac{d+1}{2}} \cdot \pi^{\frac{d-1}{2}}}{1\cdot 3\cdots d}$,
	and so, by part (3) of Claim~\ref{claim:cdfnormalhighdim}, we have
	\begin{align*}
	\int_{\substack{\x \in \RR^d\\ \|\x\|_2 \geq \sqrt{2\log t}}} \frac{1}{(\sqrt{2\pi})^d} e^{-\|\x\|_2^2/2}\,d\x &= \frac{1}{(\sqrt{2\pi})^d} dV_d \int_{\sqrt{2\log t}}^\infty r^{d-1} e^{-r^2/2} \,dr\\
	&\geq \frac{1}{(\sqrt{2\pi})^d} dV_d \left((d-2)(d-4)\cdots 3 \cdot (\sqrt{2\log t})e^{-(\sqrt{2\log t})^2/2}\right)\\
	&= \frac{1}{(\sqrt{2\pi})^d} d\left(\frac{2^{\frac{d+1}{2}} \cdot \pi^{\frac{d-1}{2}}}{1\cdot 3\cdots d}\right)
	\left((d-2)(d-4)\cdots 3 \cdot (\sqrt{2\log t})e^{-(\sqrt{2\log t})^2/2}\right)\\
	&= \sqrt{\frac{4(\log t)}{\pi }} \cdot \frac{1}{t}\\
	&\geq \frac{1}{t}\text{~~~~~~~~~(since $t\geq 3$ by assumption)}.
	\end{align*}
	Thus, regardless of the parity of $d$, we have that
	\begin{equation}
	q \geq \sqrt{ 2 \log t}. \label{eq:qubd}
	\end{equation}
	Note that,
	$$ \int_{q}^\infty r^{d-1} e^{-r^2/2}\,dr \ge \int_{q}^\infty {q^{d-2}} r e^{-r^2/2}\,dr \ge {q^{d-2}} e^{-q^2/2}$$
	Hence, it follows form part (4) of Claim~\ref{claim:cdfnormalhighdim} as well as \eqref{eq:qubd} and
\eqref{eq:qtint} that
	\begin{align}
	e^{-q^2/2} &\leq \frac{1}{q^{d-2}} \int_{q}^\infty r^{d-1} e^{-r^2/2}\,dr \nonumber\\
	&= \frac{1}{q^{d-2}} \cdot \frac{(2\pi)^{d/2} }{dV_d t} \nonumber\\
	&= \frac{(2\pi)^{d/2} }{dV_d}(2\log t)^{-\frac{d-2}{2}} t^{-1}.\nonumber
	\end{align}
	If $d$ is even, we have
	$$
	\frac{(2\pi)^{d/2} }{dV_d} = \frac{2^{d/2}(d/2)!}{d} = \frac{\prod_{i=1}^{d/2} 2i}{d} = \prod_{i=1}^{d/2-1} 2i \leq (d-2)^{d/2 -1}\,.
	$$
	If $d$ is odd, we have
	$$
	\frac{(2\pi)^{d/2} }{dV_d} = \sqrt{\frac{\pi}{2}}1\cdot3\cdot\dots\cdot d-2 \leq (d-1)^{\frac{d-1}{2}}
	$$
	Either way, we have
	\begin{equation}
	e^{-q^2/2} \leq (d-1)^{\frac{d-1}{2}}(2\log t)^{-\frac{d-2}{2}} t^{-1}. \label{eq:qpdfbd}
	\end{equation}
	Therefore, by \eqref{eq:qpdfbd},
	\begin{align*}
	\Pr\left[\frac{1}{(\sqrt{2\pi})^d} e^{-\|\a^*\|_2^2/2} \leq \frac{(d-1)^{\frac{d-1}{2}}}{(2\pi)^{d/2} (\log t)^{\frac{d-2}{2}} t}
	\right] &\geq  \Pr\left[\|\a^*\|_2 \geq q \right]\\
	&= 1 - \left(1-\frac{1}{t}\right)^t\\
	&\geq 1-\frac{1}{e}\\
	&\geq \frac{1}{2},
	\end{align*}
	as desired.
\end{proof}

\section{Proof of Theorem~\ref{thm:lev-scores-ub}}\label{sec:lev-scores-ub}

It is easy to verify that if we shift all points by the same constant vector, the leverage function stays the same (the reason is that $\matK$ is shift invariant, while the shift corresponds to a phase shift in $\z(\veta)$ and a reverse phase shift in $\z(\veta)^\conj$). This implies that without loss of generality we can assume that $\x_1,...,\x_n \in [-\rad,\rad]^d$.

Recall from Lemma~\ref{lem:altlev-ub} that
\begin{equation}
\tau_\lambda(\veta) = \min_{y \in L_2(d\mu)} \lambda^{-1}\TNormS{\matPhi y - \sqrt{p(\veta)} \z(\veta)} +
\XNormS{y}{L_2(d\mu)} \label{eq:lvrg-min}.
\end{equation}
To upper bound $\tau_\lambda(\veta)$ for any $\veta \in \RR^d$, we exhibit a test function, $y_\veta(\cdot)$, and compute
the quantity under the minimum. As discussed in Section \ref{sec:upperBoundSketch}, $y_\veta(\cdot)$ will be a `softened spike function' given by:

\begin{defn}[Softened spike function] \label{def:softspike}
	For any $\veta$, and any $u$ define $y_{\veta,u}:\RR^d \to \RR$ as follows:
	\begin{align}\label{softSpike}
	y_{\veta,u}(\bv t)= \frac{\sqrt{p(\veta)}}{p(\bv t)} \cdot e^{-\|\bv t-\veta\|_2^2 \cdot u^2/4} \cdot v^d\cdot
	\sinc{v (\bv t-\veta)}
	\end{align}
	where $v = 2({\rad} + u\sqrt{2\log n_\lambda})$.
\end{defn}

The reweighted function $g_{\veta,u}(\bv t) = {p(\bv t)} \cdot y_{\veta,u}(\bv t)$ is just a d-dimensional Gaussian with
standard deviation $\Theta(1/u)$ multiplied by a sinc function with width $\tilde O(1/(u + \rad))$, both centered at
$\veta$. Taking the Fourier transform of this function yields a Gaussian with standard deviation $\Theta(u)$ convolved
with a box of width $\tilde O(u) + \rad$. The box is wide enough such that when it is centered between $[-\rad, \rad]^d$ the box covers nearly all the mass of the Gaussian, and so the Fourier transform is nearly identically $1$ on the range $[-\rad,\rad]^d$. Shifting by $\veta$, means that it is very close to a pure cosine wave with frequency $\veta$ on this range, and hence makes the first term of \eqref{eq:lvrg-min} small. We make this argument formal
below.

\subsection{Bounding \texorpdfstring{$\lambda^{-1}\TNormS{\matPhi y_{\veta,u} - \sqrt{p(\veta)}
			\z(\veta)}$}{TEXTTOFIX}}

\begin{lem}[Test Function Fourier Transform Bound]\label{yFourierBound} For any integer $n$, every parameter $0 < \lambda \le n$ and every $u\in \mathbb{R}$ and any $\veta \in \mathbb{R}^d$, and any kernel density function $p(\veta)$	and $d \le 4 n_\lambda$ if $\bv x_j \in [-\rad , +\rad]^d$ for all $j \in [n]$, then:
	\begin{align*}
	\lambda^{-1}\TNormS{\matPhi y_{\veta,u} - \sqrt{p(\veta)} \z(\veta)} =
	\frac{1}{\lambda}\sum_{j=1}^n \left  | \hat{g}_{\veta,u}(\x_j) - \sqrt{p(\veta)} \cdot \z(\veta)_j \right |^2 \le p(\veta),
	\end{align*}
	where $g_{\veta,u}(\bv t) \eqdef {p(\bv t)} y_{\veta,u}(\bv t)$.
\end{lem}
\begin{proof}
	We have $g_{\veta,u}(\bv t) = {p(\bv t)} y_{\veta,u}(\bv t) = {\sqrt{p(\veta)}} e^{-\|\t-\veta\|_2^2 \cdot u^2/4} \cdot v^d \cdot
	\sinc{v(\t-\veta)}$. We thus have:
	\begin{align}
	\hat{g}_{\veta,u}(\x_j) &= \sqrt{p(\veta)} \int_{\mathbb{R}^d} e^{-2\pi i \t^T \x_j} e^{-\|\t-\veta\|_2^2 \cdot u^2/4} \cdot
	v^d \cdot \sinc{v(\t-\veta)} d\t\nonumber\\
	&= \sqrt{p(\veta)} e^{-2\pi i \x^T_j \veta} \int_{\mathbb{R}^d} e^{-2\pi i \t^T \x_j} e^{-\|\t\|_2^2 \cdot u^2/4} \cdot v^d
	\cdot \sinc{v\t} d\t \nonumber\\
	&= \sqrt{p(\veta)} \cdot \z(\veta)_j  \cdot h(\x_j), \label{fourierTransForm}
	\end{align}
	where $h(\x) = \Big( \frac{2\sqrt{\pi}}{u} \Big)^d e^{-4\pi^2 \|\x\|_2^2/u^2} \ast \rect_v(\x)$ by the fact that
	multiplication in time domain becomes convolution in the Fourier domain (Claim
	\ref{claim:convthm}), $\mathcal{F}(e^{-\|\t\|_2^2 \cdot u^2/4}) = \Big( \frac{2\sqrt{\pi}}{u}  \Big)^d
	e^{-4\pi^2 \|\x\|_2^2/u^2}$, and $\mathcal{F}(v^d \cdot \sinc{v\t}) = \rect_v(\x)$.

	Because $\Big( \frac{2\sqrt{\pi}}{u} \Big)^d	e^{-4\pi^2 \|\x\|_2^2/u^2}$ is a positive function everywhere, we have $h(\x) \le \int_{\mathbb{R}^d} \Big( \frac{2\sqrt{\pi}}{u} \Big)^d	e^{-4\pi^2 \|\x\|_2^2/u^2} d\x = 1$ for all $\x$. Additionally, for any $\x \in [-\rad,\rad]^d$ we
	have by Claim \ref{claim:cdfnormal} and the fact that $v=2{\rad} +
	2u\sqrt{2\log n_\lambda}$:
	\begin{align*}
	h(\x) &= \int_{\y -\x \in [-\frac{v}{2}, +\frac{v}{2}]^d} \Big(\frac{2\sqrt{\pi}}{u}\Big)^d e^{-4\pi^2
		\|\y\|_2^2/u^2}\,d\y\\
	&\ge \Big( 1 - 2 \int_{v/2-\rad}^\infty \frac{2\sqrt{\pi}}{u} e^{-4\pi^2y_1^2/u^2}\,dy_1 \Big)^d, \\
	\end{align*}
    where $y_1$ denotes a scalar variable. Hence by Claim \ref{claim:cdfnormal} we have the following:
	\begin{align*}
	h(\x) &\ge 1 - 2d \int_{v/2-\rad}^\infty \frac{2\sqrt{\pi}}{u} e^{-4\pi^2y_1^2/u^2}dy_1\\
	&\ge 1 - \frac{d}{2\pi^{3/2}}\cdot \frac{u}{v/2-\rad}
	e^{-4\pi^2 (v/2-\rad)^2/u^2}\\
	&\ge 1 - \frac{1}{\sqrt{n_\lambda}}
	\end{align*}
	(since  $d \le 4 n_\lambda$).
	Plugging into \eqref{fourierTransForm} gives
	\begin{align*}
	\left | \hat{g}_{\veta,u} (\x_j) -\sqrt{p(\veta)} \cdot \z(\veta)_j \right |^2 &= p(\veta) \left | h(\x_j) - 1 \right |^2 \\
	&\le \frac{p(\veta)}{n_\lambda},
	\end{align*}
	and so,
	\begin{align*}
	\frac{1}{\lambda}\sum_{j=1}^n \left [\hat{g}(\x_j) - \sqrt{p(\veta)} \cdot
	\z(\veta)_j \right ]^2 \le \frac{n}{\lambda} \cdot p(\veta) \cdot \frac{\lambda}{n}
	=  p(\veta),
	\end{align*}
	proving the lemma.
\end{proof}

\subsection{Bounding \texorpdfstring{$\XNormS{y_{\veta,u}}{L_2(d\mu)}$}{TEXTTOFIX}}
Having established Lemma \ref{yFourierBound}, showing that the
weighted Fourier transform of $y_{\veta,u}$ is close to
$\sqrt{p(\veta)}\bv{z}(\veta)$, bounding the leverage function reduces to bounding the norm of the test
function. To that effect, we show the following:

\begin{lem}[Test Function $\ell_2$ Norm Bound]\label{yNormBound} For any integer $n$, any parameter $0 < \lambda \le
	\frac{n}{2}$, every $\veta \in \mathbb{R}^d$ with $\|\veta\|_\infty \le 10 \sqrt{\log n_\lambda}$, and every $  2000 \log
	n_\lambda \le u \le 500 \log^{1.5} n_\lambda$, if $y_{\veta,u}(\bv t)$ is defined as in \eqref{softSpike}, as per
	Definition \ref{def:softspike}, then we have
	\begin{equation}\label{normBound}
	\XNormS{y_{\veta,u}}{L_2(d\mu)} \le \Big( 6.2  \rad + 6.2  u\sqrt{2\log n_\lambda} \Big)^d
	\end{equation}
\end{lem}

We first prove the following claim:
\begin{claim}\label{claim:blah}
	Let $0 < \lambda \leq \frac{n}{2}$. For any constant $c>0$, every $\veta \in \mathbb{R}^d$ with $\|\veta\|_\infty
	\leq 100\sqrt{\log n_\lambda}$, every $\| \t-\veta\|_\infty \leq \frac{c\sqrt{\log n_\lambda}}{\sigma}$, and any $\sigma
	\ge 100c \cdot\log n_\lambda$, we have:
	\[
	e^{\frac{\|\t\|_2^2}{2}-\frac{\|\veta\|_2^2}{2}} \le 3^d.
	\]
\end{claim}
\begin{proof}
	Let $\bs\Delta = \t-\veta$. Then, note that $\|\bs\Delta\|_\infty \leq	c\sqrt{\log n_\lambda}/\sigma$, and so,
	\begin{align*}
	e^{\frac{\|\t\|_2^2}{2}-\frac{\|\veta\|_2^2}{2}} &= e^{\bs\Delta^\T \veta + \frac{\|\bs\Delta\|_2^2}{2}}\\
	&\leq e^{d \cdot \|\bs\Delta\|_\infty \cdot \|\veta\|_\infty} \cdot e^{d \cdot {\|\bs\Delta\|_\infty^2}}\\
	&\leq e^{d(c\sqrt{\log n_\lambda}/\sigma)(100\sqrt{\log n_\lambda})} \cdot e^{d(c\sqrt{\log
			n_\lambda}/\sigma)^2}\\
	&\le 3^d,
	\end{align*}
	since $\sigma \ge 100c \cdot \log n_\lambda$ and $n_\lambda \geq 2$.
\end{proof}

Now, we are ready to prove Lemma~\ref{yNormBound}.
\begin{proof}[Proof of Lemma \ref{yNormBound}] Recall that for the Gaussian kernel, we have $p(\veta) =
	\frac{1}{(\sqrt{2\pi})^d} e^{-\|\veta\|_2^2/2}$.
	We calculate:
	\begin{align}
	\int_{\mathbb{R}^d} |y_{\veta,u}(\bv t)|^2 d\mu(\bv t) &= p(\veta) \int_{\mathbb{R}^d} {\big(\sqrt{2\pi}\big)^d}{e^{\|\bv t\|_2^2/2}} \cdot e^{-\|\bv t-\veta\|_2^2 \cdot u^2/2} \cdot v^{2d} \left (\sinc{v(\bv t-\veta)}\right )^2 d\bv t\nonumber
	\end{align}
	
	Hence, it is enough to upper bound the following integral:
\begin{align}
&\int_{ \mathbb{R}^{d}} {e^{\|\bv t\|_2^2/2}} \cdot e^{-\|\bv t-\veta\|_2^2 \cdot u^2/2} \cdot \left (\sinc{v(\bv t-\veta)}\right )^2 d\bv t\nonumber\\
&\qquad =\prod_{l=1}^{d} \int_{ \mathbb{R}} {e^{| t_l|^2/2}} \cdot e^{-| t_l-\eta_l|^2 \cdot u^2/2} \cdot \left (\sinc{v( t_l -\eta_l)}\right )^2 d t_l \label{eq29}
\end{align}
	We proceed by upper bounding the one dimensional integral along some fixed coordinate $l$ as follows:
	\begin{align}
	&\int_{ \mathbb{R}} {e^{| t_l|^2/2}} \cdot e^{-| t_l-\eta_l|^2 \cdot u^2/2} \cdot \left (\sinc{v( t_l -\eta_l)}\right )^2 d t_l \nonumber\\
	&\qquad = \int_{ | t_l - \eta_l | \le \frac{20\sqrt{\log n_\lambda}}{u} } {e^{| t_l|^2/2}} \cdot e^{-| t_l-\eta_l|^2 \cdot u^2/2} \cdot \left (\sinc{v( t_l -\eta_l)}\right )^2 d t_l\nonumber\\
	& \qquad +  \int_{| t_l-\eta_l | \ge \frac{20\sqrt{\log n_\lambda}}{u}} {e^{| t_l|^2/2}} \cdot e^{-| t_l-\eta_l|^2 \cdot u^2/2} \cdot \left (\sinc{v( t_l -\eta_l)}\right )^2 d t_l \label{eq:l2testfunc}
	\end{align}
	For the integral over $|t_l-\eta_l| \ge \frac{20\sqrt{\log n_\lambda}}{u}$ we have:
	\begin{align}
	&\int_{| t_l-\eta_l | \ge \frac{20\sqrt{\log n_\lambda}}{u}} {e^{| t_l|^2/2}} \cdot e^{-| t_l-\eta_l|^2 \cdot u^2/2} \cdot \left (\sinc{v( t_l -\eta_l)}\right )^2 d t_l \nonumber\\
	&\qquad \le  \frac{1}{\big(v \cdot \frac{20\sqrt{\log n_\lambda}}{u}\big)^{2}} \int_{|t_l-\eta_l| \ge \frac{20\sqrt{\log n_\lambda}}{u}}  {e^{t_l^2/2}} \cdot e^{-(t_l-\eta_l)^2u^2/2}\, dt_l\nonumber\\
	&\qquad \le \frac{1}{v} \int_{|t_l-\eta_l| \ge \frac{20\sqrt{\log n_\lambda}}{u}} {e^{t_l^2/2}} \cdot	e^{-(t_l-\eta_l)^2u^2/2}\, dt_l \label{eq30}
	\end{align}
	The first inequality is because by definition of $\sinc{\cdot}$ we have the following for all $|t_l-\eta_l| \ge \frac{20\sqrt{\log n_\lambda}}{u}$:
	$$|\sinc{v( t_l-\eta_l)}|^2 =  \Big| \frac{\sin(\pi v (t_l-\eta_l))}{\pi v (t_l-\eta_l)} \Big|^2 \le \frac{1}{\big(v \cdot \frac{20\sqrt{\log n_\lambda}}{u}\big)^2}$$
	The last inequality in \eqref{eq30} due to the fact that:
	\begin{align*}
	\frac{1}{\big(v \cdot \frac{20\sqrt{\log n_\lambda}}{u}\big)^2} &= \frac{1}{v} \cdot \frac{1}{v \cdot \big( \frac{20\sqrt{\log n_\lambda}}{u}\big)^2}\\
	& \le \frac{1}{v} \cdot \frac{1}{  800 \left(\frac{{\log^{1.5} n_\lambda}}{u}\right)}\text{~~~(since  $v = 2(R + u\sqrt{2\log n_\lambda}) \ge 2u \sqrt{2 \log n_\lambda}$, see Definition~\ref{def:softspike})}\\
	&\le \frac{1}{v} \text{\hspace{1.1in}(since $u \le 500 \log^{1.5} n_\lambda$)}\\
	\end{align*}
	Now note that $t_l^2 \le 2(t_l-\eta_l)^2 + 2\eta_l^2$. We have the following for all $|t_l-\eta_l| \ge \frac{20\sqrt{\log n_\lambda}}{u}$:
	\begin{align*}
	t_l^2 &\le 2(t_l-\eta_l)^2 + 2\eta_l^2\\
	&\le 2(t_l-\eta_l)^2 + 200 \log n_\lambda \text{~~~~~~~~~(by the assumption  $\|\veta\|_\infty \le 10 \sqrt{\log n_\lambda}$)} \\
	&\le 2(t_l-\eta_l)^2 + (t_l-\eta_l)^2 u^2/2 \text{~~~~(by the assumption  $|t_l-\eta_l| \ge \frac{20\sqrt{\log n_\lambda}}{u}$)}\\
	&\le \frac{2}{3} (t_l-\eta_l)^2 u^2
	\end{align*}
	where the last inequality follows from $u \ge 2000 \log n_\lambda \ge 600$ (because $n_\lambda \ge 2$). Hence,
	\begin{align}
	\frac{1}{v} \int_{|t_l-\eta_l| \ge \frac{20\sqrt{\log n_\lambda}}{u}}  {e^{t_l^2/2}} \cdot e^{-(t_l-\eta_l)^2u^2/2}
	dt_l &\le \frac{1}{v} \int_{|t_l-\eta_l| \ge \frac{20\sqrt{\log n_\lambda}}{u}} e^{-(t_l-\eta_l)^2u^2/6} dt_l
	\nonumber\\
	&=\frac{1}{v} \int_{|t'| \ge \frac{20\sqrt{\log n_\lambda}}{u}} e^{-(t')^2u^2/6} dt'
	\nonumber\\
	&\le \frac{1}{v} \cdot n_\lambda^{-60} \label{eq:17}
	\end{align}
	The last inequality follows from Claim \ref{claim:cdfnormal}.
	
	Now, we bound the first integral on the right side of \eqref{eq:l2testfunc}:
	\begin{align}
	\int_{ t \in \big[\eta_l - \frac{20\sqrt{\log n_\lambda}}{u} , \, \eta_l + \frac{20\sqrt{\log n_\lambda}}{u} \big]}
	{e^{| t_l|^2/2}} \cdot e^{-|t_l-\eta_l|^2 \cdot u^2/2} \left (\sinc{v( t_l -\eta_l)}\right )^2 dt_l
	& \le 3 e^{\frac{|\eta_l|^2}{2}}  \int_{\mathbb{R}} \left(\sinc{v( t_l -\eta_l)}\right )^2 d t_l \nonumber\\
	&= \frac{3 e^{\frac{|\eta_l|^2}{2}}}{v}. \label{eq:18}
	\end{align}
	where the inequality follows from Claim \ref{claim:blah} with $c=20$ because by assumption $u \ge 2000\log n_\lambda$.
	
	Now by incorporating \eqref{eq:17} and \eqref{eq:18} into \eqref{eq:l2testfunc}, we have
\begin{align*}
&\int_{ \mathbb{R}} {e^{| t_l|^2/2}} \cdot e^{-| t_l-\eta_l|^2 \cdot u^2/2} \cdot \left (\sinc{v( t_l -\eta_l)}\right )^2 d t_l \\
&\qquad \le \frac{3 e^{\frac{|\eta_l|^2}{2}}}{v} + \frac{1}{v} \cdot n_\lambda^{-60}\\
&\qquad \le \frac{3.1 e^{\frac{|\eta_l|^2}{2}}}{v}.
\end{align*}
	If we plug the above inequality into \eqref{eq29}, we get the following:
	
	\begin{equation}
	\int_{\mathbb{R}^d} |y_{\veta,u}(\bv t)|^2 d\mu(\t) \le \big({\sqrt{2\pi}}\big)^d p(\veta)\cdot v^{2d} \Big(  \frac{3.1^d e^{\frac{\|\veta\|_2^2}{2}}}{v^d} \Big) \le (3.1v)^d.
	\end{equation}
\end{proof}

\begin{proof}[Proof of Theorem \ref{thm:lev-scores-ub}]
	By the assumptions of the theorem $n$ is an integer, parameter $0 < \lambda \le n/2$,  and $\rad >0$, and all $\bv x_1,...,\bv x_n \in [-\rad,\rad]^d$ and $p(\veta) = \frac{1}{\sqrt{2\pi}} e^{-\frac{\|\veta\|_2^2}{2}}$, therefore Lemmas \ref{yFourierBound}, and \ref{yNormBound} go through.
	Hence the theorem follows immediately from setting $u = 2000{\log n_\lambda}$ and then plugging Lemmas
	\ref{yFourierBound} and \ref{yNormBound} into \eqref{eq:lvrg-min}.
\end{proof}

\section{Proof of Theorem~\ref{thm:main-lev-lb}}\label{sec:main-lev-lb}

With the choice of the Gaussian kernel with $\sigma=(2\pi)^{-1}$ we have
$p(\veta) = (2\pi)^{-d/2}\exp(-\|\veta\|_2^2/2)$.
Recall from Lemma~\ref{lem:altlev-lb} that
\begin{equation}
\tau_\lambda(\veta) = \max_{\valpha \in \CC^n}
\frac{p(\veta)\cdot|\valpha^\conj
	\z(\veta)|^2}{\XNormS{\matPhi^\conj \valpha}{L_2(d\mu)}+\lambda
	\TNormS{\valpha}}.
\label{eq:taulamb}
\end{equation}
In particular, this gives us a method of bounding the leverage function from
below, namely, by exhibiting some $\valpha$ and computing the quantity under the
maximum.

The rest of this section is organized as follows. In
Section~\ref{subsec:dataconstruction}, we construct our candidate set of data
points $\x_1, \x_2, \dots, \x_n$ along with the vector $\valpha$. In
particular,
$\valpha$ will be chosen to be a vector of samples of a function
$f_{\bs{\Delta},b,v}$ at each of the data points.
Section~\ref{subsec:basicprop} then describes basic Fourier properties of the
function $f_{\bs{\Delta},b,v}$ and $\valpha$ that we will require later.
The
remaining sections then bound each of the relevant quantities that appear in
\eqref{eq:taulamb} for our specific choice of $\x_1,\x_2,\dots,\x_n$ and
$\valpha$.
In particular, Section~\ref{subsec:azbd} shows a lower bound for
$\valpha^*\z(\veta)$, while Section~\ref{subsec:alphanormbd} shows an
upper bound
for $\|\valpha\|_2^2$ and Section~\ref{subsec:xnormbd} shows an upper bound for
$\XNormS{\matPhi^\conj \valpha}{L_2(d\mu)}$.

\subsection{Construction of Data Point Set and the Vector of Coefficients
	\texorpdfstring{$\valpha$}{TEXTTOFIX}} \label{subsec:dataconstruction}
In this section, we construct a set of data points as well an $\valpha$. As discussed in
Section~\ref{sec:actualBounds}, we choose the data points to lie on an evenly spaced grid inside $[-\rad, \rad]^d$.
Moreover, because of the duality of Lemmas \ref{lem:altlev-lb} and \ref{lem:altlev-ub}, we choose $\bs\alpha$ to be
related to the test function $y_{\veta}$ in the leverage score upper bound provided in
Section~\ref{sec:lev-scores-ub}. In particular, $\bs\alpha$ is formed by taking samples of a modified
version of $\matPhi y_{\veta}$ (i.e., a weighted Fourier transform of $y_{\veta}$) on the data points. In
particular, the function we sample is $f_{\bs{\Delta},b,v}$, which we now formally define. We then proceed
to proving some useful properties before formally defining $\x_1, \x_2,\dots, \x_n$
and $\bs\alpha$.

\begin{defn} \label{def:alpha-test-func}
	For parameters $\bs{\Delta}\in\RR^d$, $b>0$ and $v>0$, let the function
	$f_{\bs{\Delta}, b, v}: \RR^d\to\RR$ be defined as follows:
	\begin{align}
	f_{\bs{\Delta}, b, v}(\a) &=  2\cos (2\pi \bs{\Delta}^\T \a)
	\left(
	\frac1{\left(\sqrt{2\pi}
		b\right)^d} e^{-\|\cdot\|_2^2/2b^2} \ast \rect_v \right)(\bs{a})\nonumber\\
	&= 2\cos (2\pi \bs{\Delta}^\T \a) \int_{a_1-v/2}^{a_1+v/2}
	\int_{a_2-v/2}^{a_2+v/2} \dots \int_{a_d-v/2}^{a_d+v/2}
	\frac1{\left(\sqrt{2\pi}b\right)^d} e^{-\|\t\|_2^2/2b^2}
	\,dt_d\,\cdots\, dt_2\,dt_1,
	\nonumber
	\end{align}
	where $\a = (a_1,a_2,\dots, a_d)$ and $\t = (t_1, t_2, \dots, t_d)$.
\end{defn}

\begin{lem}\label{f:1}
	For any $\bs{\Delta}\in\RR^d$, $v>0$, and $b>0$, if we define the
	function
	$f_{\bs{\Delta}, b, v}$ as in Definition~\ref{def:alpha-test-func}, then
	\begin{equation}
	\Fc \left(f_{\bs{\Delta}, b, v}\right)(\bs{\xi}) = e^{-2\pi^2 b^2
		\|\bs{\xi}-\bs{\Delta}\|_2^2}
	(v^d \cdot \sinc {v(\bs{\xi}-\bs{\Delta})} + e^{-2\pi^2 b^2
		\|\bs{\xi}+\bs{\Delta}\|_2^2}
	(v^d \cdot \sinc {v(\bs{\xi}+\bs{\Delta})}).\nonumber
	\end{equation}
\end{lem}
\begin{proof}
	Note that
	\[
	\Fc\left(\frac{1}{(\sqrt{2\pi}b)^d}
	e^{-\|\cdot\|_2^2/2b^2}\right)(\bs{\xi}) =
	e^{-2\pi^2b^2 \|\bs{\xi}\|_2^2}.
	\]
	Thus, by the convolution theorem (see Claim~\ref{claim:convthm}),
	\[
	\Fc \left( \frac1{(\sqrt{2\pi}b)^d}
	e^{-\|\cdot\|_2^2/2b^2} \ast \rect_v \right)(\bs{\xi})\nonumber =
	e^{-2\pi^2 b^2 \|\bs{\xi}\|_2^2} \cdot v^d \cdot \sinc{ v(\bs{\xi})}.
	\]
	Now by the duality of phase shift in time domain and frequency shift in the Fourier domain,
	\begin{align*}
	\Fc(f_{\bs{\Delta}, b, v})(\bs{\xi}) &= \Fc\left((e^{2\pi i
		\bs{\Delta}^\T\cdot} +
	e^{-2 \pi i \bs{\Delta}^\T \cdot}) \left( \frac1{(\sqrt{2\pi}b)^d}
	e^{-\|\cdot\|_2^2/2b^2} \ast \rect_v \right)\right)(\bs{\xi})\\
	&= \Fc\left( \frac1{(\sqrt{2\pi}b)^d}
	e^{-\|\cdot\|_2^2/2b^2} \ast \rect_v\right)(\bs{\xi}-\bs{\Delta}) +
	\Fc\left(\frac1{(\sqrt{2\pi}b)^d} e^{-\|\cdot\|_2^2/2b^2} \ast
	\rect_v\right)(\bs{\xi}+\bs{\Delta})\\
	&= e^{-2\pi^2 b^2 (\bs{\xi}-\bs{\Delta})^2} \cdot v^d \cdot \sinc
	{v(\bs{\xi}-\bs{\Delta})} +
	e^{-2\pi^2 b^2 (\bs{\xi}+\bs{\Delta})^2} \cdot v^d \cdot \sinc
	{v(\bs{\xi}+\bs{\Delta})}.
	\end{align*}
\end{proof}

\begin{defn}
	[Construction of data points and $\valpha$]
	We let $n = m^d$ for an odd integer $m > 0$. Then, we define a set of $n$ data
	points $\x_1, \x_2,\dots \x_n \in \RR^d$ as follows: We index the
	points by a $d$-tuple $\mathbf{j} = (j_1, j_2, \dots, j_d) \in \{1,2,\dots,m\}^d$
	for convenience. In particular, we rename $\x_1,\x_2,\dots,\x_n$ as
	$\x^{\mathbf{j}} = \x^{(j_1,j_2,\dots,j_d)}$, over $j_1,j_2,\dots,j_d \in
	\{ 1, 2, \dots, m\}$, where $\x^{(j_1, j_2,\dots,j_d)}$ is defined as
	\begin{equation}
	\x^{(j_1,j_2,\dots,j_d)} = \left(\left(j_1 - \frac{m + 1}{2}\right) \cdot
	\frac{2\rad}{m}, \left(j_2 - \frac{m+1}{2}\right)\cdot \frac{2\rad}{m},
	\dots, \left(j_d - \frac{m+1}{2}\right)\cdot \frac{2\rad}{m}\right). \nonumber
	\end{equation}
	Thus, the data points are on a grid of width $\frac{2\rad}{m}$ extending
	from $-\rad$ to $\rad$ in all $d$ dimensions. For convenience, we let $c_j
	= \left(j - \frac{m+1}{2}\right) \cdot \frac{2\rad}{m}$. Thus, note that
	$\x^{(j_1,j_2,\dots,j_d)} = (c_{j_1}, c_{j_2}, \dots, c_{j_d})$.
	
	Given a point $\veta\in\RR^d$ at which we wish to bound the ridge leverage function,
	we define the vector $\valpha \in \CC^d$ to be the tuple of evaluations of
	$f_{\veta, b, v}$ at the individual $\x^{\mathbf{j}}$, for
	some choice of parameters $b$ and $v$ that we set later. More specifically, we define
	$\valpha = \{\alpha_{j_1,j_2,\dots,j_d}\}_{1\leq j_1,j_2,\dots,j_d\leq m}$ by
	\begin{align}
	\valpha_{\mathbf{j}} = \valpha_{j_1,j_2,\dots,j_d} &=
	f_{\veta,b,v}(\x^{(j_1,j_2,\dots,j_d)})\nonumber\\
	&= 2\cos(2\pi \veta^\T \x^{(j_1,j_2,\dots,j_d)})
	\int_{x_1-\frac{v}{2}}^{x_1+\frac{v}{2}} \cdots
	\int_{x_d-\frac{v}{2}}^{x_d+\frac{v}{2}}
	\frac1{(\sqrt{2\pi}b)^d}
	e^{-\|\t\|_2^2/2b^2} \,dt_d\,\cdots\,dt_1. \label{ins:alpha}
	\end{align}
	\label{def:alpha}
\end{defn}

\subsection{Basic Properties of \texorpdfstring{$f_{\bs\Delta,b,v}$ and $\valpha$}{TEXTTOFIX}}
\label{subsec:basicprop}
By the Nyquist-Shannon sampling theorem, we have the following lemma.

\begin{lem} \label{lem:ft1}
	For any parameters $\bs{\Delta} \in \RR^d$, $v>0$, and $b>0$, if we define the
	function $f_{\veta, b, v}$ as in Definition \ref{def:alpha-test-func},
	then for any $w>0$,
	\begin{align*}
	\Fc \left( f_{\bs{\Delta}, b,v}(\cdot) \cdot \sum_{\mathbf{j} \in\ZZ^d}
	\delta(\cdot
	- w\mathbf{j}) \right)(\bs{\xi}) &= w^{-d} v^d \sum_{\mathbf{j} \in\ZZ^d}
	e^{-2\pi^2b^2\|\bs{\xi}-\bs{\Delta}-w^{-1}\mathbf{j}\|_2^2} \cdot
	\sinc{v(\bs{\xi}-\bs{\Delta}-w^{-1}\mathbf{j})}\\ &\quad + w^{-d} v^d
	\sum_{\mathbf{j}
		\in\ZZ^d} e^{-2\pi^2b^2\|\bs{\xi}+\bs{\Delta}-w^{-1}\mathbf{j}\|_2^2}
	\cdot
	\sinc{v(\bs{\xi}+\bs{\Delta}-w^{-1}\mathbf{j})}.
	\end{align*}
\end{lem}
\begin{proof}
	By Claim \ref{claim:nyquist}, we have
	\begin{align}
	\Fc \left( f_{\bs{\Delta}, b,v}(\cdot)  \sum_{\mathbf{j}\in\ZZ^d}
	\delta_d(\cdot-w\mathbf{j}) \right)(\bs{\xi}) &= \sum_{\mathbf{j}\in\ZZ^d} w^{-d} \Fc(f_{\bs{\Delta},
		b,v})(\bs{\xi}-w^{-1}\mathbf{j}).
	\label{eq:alias}
	\end{align}
	Thus, by Lemma \ref{f:1}, we find that \eqref{eq:alias} can be written as
	\begin{align*}
	\sum_{\mathbf{j}\in\ZZ^d} w^{-d} \Fc(f_{\bs{\Delta}, b,v})(\bs{\xi}-w^{-1}
	\mathbf{j})
	&=
	w^{-d} \sum_{\mathbf{j}\in\ZZ^d}
	e^{-2\pi^2b^2\|\bs{\xi}-\bs{\Delta}-w^{-1}\mathbf{j}\|^2} \cdot
	v^d\cdot\sinc{v(\bs{\xi}-\bs{\Delta}-w^{-1}\mathbf{j})}\\
	&\quad + w^{-d} \sum_{\mathbf{j}\in\ZZ^d}
	e^{-2\pi^2b^2\|\bs{\xi}+\bs{\Delta}-w^{-1}\mathbf{j}\|_2^2}
	\cdot v^d\cdot\sinc{v(\bs{\xi}+\bs{\Delta}-w^{-1}\mathbf{j})},
	\end{align*}
	which completes the proof.
\end{proof}

\begin{lem} \label{lem:ft2}
	For every odd integer $m \ge 3$ and parameters $n=m^d$, $1\leq d \leq 64 n_\lambda^{5/2} \log^{3/2} n_\lambda$,
	$0 < \lambda < n/3$,
	$\veta \in\RR^d$, $\frac{\rad}{2} < v \leq \rad$, and $0 < b \leq \frac{\rad}{8\sqrt{\log
			n_\lambda}}$, if we define the function $f_{\veta, b, v}$ as in
	Definition \ref{def:alpha-test-func}, then
	\[
	\left|\Fc \left(\sum_{\|\mathbf{j}\|_\infty >\frac{m}{2}}
	f_{\veta,b,v}\left(\frac{2\rad}{m} \mathbf{j} \right)\cdot \delta\left(\cdot- \frac{2\rad}{m} \mathbf{j} \right)
	\right)(\bs\xi) \right| \le \sqrt{\lambda n}
	\]
	for all $\bs\xi\in\RR^d$.
\end{lem}
\begin{proof}
	By definition of $f_{\veta,b,v}$, we have the following for
	all $\a = (a_1, a_2, \dots, a_d)$:
	\begin{equation}
	|f_{\veta,b,v}(\a)| \leq
	\int_{a_1-\frac{v}{2}}^{a_1+\frac{v}{2}}
	\int_{a_2-\frac{v}{2}}^{a_2+\frac{v}{2}} \cdots
	\int_{a_d-\frac{v}{2}}^{a_d+\frac{v}{2}} \frac{2}{(\sqrt{2\pi} b)^d}
	e^{-\|\t\|_2^2/2b^2}\, dt_d\,\cdots\,dt_2\,dt_1. \label{eq:fbd}
	\end{equation}
	Note that if $\mathbf{j}\in\RR^d$ satisfies $|j_k| > \frac{m}{2}$ for some
	$k\in\{1,2,\dots,d\}$, then \eqref{eq:fbd} implies that
	\begin{align*}
	\left|f_{\veta,b,v}\left(\frac{2\rad}{m}\mathbf{j}\right) \right|
	&\leq 2 \prod_{i=1}^{d} \int_{\frac{2\rad}{m}j_i -
		\frac{v}{2}}^{\frac{2\rad}{m}j_i + \frac{v}{2}} \frac{1}{\sqrt{2\pi}b}
	e^{-t^2/2b^2}\,dt \\
% 	&\leq 2 \left(\int_{\frac{2\rad}{m}|j_k| -\frac{v}{2}}^{\infty}
% 	\frac{1}{\sqrt{2\pi}b} e^{-t^2/2b^2}\,dt\right) \prod_{\substack{1\leq i\leq
% 			d\\ i\neq k}} \int_{\frac{2\rad}{m}j_i - \frac{v}{2}}^{\frac{2\rad}{m}j_i
% 		+ \frac{v}{2}} \frac{1}{\sqrt{2\pi}b} e^{-t^2/2b^2}\,dt\\
% 	&\leq \left(2 \int_{\frac{2\rad}{m}|j_k| -\frac{v}{2}}^{\infty}
% 	\frac{1}{\sqrt{2\pi}b} e^{-t^2/2b^2} \,dt\right) \prod_{\substack{1\leq i\leq
% 			d\\ i\neq k}} \int_{\frac{2\rad}{m}j_i - \frac{v}{2}}^{\frac{2\rad}{m}j_i
% 		+ \frac{v}{2}} \frac{1}{\sqrt{2\pi}b} e^{-t^2/2b^2}\,dt\\
	&\le \left(\frac{2}{\sqrt{2\pi} b} \int_{\frac{\rad}{m}|j_k|}^{\infty}
	e^{-t^2/2b^2}\,dt\right) \prod_{\substack{1\leq i\leq
			d\\ i\neq k}} \int_{\frac{2\rad}{m}j_i - \frac{v}{2}}^{\frac{2\rad}{m}j_i
		+ \frac{v}{2}} \frac{1}{\sqrt{2\pi}b} e^{-t^2/2b^2}\,dt \\
	&\leq \frac{2}{\sqrt{2\pi}}\cdot \frac{mb}{\rad |j_k|}\cdot
	e^{-\frac{1}{2}\cdot\left(\frac{\rad |j_k|}{mb}\right)^2}
	\prod_{\substack{1\leq i\leq
			d\\ i\neq k}} \int_{\frac{2\rad}{m}j_i - \frac{v}{2}}^{\frac{2\rad}{m}j_i
		+ \frac{v}{2}} \frac{1}{\sqrt{2\pi}b} e^{-t^2/2b^2}\,dt \\
	&\leq \frac{2b}{\rad}\cdot
	e^{-\frac{\rad^2 |j_k|^2}{2m^2 b^2}} \prod_{\substack{1\leq i\leq
			d\\ i\neq k}} \int_{\frac{2\rad}{m}j_i - \frac{v}{2}}^{\frac{2\rad}{m}j_i
		+ \frac{v}{2}} \frac{1}{\sqrt{2\pi}b} e^{-t^2/2b^2}\,dt,
	\end{align*}
	where we have used the fact that $\frac{2\rad}{m}|j_k| - \frac{v}{2}
	\geq \frac{2\rad}{m}|j_k| - \frac{\rad}{2} \geq \frac{\rad}{m}|j_k|$,
	along with Claim~\ref{claim:cdfnormal}. Therefore,
	\begin{align*}
	&\left|\Fc \left(\sum_{\|\mathbf{j}\|_\infty>\frac{m}{2}}
	  f_{\veta,b,v}\left(\frac{2\rad}{m} \mathbf{j} \right)\cdot \delta\left(\cdot- \frac{2\rad}{m} \mathbf{j} \right)
	  \right)(\bs\xi) \right|\\
	&\qquad\qquad\qquad\qquad\leq \sum_{\|\mathbf{j}\|_\infty > \frac{m}{2}} \left|f_{\veta,b,v}
	  \left(\frac{2\rad}{m}\mathbf{j}\right)\right|\\
	&\qquad\qquad\qquad\qquad\leq \sum_{k=1}^d \sum_{\substack{\mathbf{j}\in\ZZ^d\\ |j_k| > \frac{m}{2}}}
	  \left|f_{\veta,b,v} \left( \frac{2\rad}{m}\mathbf{j} \right)\right| \\
	&\qquad\qquad\qquad\qquad\leq \sum_{k=1}^d \sum_{\substack{\mathbf{j}\in\ZZ^d\\ |j_k| > \frac{m}{2}}}\frac{2b}{\rad}
	  e^{-\frac{\rad^2 |j_k|^2}{2m^2 b^2}} \prod_{\substack{1\leq i\leq d\\ i\neq k}} \int_{\frac{2\rad}{m}j_i -
	  \frac{v}{2}}^{\frac{2\rad}{m}j_i + \frac{v}{2}} \frac{1}{\sqrt{2\pi}b} e^{-t^2/2b^2}\,dt
	\end{align*}
	We bound:
	\begin{align*}
	\sum_{k=1}^d \sum_{\substack{\mathbf{j}\in\ZZ^d\\ |j_k| > \frac{m}{2}}}\frac{2b}{\rad}
	& e^{-\frac{\rad^2 |j_k|^2}{2m^2 b^2}} \prod_{\substack{1\leq i\leq d\\ i\neq k}} \int_{\frac{2\rad}{m}j_i -
		\frac{v}{2}}^{\frac{2\rad}{m}j_i + \frac{v}{2}} \frac{1}{\sqrt{2\pi}b} e^{-t^2/2b^2}\,dt \\
	&\qquad\qquad\qquad\qquad\leq \frac{2b}{\rad} \sum_{k=1}^d \left(\sum_{|j_k| > \frac{m}{2}} e^{-\frac{\rad^2
			|j_k|^2}{2m^2 b^2}}\right) \cdot \prod_{\substack{1\leq i\leq d\\ i\neq k}} \left( \sum_{j_i=-\infty}^{\infty}
	\int_{\frac{2\rad}{m}j_i - \frac{v}{2}}^{\frac{2\rad}{m}j_i + \frac{v}{2}} \frac{1}{\sqrt{2\pi}b}
	e^{-t^2/2b^2}\,dt \right) \\
	&\qquad\qquad\qquad\qquad\leq \frac{2b}{\rad} \sum_{k=1}^d \left(\sum_{|j_k| > \frac{m}{2}} e^{-\frac{\rad^2
			|j_k|^2}{2m^2 b^2}}\right) \cdot \prod_{\substack{1\leq i\leq d\\ i\neq k}} \left(\left\lceil \frac{vm}{2\rad}
	\right\rceil \int_{-\infty}^\infty \frac{1}{\sqrt{2\pi}b}   e^{-t^2/2b^2}\,dt \right)
	\end{align*}
	where the last inequality is due to the fact that each point in $\RR$ appears in at most $\lceil \frac{vm}{2\rad}
	\rceil$ summands in the infinite sum. Again using Claim~\ref{claim:cdfnormal}:
	\begin{align*}
	&\frac{2b}{\rad} \sum_{k=1}^d \left(\sum_{|j_k| > \frac{m}{2}} e^{-\frac{\rad^2
			|j_k|^2}{2m^2 b^2}}\right) \cdot \prod_{\substack{1\leq i\leq d\\ i\neq k}} \left(\left\lceil \frac{vm}{2\rad}
	\right\rceil \int_{-\infty}^\infty \frac{1}{\sqrt{2\pi}b}   e^{-t^2/2b^2}\,dt \right)\\
	&\qquad\qquad\qquad\qquad\leq \frac{2b}{\rad} \left(\frac{vm}{\rad}\right)^{d-1} \sum_{k=1}^d
	\left(\sum_{|j_k| >   \frac{m}{2}} e^{-\frac{\rad^2 |j_k|^2}{2m^2 b^2}}\right) \\
	&\qquad\qquad\qquad\qquad\leq \frac{4b}{\rad} \left(\frac{vm}{\rad}\right)^{d-1} \sum_{k=1}^d
	\int_{\frac{m-1}{2}}^\infty e^{-\frac{\rad^2 t^2}{2m^2 b^2}}\,dt\\
	&\qquad\qquad\qquad\qquad\leq \frac{4bd}{\rad} \cdot m^{d-1} \int_{\frac{m-1}{2}}^\infty   e^{-\frac{\rad^2
			t^2}{2m^2 b^2}}\,dt\\
	&\qquad\qquad\qquad\qquad\leq \frac{4bd}{\rad} \cdot m^{d-1} \cdot
	  \frac{m^2 b^2/\rad^2}{\left(\frac{m-1}{2}\right)} e^{-\frac{\rad^2
	  \left(\frac{m-1}{2}\right)^2}{2m^2 b^2}} \\
	&\qquad\qquad\qquad\qquad\leq 12dn \left(\frac{b}{\rad}\right)^3 e^{-\rad^2/18b^2} \\
	&\qquad\qquad\qquad\qquad\leq \frac{3}{2}\lambda \leq \sqrt{\lambda n},
	\end{align*}
	since $m \ge 3$, $\rad \ge 8 b \sqrt{\log n_\lambda}$, $d \leq 64 {n_\lambda}^{5/2} \log^{3/2} n_\lambda$, and $\lambda \leq n/3$.
\end{proof}

\begin{lem}
	For every odd integer $m \ge 3$ and parameters $n=m^d$, $1\leq d \leq 64 n_\lambda^{5/2} \log^{3/2} n_\lambda$,
$0 <
	\lambda < n/3$, $\veta \in\RR^d$, $0 < v \leq \rad$, and $0 < b \leq \frac{\rad}{8\sqrt{\log
			n_\lambda}}$, if $\valpha$ is defined as in \eqref{ins:alpha} of Definition~\ref{def:alpha}, then we have,
	\begin{align}
	&\left| \valpha^* \z (\bs\xi) - \left(\frac{mv}{2\rad}\right)^d \sum_{\mathbf{j}\in\ZZ^d} \left(e^{-2\pi^2
		b^2
		\left\|\bs\xi - \veta - \frac{m}{2\rad}\mathbf{j} \right\|_2^2} \cdot \sinc{v\left (\bs\xi- \veta - \frac{m}{2\rad}
		\mathbf{j}\right)} \right. \right. \nonumber\\
	&\qquad\qquad\qquad\qquad\qquad \left. \left. + e^{-2\pi^2 b^2 \left\|\bs\xi + \veta - \frac{m}{2\rad}\mathbf{j}
		\right\|_2^2 } \cdot \sinc{v \left(\bs\xi + \veta - \frac{m}{2\rad} \mathbf{j} \right)} \right)
	\vphantom{\sum_{\mathbf{j}\in\ZZ^d}} \right| \le \sqrt{\lambda n}. \label{eq:zeta}
	\end{align}
	\label{lem:8}
\end{lem}

\begin{proof}
	Note that
	\begin{align}
	\valpha^* \z (\bs{\xi}) &= \sum_{1\leq j_1,j_2,\dots, j_d\leq m}
	\alpha_{\mathbf{j}} e^{-2\pi i \bs{x}^{\mathbf{j}}\cdot
		\bs{\xi}} \nonumber\\
	&= \sum_{\substack{\mathbf{j}\in\ZZ^d\\ \|\mathbf{j}\|_\infty \le \frac{m}{2}}}
	f_{\veta,b,v}\left(\frac{2\rad}{m}\mathbf{j}\right) \cdot
	e^{-2\pi i \left(\frac{2\rad}{m}\right)\mathbf{j}^\T \bs{\xi}} \nonumber\\
	&=  \Fc\left(\sum_{\substack{\mathbf{j}\in\ZZ^d\\ \|\mathbf{j}\|_\infty \le \frac{m}{2}}}
	f_{\veta,b,v}\left(\frac{2\rad}{m}\mathbf{j}\right) \cdot
	\delta\left(\cdot
	-\frac{2\rad}{m}\mathbf{j}\right)\right)(\bs{\xi}) \nonumber\\
	&= \Fc\left(\sum_{\mathbf{j} \in \ZZ^d} f_{\veta,
		b,v}(\cdot)\cdot
	\delta\left(\cdot -\frac{2\rad}{m}\mathbf{j}\right)\right)(\bs{\xi}) \nonumber\\
	&\qquad\qquad\qquad\qquad\qquad - \Fc\left(\sum_{\substack{\mathbf{j}\in\ZZ^d\\ \|\mathbf{j}\|_\infty > \frac{m}{2}}}
	f_{\veta,b,v}\left(\frac{2\rad}{m} \mathbf{j}\right) \cdot
	\delta\left(\cdot-\frac{2\rad}{m}\mathbf{j}\right)\right)(\bs{\xi}).
	\label{eq:fourierdiff}
	\end{align}
	By Lemma \ref{lem:ft1} (applied with $w = 2\rad/m$), we have the
	following expression for the first term in \eqref{eq:fourierdiff}:
	\begin{align}
	&\Fc\left(\sum_{\mathbf{j}\in\ZZ^d} f_{\veta, b,v}(\cdot)\cdot
	\delta\left(\cdot -\frac{2\rad}{m}\mathbf{j}\right)\right)(\bs{\xi}) \nonumber\\
	&\qquad\qquad\qquad\qquad =
	\left(\frac{mv}{2\rad}\right)^d \sum_{\mathbf{j}\in\ZZ^d} e^{-2\pi^2 b^2
		\|\bs{\xi} - \veta - \frac{m}{2\rad}\mathbf{j}\|_2^2 } \cdot
	\sinc{v\left(\bs{\xi} - \veta - \frac{m}{2\rad}\mathbf{j}\right)}\nonumber\\
	&\qquad\qquad\qquad\qquad\qquad +
	\left(\frac{mv}{2\rad}\right)^d
	\sum_{\mathbf{j}\in\ZZ^d} e^{-2\pi^2 b^2
		\|\bs{\xi} + \veta - \frac{m}{2\rad}\mathbf{j}\|_2^2 } \cdot
	\sinc{v\left(\bs{\xi} + \veta - \frac{m}{2\rad}\mathbf{j}\right)}.
	\label{eq:23}
	\end{align}
	
	Now, by the assumptions that $m\geq 3$, $\rad \ge 8b \sqrt{\log n_\lambda}$, $v \le \rad$, $1\leq d\leq 64
	n_\lambda^{5/2} \log^{3/2} n_\lambda$, and $\lambda < n$, it follows from
	Lemma~\ref{lem:ft2} that the second term in \eqref{eq:fourierdiff}
	can be bounded as
	\begin{equation}
	\left|\Fc\left(\sum_{\substack{\mathbf{j}\in\ZZ^d\\ \|\mathbf{j}\|_\infty > \frac{m}{2}}} f_{\veta,b,v}
	\left(\frac{2\rad}{m}\mathbf{j}\right) \cdot \delta\left(\cdot - \frac{2\rad}{m}
	\mathbf{j}\right)\right)(\bs\xi) \right| \le \sqrt{\lambda n}. \label{eq:11}
	\end{equation}
	Thus, the desired result follows by combining \eqref{eq:fourierdiff},
	\eqref{eq:23}, and \eqref{eq:11}.	
\end{proof}

\subsection{Bounding \texorpdfstring{$\valpha^*\z(\bs\eta)$}{TEXTTOFIX}} \label{subsec:azbd}

\begin{lem}\label{lem:azlbd}
	For every odd integer $m \geq \max(64\log n_\lambda, 3)$ and $n = m^d \geq 17$ such that $$d \leq \min\left(\frac{\log n}{\log\log (18 \log n)},64 n_\lambda^{5/2} \log^{3/2} n_\lambda\right)\,,$$
any
	parameter $0 < \lambda \le
	\min\left\{(\frac{v}{2\rad})^{2d} \cdot n/1024, n^{1-\frac{1}{128}}\right\}$, every
	$\veta$ satisfying $\|\veta\|_\infty \le \frac{n^{1/d}}{10\rad}$, and any parameter $v \le {\rad}$ and
	$b = \frac{\rad}{8\sqrt{\log n_\lambda}}$, if $\valpha$ is defined as in \eqref{ins:alpha} of Definition
	\ref{def:alpha}, then we have
	\begin{equation*}
	|\valpha^*  \z (\veta)| \geq \frac{n}{4}\left(\frac{v}{2\rad}\right)^d.
	\end{equation*}
\end{lem}

\begin{proof}
	Since $v \le {\rad}$, $b = \frac{\rad}{8\sqrt{\log n_\lambda}}$, and $1\leq d \leq  64 n_\lambda^{5/2} \log^{3/2} n_\lambda$,
Lemma~\ref{lem:8} implies that
	\begin{align}
	&\left| \valpha^* \z (\veta) -
	\left(\frac{mv}{2\rad}\right)^d
	\sum_{\mathbf{j}\in\ZZ^d} \left(e^{-2\pi^2 b^2\|- \frac{m}{2\rad}\mathbf{j}\|_2^2
	} \cdot\sinc{v\left(-\frac{m}{2\rad}\mathbf{j}\right)} \right.\right. \nonumber\\
	&\qquad\qquad\qquad\qquad\qquad\qquad\qquad \left.\left. + e^{-2\pi^2 b^2 \|2\veta -
		\frac{m}{2\rad}\mathbf{j}\|_2^2 } \cdot \sinc{v\left(2\veta-\frac{m}{2\rad}\mathbf{j}\right)}\right) \right|  \le
	\sqrt{\lambda n}.
	\end{align}
	Hence, since $|\sinc{\cdot}| \leq 1$ and $\sinc{\cdot} \geq -\frac{1}{4}$, we have
	\begin{align}
	\left|\valpha^*  \z (\veta)\right| &\ge \left(\frac{mv}{2\rad}\right)^d  \left|\sum_{\mathbf{j}\in\ZZ^d}
	\left(e^{-2\pi^2 b^2 \|- \frac{m}{2\rad} \mathbf{j}\|_2^2} \cdot \sinc{v\left(-\frac{m}{2\rad}\mathbf{j}\right)} \right.
	\right. \nonumber\\
	&\qquad\qquad\qquad \left.\left. \vphantom{d} + e^{-2\pi^2 b^2 \|2\veta-\frac{m}{2\rad}\mathbf{j}\|_2^2 }
	\cdot\sinc{v\left(2\veta - \frac{m}{2\rad}\mathbf{j}\right)}\right)\right| - \sqrt{\lambda n}\nonumber\\
	&\ge \left(\frac{mv}{2\rad}\right)^d e^{-2\pi^2 b^2 \| \bs 0 \|_2^2 } \cdot\sinc{v(\bs 0)} + \left(
	\frac{mv}{2\rad}
	\right)^d e^{-2\pi^2 b^2 \|2\veta\|_2^2 } \cdot\sinc{v(2\veta)}\nonumber\\
	&\qquad\qquad\qquad - \left(\frac{mv}{2\rad}\right)^d \sum_{\substack{\mathbf{j}\in\ZZ^d \\ \mathbf{j} \neq \bs 0}}
	\left(e^{-2\pi^2 b^2 \|- \frac{m}{2\rad}\mathbf{j}\|_2^2 } + e^{-2\pi^2 b^2\|2\veta - \frac{m}{2\rad}\mathbf{j}\|_2^2
	}\right) - \sqrt{\lambda n} \nonumber\\
	&\geq \frac{3}{4} \left(\frac{mv}{2\rad}\right)^d - \left(\frac{mv}{2\rad}\right)^d \sum_{\substack{\mathbf{j}\in\ZZ^d
			\\ \mathbf{j} \neq \bs 0}}
	\left(e^{-2\pi^2 b^2 \|- \frac{m}{2\rad}\mathbf{j}\|_2^2 } + e^{-2\pi^2 b^2\|2\veta - \frac{m}{2\rad}\mathbf{j}\|_2^2
	}\right) - \sqrt{\lambda n}, \label{eq:400}
	\end{align}
	where $\bs 0 = (0,0,\dots,0)$ (the length-$d$ vector of all zeroes).
	
	Now we show that $\sum_{\mathbf{j} \in \ZZ^d, \mathbf{j}\neq \bs{0}} \left(e^{-2\pi^2 b^2 \|-\frac{m}{2\rad}\mathbf{j} \|_2^2 } +
	e^{-2\pi^2 b^2\|2\veta - \frac{m}{2\rad}\mathbf{j}\|_2^2 }\right)$ is small. Note that by the assumption that $b =
\frac{R}{8\sqrt{\log n_\lambda}}$, we have $e^{-2\pi^2 b^2 \|-
		\frac{m}{2\rad}\mathbf{j}\|_2^2 } \le e^{-\frac{1}{16} \cdot \frac{m^2}{\log n_\lambda} \|\mathbf{j}\|_2^2} \leq
	e^{-4m\|\mathbf{j}\|_1}$, since $m \geq 64 \log n_\lambda)$. Thus,
	\begin{align}
	\sum_{\substack{\mathbf{j}\in\ZZ^d\\ \mathbf{j}\neq\bs{0}}} e^{-2\pi^2 b^2 \|-\frac{m}{2\rad}\mathbf{j}\|_2^2} &\leq
	\sum_{\substack{\mathbf{j}\in\ZZ^d\\ \mathbf{j}\neq\bs{0}}} e^{-4m\|\mathbf{j}\|_1} \nonumber\\
	&= \sum_{\substack{\mathbf{j}\in\ZZ^d\\ \mathbf{j}\neq\bs{0}}} e^{-4m(|j_1| + |j_2| + \cdots + |j_d|)} \nonumber\\
	&\leq \left(\sum_{j_1=-\infty}^\infty e^{-4m|j_1|}\right)
	\left(\sum_{j_2=-\infty}^\infty e^{-4m|j_2|}\right) \cdots
	\left(\sum_{j_d=-\infty}^\infty e^{-4m|j_d|}\right) - 1 \nonumber\\
	&= \left(1 + \frac{2e^{-4m}}{1-e^{-4m}}\right)^d - 1 \nonumber\\
	&\leq \left(1 + 3e^{-4m}\right)^d - 1 \nonumber\\
	&\leq e^{3de^{-4m}} - 1 \nonumber\\
	&\leq 6de^{-4m}, \label{eq:jsum1}
	\end{align}
	since $d \leq \frac{\log n}{\log\log (18\log n)}$ implies that $6de^{-4m} < 1$. Moreover, recall
that $\InfNorm{\veta}
	\le \frac{m}{10\rad}$,
	and so, $\left\|2\veta - \frac{m}{2\rad}\mathbf{j}\right\|_2^2 \ge
	\|\frac{m}{4\rad}\mathbf{j}\|_2^2$ . Thus, in a similar fashion, we have
	\begin{align}
	\sum_{\substack{\mathbf{j}\in\ZZ^d\\ \mathbf{j}\neq\bs{0}}} e^{-2\pi^2 b^2\|2\veta - \frac{m}{2\rad}\mathbf{j}\|_2^2 }
	&\leq \sum_{\substack{\mathbf{j}\in\ZZ^d\\ \mathbf{j}\neq\bs{0}}} e^{-2\pi^2 b^2 \|\frac{m}{4\rad}\mathbf{j}\|_2^2} \nonumber\\
	&\leq \sum_{\substack{\mathbf{j}\in\ZZ^d\\ \mathbf{j}\neq\bs{0}}} e^{-m\|\mathbf{j}\|_1} \nonumber\\
	&\leq \left(1 + \frac{2e^{-m}}{1-e^{-m}}\right)^d - 1 \nonumber\\
	&\leq 6de^{-m}. \label{eq:jsum2}
	\end{align}
	Thus, combining \eqref{eq:400}, \eqref{eq:jsum1}, and \eqref{eq:jsum2}, we have
	\begin{align*}
	\left|\valpha^*  \z (\veta)\right| &\geq
	\left(\frac{mv}{2\rad}\right)^d \left(\frac{3}{4} - 6de^{-4m} - 6de^{-m}\right) - \sqrt{\lambda n}\\
	&\geq \frac{n}{4}\left(\frac{v}{2\rad}\right)^d,
	\end{align*}
	since $de^{-m} \leq (\log n) e^{-n^{\log\log(18\log n)/\log n}} = \frac{1}{18}$ (this is because by definition $n = m^d$ and $d \leq \log n/\log\log(18\log n)$ by assumption of the lemma and hence $m \geq n^{\log\log(18\log n)/\log n}$), and $de^{-4m} \leq
\frac{1}{8} de^{-m} \leq \frac{1}{144}$, as well as
\[
 \sqrt{\lambda n} \leq \sqrt{\frac{n}{1024}\left(\frac{v}{2R}\right)^{2d}\cdot n} =
\frac{n}{32}\left(\frac{v}{2R}\right)^d.
\]

\end{proof}

\subsection{Bounding \texorpdfstring{$\TNormS{\valpha}$}{TEXTTOFIX}} \label{subsec:alphanormbd}

\begin{lem} \label{lem:alpha2bound}
	For every odd integer $m\geq 3$ and parameters $n = m^d$, $\veta\in\RR^d$, and $b,v > 0$, if
	$\valpha$ is defined as in \eqref{ins:alpha} of Definition \ref{def:alpha}, then we have
	\[
	\|\valpha\|_2^2 \le 4n.
	\]
\end{lem}
\begin{proof}
	Let $w = 2\rad/m$. Then, letting $\mathbf{j} = (j_1,j_2,\dots,j_d)$, we
	observe that
	\begin{align}
	\|\valpha\|_2^2 &= \sum_{\mathbf{j}\in\{1,2,\dots,m\}^d}
	\valpha_{\mathbf{j}}^2 \nonumber\\
	&\leq \sum_{\substack{\mathbf{j}\in\ZZ^d\\ \|\mathbf{j}\|_\infty \leq \frac{m-1}{2}}}
	\left(\frac{2}{(\sqrt{2\pi}b)^d}\cos(2\pi w \veta^\T \mathbf{j})
	\prod_{i=1}^d \int_{j_i w - \frac{v}{2}}^{j_i w + \frac{v}{2}}
	e^{-x^2/2b^2}\,dx\right)^2\nonumber\\
	&\leq \sum_{\substack{\mathbf{j}\in\ZZ^d\\ \|\mathbf{j}\|_\infty \leq \frac{m-1}{2}}}
	\left(2\cos(2\pi w \veta^\T \mathbf{j}) \prod_{i=1}^d
	\int_{-\infty}^{\infty} \frac{1}{\sqrt{2\pi}b} e^{-x^2/2b^2}\,
	dx\right)^2\nonumber\\
	&\leq \sum_{\substack{\mathbf{j}\in\ZZ^d\\ \|\mathbf{j}\|_\infty \leq \frac{m-1}{2}}}
	\left(2 \prod_{i=1}^d
	\int_{-\infty}^{\infty} \frac{1}{\sqrt{2\pi}b} e^{-x^2/2b^2}\,
	dx\right)^2 \nonumber \\
	&\leq \sum_{\substack{\mathbf{j}\in\ZZ^d\\ \|\mathbf{j}\|_\infty \leq \frac{m-1}{2}}} 4 \nonumber\\
	&\leq 4m^d \nonumber\\
	&= 4n,
	\end{align}
	as desired.
\end{proof}

\subsection{Bounding \texorpdfstring{$\XNormS{\matPhi^\conj \valpha}{L_2(d\mu)}$}{TEXTTOFIX}}
\label{subsec:xnormbd}

Note that all the results so far hold for any kernel $p(\veta)$ and are
independent of the kernel function. Now, we upper bound
$\norm{\matPhi^*\valpha}_{L_2(d\mu)}$. This quantity depends on the
particular choice of kernel, which we assume to be Gaussian.

\begin{lem} \label{lem:alphabarnorm}
	For every odd integer $m \geq \max(64\log n_\lambda, 3)$ and parameters $n = m^d \geq 55$, $1\leq d\leq \frac{2\log n}{5\log\log n}$,
	$\frac{10}{n} < \lambda\le \frac{n}{1024} \left(\frac{1}{2}\right)^{2d}$, every $\veta$ satisfying
$\|\veta\|_\infty \le 100 \sqrt{\log
		n_\lambda}$, and any $2000 \cdot \log n_\lambda \le \rad \le \frac{m}{500\sqrt{\log n_\lambda}}$, and
	$b = \frac{\rad}{8\sqrt{\log n_\lambda}}$, if $\valpha$ is defined as in
	\eqref{ins:alpha} of Definition \ref{def:alpha} with parameter $v={\rad}$, then for the Gaussian kernel with
	$p(\bs\xi) = \frac{1}{(\sqrt{2\pi})^d} e^{-\|\bs\xi\|_2^2/2}$, we have:
	\begin{equation}
	\norm{\matPhi^*\valpha}_{L_2(d\mu)}^2 \le  8n^2 \left(\frac{3}{4\rad}\right)^d \cdot
	p(\veta) + 4{\lambda}{n} . \label{eq:gauskernelbd}
	\end{equation}
\end{lem}
\begin{proof}
	Recall that we set $v=\rad$. Thus, since $\lambda \leq \frac{n}{1024}\left(\frac{1}{2}\right)^{2d}$, we have
that $d \leq 64n_\lambda^{5/2} \log^{3/2} n_\lambda$, and so,
	Lemma~\ref{lem:8} implies that
	\begin{align*}
	|\valpha^* \z (\bs{\xi})|^2 &\le
	\left(\left|\left(\frac{mv}{2\rad}\right)^d \sum_{\mathbf{j} \in\ZZ^d}
	\left(e^{-2\pi^2 b^2
		\|\bs{\xi} - \veta - \frac{m}{2\rad}\mathbf{j}\|_2^2 } \cdot
	\sinc{v\left(\bs{\xi} - \veta - \frac{m}{2\rad}\mathbf{j}\right)}
	\right.\right.\right. \\
	&\qquad\qquad\qquad\quad \left.\left.\left. + e^{-2\pi^2 b^2
		\|\bs{\xi} + \veta - \frac{m}{2\rad}\mathbf{j}\|_2^2 } \cdot
	\sinc{v\left(\bs{\xi} + \veta -
		\frac{m}{2\rad}\mathbf{j}\right)} \right)
	\vphantom{\sum_{j_1,j_2,\dots,j_d\in\ZZ}}\right| + \sqrt{\lambda n} \, \right)^2 \\
	&\leq 2\left(\frac{mv}{2\rad}\right)^{2d}
	\left(\sum_{\mathbf{j} \in\ZZ^d}
	\left(e^{-2\pi^2 b^2
		\|\bs{\xi} - \veta - \frac{m}{2\rad}\mathbf{j}\|_2^2 } \cdot
	\sinc{v\left(\bs{\xi} - \veta - \frac{m}{2\rad}\mathbf{j}\right)}
	\right.\right. \\
	&\qquad\qquad\qquad\quad \left.\left. + e^{-2\pi^2 b^2
		\|\bs{\xi} + \veta - \frac{m}{2\rad}\mathbf{j}\|_2^2 } \cdot
	\sinc{v\left(\bs{\xi} + \veta -
		\frac{m}{2\rad}\mathbf{j}\right)}\right)
	\vphantom{\sum_{j_1,j_2,\dots,j_d\in\ZZ}} \right)^2 + 2(\sqrt{\lambda n})^2.
	\end{align*}
	Now, by the definition of the ${L_2(d\mu)}$ norm,
	$\norm{\matPhi^*\valpha}_{L_2(d\mu)}^2 = \int_{\RR^d} |\valpha^*
	\z(\bs{\xi})|^2 p(\bs{\xi}) \, d\bs{\xi}$, and so, we have
	
	\begin{align}
	\norm{\matPhi^* \valpha}_{L_2(d\mu)}^2 &\le
	\int_{\RR^d}
	2\left(\frac{mv}{2\rad}\right)^{2d} \left( \sum_{\mathbf{j} \in\ZZ^d}
	e^{-2\pi^2 b^2
		\|\bs{\xi} - \veta - \frac{m}{2\rad}\mathbf{j}\|_2^2 } \cdot
	\sinc{v\left(\bs{\xi} - \veta - \frac{m}{2\rad}\mathbf{j}\right)}
	\right. \nonumber\\
	&\quad \left. \vphantom{\sum_{j_1,j_2,\dots,j_d\in\ZZ}} + e^{-2\pi^2 b^2
		\|\bs{\xi} + \veta - \frac{m}{2\rad}\mathbf{j}\|_2^2 } \cdot
	\sinc{v\left(\bs{\xi} + \veta - \frac{m}{2\rad}\mathbf{j}\right)}\right)^2
	p(\bs{\xi}) \,d\bs{\xi} + \int_{\RR^d} 2\left(\sqrt{\lambda n}\right)^2 p(\bs{\xi}) \,d\bs{\xi}\nonumber\\
	&=
	8\left(\frac{mv}{2\rad}\right)^{2d}
	\int_{\RR^d}\left(\sum_{\mathbf{j} \in\ZZ^d}
	e^{-2\pi^2 b^2\left\|\bs{\xi} - \veta - \frac{m}{2\rad}\mathbf{j}\right\|_2^2
	} \cdot
	\sinc{v\left(\bs{\xi}-\veta-\frac{m}{2\rad}\mathbf{j}\right)}\right)^2
	p(\bs{\xi}) \,d\bs{\xi} + 2\lambda n,
	\label{eq:alphabarnorm}
	\end{align}
	where the last equality occurs because the kernel probability
	distribution function $p(\bs{\xi})$ is symmetric in our case, and the sum is
	over all $\mathbf{j}\in\ZZ^d$.
	Now, the integral in \eqref{eq:alphabarnorm} can be
	split into two integrals as follows:
	\begin{align}
	&\int_{\RR^d}\left(\sum_{\mathbf{j} \in\ZZ^d}
	e^{-2\pi^2 b^2\left\|\bs{\xi} - \veta - \frac{m}{2\rad}\mathbf{j}\right\|_2^2
	} \cdot
	\sinc{v\left(\bs{\xi}-\veta-\frac{m}{2\rad}\mathbf{j}\right)}\right)^2
	p(\bs{\xi}) \,d\bs{\xi} \nonumber\\
	&\qquad\quad =
	\int_{\|\bs\xi\|_\infty \leq 10\sqrt{\log
			n_\lambda}}\left(\sum_{\mathbf{j} \in\ZZ^d}
	e^{-2\pi^2 b^2\left\|\bs{\xi} - \veta - \frac{m}{2\rad}\mathbf{j}\right\|_2^2
	} \cdot
	\sinc{v\left(\bs{\xi}-\veta-\frac{m}{2\rad}\mathbf{j}\right)}\right)^2
	p(\bs{\xi}) \,d\bs{\xi}\nonumber\\
	&\qquad\quad\quad + \int_{\|\bs\xi\|_\infty \ge
		10\sqrt{\log n_\lambda}} \left(\sum_{\mathbf{j} \in\ZZ^d}
	e^{-2\pi^2 b^2\left\|\bs{\xi} - \veta - \frac{m}{2\rad}\mathbf{j}\right\|_2^2
	} \cdot
	\sinc{v\left(\bs{\xi}-\veta-\frac{m}{2\rad}\mathbf{j}\right)}\right)^2
	p(\bs{\xi}) \,d\bs{\xi}.
	\label{eq:integsplit}
	\end{align}
	First, we consider the case in which $\|\bs\xi\|_\infty \leq 10\sqrt{\log n_\lambda}$. By the assumption of the lemma,
	$\|\veta\|_\infty \le 100\sqrt{\log n_\lambda}$, and hence, $\|\bs\xi - \veta\|_\infty \le 110\sqrt{\log
		n_\lambda}$. This implies that $\|\bs\xi - \veta\|_\infty \le \frac{1}{2} (\frac{m}{2\rad})$, since we are assuming
	that $\rad \leq \frac{m}{500\sqrt{\log n_\lambda}}$. Therefore, for any $\mathbf{j} \neq (0,0,\dots,0)$, there exists
	some $k$ such that $j_k \neq 0$, and so,
	\begin{align}
	\sum_{\substack{\mathbf{j}\in\ZZ^d\\ \mathbf{j}\neq\bs{0}}} e^{-2\pi^2 b^2 \left\|\bs\xi - \veta - \frac{m}{2\rad}\mathbf{j}
		\right\|_2^2} &\leq \sum_{k=1}^d \sum_{\substack{\mathbf{j} \in \ZZ^d\\ j_k\neq 0}}  e^{-2\pi^2 b^2 \left\|\bs\xi -
		\veta - \frac{m}{2\rad}\mathbf{j} \right\|_2^2} \nonumber\\
	&= \sum_{k=1}^d \left[ \left(\sum_{|j_k|\geq 1} e^{-2\pi^2 b^2 \left(\xi_k - \eta_k - \frac{m}{2\rad}j_k\right)^2}
	\right) \prod_{\substack{1\leq i\leq d\\ i\neq k}} \sum_{j_i=-\infty}^\infty e^{-2\pi^2 b^2 \left(\xi_i - \eta_i -
		\frac{m}{2\rad}j_i\right)^2} \right] \nonumber\\
	&\leq \sum_{k=1}^d \left[\left(\sum_{|j_k|\geq 1} e^{-\frac{\pi^2 b^2 m^2}{8\rad^2} (2|j_k|-1)^2}\right)
	\prod_{\substack{1\leq i \leq d\\ i\neq k}} \left( 1 + \sum_{|j_i|\geq 1} e^{-\frac{\pi^2 b^2 m^2}{8\rad^2}
		(2|j_i|-1)^2} \right) \right] \nonumber\\
	&\leq \sum_{k=1}^d \left[\left(2 \sum_{j_k=1}^\infty e^{-\frac{\pi^2 b^2 m^2}{8\rad^2} j_k}\right)
	\prod_{\substack{1\leq i \leq d\\ i\neq k}} \left( 1 + 2 \sum_{j_i=1}^\infty e^{-\frac{\pi^2 b^2 m^2}{8\rad^2}
		j_i} \right) \right] \nonumber\\
	&\leq \sum_{k=1}^d \left[\left(2 \sum_{j_k=1}^\infty e^{-m j_k}\right) \prod_{\substack{1\leq i \leq d\\ i\neq
			k}} \left( 1 + 2 \sum_{j_i=1}^\infty e^{-m j_i}\right) \right] \nonumber\\
	&\leq d \left(4 e^{-m}\right)(1 + 4e^{-m})^{d-1} \nonumber\\
	&\leq 2e^{-m/2}, \label{eq:35}
	\end{align}
	where we have used the assumptions $b = \frac{\rad}{8\sqrt{\log n_\lambda}}$ and $m \geq \max(64\log n_\lambda, 3)$, as well as the fact
that $d\leq
	\sqrt{m}$ (which follows from the fact that $d \leq \frac{2\log n}{5\log \log n}$).
	
	Now, using \eqref{eq:35}, we see that the first integral in
	\eqref{eq:integsplit} can be bounded as follows:
	\begin{align}
	&\int_{\|\bs{\xi}\|_\infty \leq 10\sqrt{\log n_\lambda}} p(\bs\xi)
	\left(\sum_{\mathbf{j}\in\ZZ^d} e^{-2\pi^2 b^2 \left\|\bs\xi - \veta -
		\frac{m}{2\rad}\mathbf{j} \right\|_2^2} \cdot \sinc{v\left(\bs\xi - \veta -
		\frac{m}{2\rad}\mathbf{j}\right)}\right)^2\,d\bs\xi\nonumber\\
	&\qquad \le 2\int_{\|\bs{\xi}\|_\infty \leq 10\sqrt{\log n_\lambda}} p(\bs\xi)
	\left(e^{-2\pi^2 b^2 \|\bs\xi - \veta\|_2^2} \cdot
	\sinc{v(\bs\xi -\veta)}^2\right)^2\,d\bs\xi \nonumber\\
	&\qquad\qquad + 2\int_{\|\bs\xi\|_\infty \leq 10 \sqrt{\log n_\lambda}}
	p(\bs\xi) \left(\sum_{\mathbf{j} \neq \bs{0}} e^{-2\pi^2 b^2 \left\|\bs\xi - \veta
		-\frac{m}{2\rad} \mathbf{j} \right\|_2^2} \cdot\sinc{v\left(\bs\xi - \veta -
		\frac{m}{2\rad}\mathbf{j}\right)}
	\right)^2 \,d\bs\xi \nonumber\\
	&\qquad \leq 2\int_{\|\bs{\xi}\|_\infty \leq 10\sqrt{\log n_\lambda}}
	\frac{1}{(\sqrt{2\pi})^d} e^{-\|\bs\xi\|_2^2/2} \left(e^{-2\pi^2 b^2
		\|\bs\xi - \veta\|_2^2} \sinc{v(\bs\xi -\veta)}^2 + 4
	e^{-m/2}\right)\,d \bs\xi \nonumber\\
	&\qquad = 2\int_{\|\bs{\xi}\|_\infty \leq 10\sqrt{\log n_\lambda}}
	\frac{1}{(\sqrt{2\pi})^d}e^{-\|\bs\xi\|_2^2/2}
	e^{-b^2 \|\bs\xi - \veta\|_2^2} \cdot \sinc{v(\bs\xi-\veta)}^2 \,d\bs\xi \nonumber\\
	&\qquad\qquad + 8\int_{\RR^d} \frac{1}{(\sqrt{2\pi})^d} e^{-\|\bs\xi\|_2^2/2} e^{-m/2}\,d\bs\xi
	\nonumber\\
	&\qquad \le \int_{\|\bs{\xi}\|_\infty \leq 10\sqrt{\log n_\lambda}}
	\frac{1}{(\sqrt{2\pi})^d} e^{-\|\bs\xi\|_2^2/2}
	e^{-b^2 \|\bs\xi - \veta\|_2^2} \cdot \sinc{v(\bs\xi-\veta)}^2 \,d\bs\xi  +
	8e^{-m/2}. \label{eq:firstintegral}
	\end{align}
	
	Next, by Claim~\ref{claim:blah}, we have $e^{-\|\bs\xi\|_2^2/2} \le  3^d e^{-\|\veta\|_2^2/2}$
	for $\|\bs\xi-\veta\|_\infty \le \frac{10\sqrt{\log n_\lambda}}{b}$ (since $b \geq \rad \geq 1000
	\log n_\lambda$). Hence,
	
	\begin{align}
	&\int_{\substack{\|\bs\xi-\veta\|_\infty \leq \frac{10\sqrt{\log
					n_\lambda}}{b}\\ \|\bs\xi\|_\infty \leq 10\sqrt{\log n_\lambda}}}
	\frac{1}{(\sqrt{2\pi})^d} e^{-\|\bs\xi\|_2^2/2} e^{-b^2 \|\bs\xi -
		\veta\|_2^2} \cdot
	\sinc{v(\bs\xi-\veta)}^2 \,d\bs\xi \nonumber\\
	&\qquad\qquad\qquad\qquad\qquad \le  3^d \cdot \frac{1}{(\sqrt{2\pi})^d}
	e^{-\|\veta\|_2^2/2} \int_{\RR^d}
	e^{-b^2 \|\bs\xi - \veta\|_2^2} \cdot \sinc{v(\bs\xi-\veta)}^2 \,d\bs\xi
	\nonumber\\
	&\qquad\qquad\qquad\qquad\qquad \le 3^d \cdot \frac{1}{(\sqrt{2\pi})^d}
	e^{-\|\veta\|_2^2/2} \int_{\RR^d} \sinc{v(\bs\xi-\veta)}^2 \,d\bs\xi
	\nonumber\\
	&\qquad\qquad\qquad\qquad\qquad = \frac{3^d p(\veta)}{v^d} \label{eq:integbd}
	\end{align}
	Note that the last line follows from the fact that $v^d\cdot\sinc{v(\cdot)}$ is
	the Fourier transform of $\rect_v$, and so, by the convolution theorem
	(Claim~\ref{claim:convthm}), we have
	\begin{align*}
	\int_{\RR^d} (v^d\cdot\sinc{v \t})^2\,d\t &= \left(\rect_v \ast
	\rect_v\right)(0)\\
	&= v^d.
	\end{align*}
	Moreover,
	\begin{align}
	\int_{\substack{\|\bs\xi-\veta\|_\infty \geq \frac{10\sqrt{\log n_\lambda}}{b}\\ \|\bs\xi\|_\infty \leq 10\sqrt{\log
				n_\lambda}}}
	\frac{1}{(\sqrt{2\pi})^d} e^{-\|\bs\xi\|_2^2/2}
	e^{-b^2 \|\bs\xi-\veta\|_2^2}\cdot\sinc{v(\bs\xi-\veta)}^2\,d\bs\xi &\leq
	n_\lambda^{-100} \int_{\RR^d}
	\frac{1}{(\sqrt{2\pi})^d} e^{-\|\bs\xi\|_2^2/2}\,d\bs\xi \nonumber\\
	&= n_\lambda^{-100}, \label{eq:integubd}
	\end{align}
	since $\|\bs\xi - \veta\|_2 \geq \|\bs\xi-\veta\|_\infty$. Thus,
	\eqref{eq:firstintegral}, \eqref{eq:integbd}, and \eqref{eq:integubd} imply that
	\begin{multline}
	\int_{\|\bs\xi\|_\infty \leq 10\sqrt{\log n_\lambda}} \left(\sum_{\mathbf{j}\in\RR^d} e^{-2\pi^2 b^2 \left\|\bs\xi -
		\veta
		- \frac{m}{2\rad}\mathbf{j}\right\|_2^2} \cdot \sinc{v\left(\bs\xi - \veta - \frac{m}{2\rad}\mathbf{j}
		\right)}\right)^2 p(\bs\xi)\,d\bs\xi\\ \leq \frac{3^d p(\veta)}{v^d} + n_\lambda^{-100} + 8e^{-m/2}.
	\label{eq:integ1}
	\end{multline}
	
	Next, we bound the second integral in \eqref{eq:integsplit}. We first show that the quantity in parentheses is upper bounded by a constant for all $\bs\xi$ in the appropriate range, and then use this bound to upper bound the integral itself.
	Consider $\bs\xi$ satisfying $\|\bs\xi\|_\infty \ge 10\sqrt{\log  n_\lambda}$.
	Let $t_i$, for $i=1,\dots,d$, be an integer such that $|\xi_i-\eta_i-t_i m/2R| \leq m/4R$.
	Note that the following upper bound holds:
	\begin{align*}
	\left|\sum_{\mathbf{j}\in\ZZ^d} e^{-2\pi^2 b^2 \left\|\bs\xi - \veta - \frac{m}{2\rad}\mathbf{j}\right\|_2^2} \cdot
	\sinc{v\left(\bs\xi-\veta-\frac{m}{2\rad}\mathbf{j}\right)}\right| &\leq \sum_{\mathbf{j}\in\ZZ^d} e^{-2\pi^2
		b^2\left\|\bs\xi-\veta - \frac{m}{2\rad}\mathbf{j}\right\|_2^2} \\
	&\leq \prod_{i=1}^d \sum_{j_i=-\infty}^\infty e^{-2\pi^2 b^2 \left(\xi_i-\eta_i-\frac{m}{2\rad}j_i\right)^2} \\
	&\leq \prod_{i=1}^d \left(1 + \sum_{k \neq t_i} e^{-2\pi^2 b^2 \left(\xi_i-\eta_i-\frac{m}{2\rad}k\right)^2}\right) \\
	&\leq \prod_{i=1}^d \left(1 + \frac{2\rad}{m} \int_{-\infty}^\infty e^{-2\pi^2 b^2 (\xi_i-\eta_i-t)^2}\,dt\right) \\
	&\leq \left(1 + \frac{\rad}{mb}\right)^d \\
	&\leq e^{\rad d/mb} \\
	&\leq e^{8d\sqrt{\log n_\lambda}/m} \\
	&\leq e^{\frac{1}{16\log \log n}}\\
	&\leq 4,
	\end{align*}
since $d\leq \frac{\log n}{2\log\log n}$ and $m \geq 64\log(n) \sqrt{\log n_\lambda}$.
	Thus, we can bound the second integral in \eqref{eq:integsplit} as follows:
	\begin{align}
	&\int_{\|\bs\xi\|_\infty \geq 10\sqrt{\log n_\lambda}} \left( \sum_{\mathbf{j}\in\ZZ^d} e^{-2\pi^2 b^2
		\left\|\bs{\xi}-\veta-\frac{m}{2\rad}\mathbf{j}\right\|_2^2} \cdot \sinc{v\left(\bs\xi-\veta -
		\frac{m}{2\rad}\mathbf{j}\right)}\right)^2 p(\bs{\xi}) \,d\bs{\xi} \nonumber\\
	&\qquad\qquad\qquad\qquad \leq 16\int_{\|\bs\xi\|_\infty \geq 10\sqrt{\log n_\lambda}} p(\bs{\xi}) \,d\bs{\xi}
	\nonumber\\
	&\qquad\qquad\qquad\qquad \leq 16 \sum_{k=1}^d \left(2 \int_{10\sqrt{\log n_\lambda}}^\infty
\frac{1}{\sqrt{2\pi}}
	e^{-\xi_k^2/2}\,d\xi_k\right) \prod_{i\neq k} \int_{-\infty}^\infty \frac{1}{\sqrt{2\pi}} e^{-\xi_i^2/2}\,d\xi_i
	\nonumber\\
	&\qquad\qquad\qquad\qquad \leq \frac{32d}{\sqrt{2\pi}} \cdot \frac{n_\lambda^{-50}}{10\sqrt{\log n_\lambda}}
\nonumber\\
	&\qquad\qquad\qquad\qquad \leq n_{\lambda}^{-25}, \label{eq:integ2}
	\end{align}
	by Claim~\ref{claim:cdfnormal} as well as the facts that $d\leq \frac{2\log n}{5\log\log n}$ and $\frac{10}{n} < \lambda\le \frac{n}{1024} \left(\frac{1}{2}\right)^{2d}$.
	
	Combining \eqref{eq:alphabarnorm}, \eqref{eq:integsplit}, \eqref{eq:integ1},
	and \eqref{eq:integ2} now imply that
	\begin{align*}
	\XNormS{\matPhi^\conj \valpha}{L_2(d\mu)} &\leq 8\left(\frac{mv}{2\rad}\right)^{2d} \left(\frac{3^d
		p(\veta)}{v^d} + n_\lambda^{-100} + 8e^{-m/2} + n_\lambda^{-25}\right) + 2\lambda n\\
	&= 8 \left(\frac{m}{2}\right)^{2d} \left(\frac{3^d p(\veta)}{\rad^d} + n_\lambda^{-100} + 8e^{-m/2} +
	n_\lambda^{-25}\right) + 2\lambda n\\
	&\leq 8n^2\left(\frac{3}{4\rad}\right)^d \cdot p(\veta) + 4\lambda n,
	\end{align*}
	as desired. In the above, the last inequality follows from $\frac{1}{4^d}(n_\lambda^{-100}+n_\lambda^{-25} + 8e^{-m/2}) \leq
	\frac{1}{4}n_\lambda^{-1}$, which follows from the fact that $n \geq 55$, $m = n^{1/d}$, and $d \leq \frac{2\log n}{5\log\log n}$,
	so $m \geq \log^{5/2} n$.
	
\end{proof}

\begin{proof}[Proof of Theorem~\ref{thm:main-lev-lb}]\label{thrm:lowerbound}
	Note that we can choose data points $\x_1, \x_2, \dots, \x_n$ and the vector $\valpha$ according to the
	construction in Definition~\ref{def:alpha} with $v=\rad$ and $b = \frac{\rad}{8\sqrt{\log n_\lambda}}$.
	Thus, Lemmas \ref{lem:azlbd}, \ref{lem:alpha2bound}, and \ref{lem:alphabarnorm}, as well as \eqref{eq:taulamb}, imply
	that
	\begin{align*}
	\tau_\lambda(\veta) &\geq \frac{p(\veta)\cdot|\valpha^\conj \z(\veta)|^2}{\XNormS{\matPhi^\conj
			\valpha}{L_2(d\mu)}+\lambda \TNormS{\valpha}}\\
	&\geq \frac{p(\veta) \cdot
		\left(\frac{n}{4}\left(\frac{1}{2}\right)^d\right)^2}{8n^2\left(\frac{3}{4\rad}\right)^d p(\veta) +
		4\lambda n + \lambda (4n)}\\
	&\geq \frac{1}{128}\left( \frac{\rad}{3}\right)^d \cdot \frac{p(\veta)}{p(\veta) + (4\rad/3)^d n_\lambda^{-1}},
	\end{align*}
	as desired.
\end{proof}

\section{Proof of Corollary~\ref{cor:sd-bound}}\label{sec:cor-sd-bound}
In the proof of the corollary we often need to compute the volume of a d-dimensional ball hence we state it as a claim.

\begin{claim}\label{d-dmiensional-ball}
For any integer $d \ge 1$ the following holds:
$$\int_{\substack{\veta \in \RR^d \\ \|\veta\|_2 \le R}} 1 d\veta = \frac{(\sqrt{\pi}R)^d}{\Gamma(d/2+1)}$$
where $\Gamma$ is the Gamma function.
\end{claim}
	\paragraph{First claim of the corollary (upper bound on statistical
		dimension):}
	Let $t=10\sqrt{\log n_\lambda}$ and . We have:
	\begin{align*}
	s_\lambda = \int_{\mathbb{R}^d} \tau(\veta) d\veta = \int_{\substack{\veta \in \RR^d\\ \|\veta\|_2 \le t}} \tau(\veta) d\veta + \int_{\substack{\veta \in \RR^d\\ \|\veta\|_2 > t}} \tau(\veta) d\veta
	\end{align*}
	By the naive bound in Proposition \ref{prop:simple-tau-bound} we have:
	\begin{align}
	\int_{\substack{\veta \in \RR^d\\ \|\veta\|_2 > t}} \tau(\veta) d\veta &\le n_\lambda \int_{\substack{\veta \in \RR^d\\ \|\veta\|_2 > t}}  e^{-\frac{\|\veta\|_2^2}{2}} d \veta\nonumber\\
	&=  n_\lambda \Big( \prod_{i=1}^{d-1} \int_{\theta_i \in [0,2\pi]} d\theta_i \Big)\int_{ [-\infty,-t]\cup[t,\infty]} r^{d-1}e^{-{r^2}/2} d r\nonumber\\
	&= (\sqrt{2\pi})^d n_\lambda \int_{ [-\infty,-t]\cup[t,\infty]}  \frac{r^{d-1}}{\sqrt{2\pi}}e^{-{r^2}/2} d r\nonumber\\
	&\le (\sqrt{2\pi})^d n_\lambda \int_{ [-\infty,-t]\cup[t,\infty]}  \frac{1}{\sqrt{2\pi}}e^{-{r^2}/4} d r\nonumber\\
	&\le (\sqrt{2\pi})^d  n_\lambda \cdot  \left (\frac{e^{-t^2}}{t} \right)\nonumber\\
	&\le (\sqrt{2\pi})^d \label{highFreqBound}
	\end{align}
	where the first equality follows by converting from polar coordinates to cartesian coordinates. The second inequality uses the fact that if $d\le \frac{t^2}{4 \log t}$ then for all $r$ with $|r| \ge t$ we have ${r^{d-1}}e^{-{r^2}/2} \le e^{-{r^2}/4}$ which holds true by the assumption of the lemma. To see this note that for $r$ with $|r| \ge t$:
	
	$$ r^{d-1} = e^{(d-1) \log r} \le e^{\frac{t^2}{4 \log t} \log r} \le e^{\frac{r^2}{4 \log r} \log r} = e^{r^2/4}$$		
	and therefore, ${r^{d-1}}e^{-{r^2}/2} \le e^{-{r^2}/4}$.
	
	Further, by the refined bound of Theorem \ref{thm:lev-scores-ub}, for any $\veta$ with $\|\veta\|_\infty \le 10\sqrt{\log n_\lambda} = t$ and hence $\|\veta\|_2 \le 10\sqrt{\log n_\lambda} = t$ we have
	\begin{align}
	\int_{\substack{\veta \in \RR^d\\ \|\veta\|_2 \le t}} \tau(\veta) d\veta &\le  \int_{\substack{\veta \in \RR^d\\ \|\veta\|_2 \le t}} \left(\Big( 12.4\max(\rad, 2000 \log^{1.5} n_\lambda) \Big)^{d} +1\right) d\veta\nonumber\\
	&\le (2t)^d / \Gamma(d/2+1) \cdot \left(\Big( 12.4\max(\rad, 2000 \log^{1.5} n_\lambda) \Big)^{d} +1 \right) \nonumber\\
	&= \Big( 20 \sqrt{\log n_\lambda} \Big)^d \left( \Big( 12.4\max(\rad, 2000 \log^{1.5} n_\lambda) \Big)^{d} + 1\right) \Big/ \Gamma(d/2+1).\label{lowFreqBound}
	\end{align}
	The second inequality follows from Claim \ref{d-dmiensional-ball}. Combining \eqref{highFreqBound} and \eqref{lowFreqBound} gives the lemma.
	
	\paragraph{Second claim of the corollary:}
	
	We use the same construction of points as in Theorem \ref{thm:main-lev-lb}.
	Note that for all $\|\veta\|_2 \le \sqrt{2 \log \frac{n_\lambda}{\rad^d}}$ we have $p(\veta) \ge (\frac{\rad}{\sqrt{2\pi}})^d /n_\lambda$, hence we have:
	$$p(\veta) + (4\rad/3)^d n_\lambda^{-1} \le 6^d p(\eta)$$
	Hence, by Theorem \ref{thm:main-lev-lb}, we have:
	$$\tau(\eta) \ge \frac{1}{1024}\left( \frac{\rad}{18}\right)^d$$
	therefore,
	\begin{align}
	s_\lambda(\bv{K}) &= \int_{\mathbb{R}^d}\tau(\veta) d\veta \nonumber\\
	&\ge \int_{\|\veta\|_2 \le \sqrt{2 \log \frac{n_\lambda}{\rad^d}} } \frac{1}{1024}\left( \frac{\rad}{18}\right)^d d\veta \nonumber\\
	&= \Omega\left( \left( \frac{\sqrt{\pi}\rad}{18} \sqrt{ \log \frac{n_\lambda}{\rad^d}} \right)^d \Big/ \Gamma(d/2+1) \right)
	\end{align}
	The inequality above is because $\tau$ is a non-negative function everywhere. The last equality is due to Claim \ref{d-dmiensional-ball}.
\section{Proof of Theorem~\ref{thm:rr-samples}}
\label{appendix:rr-sampling}

We now show our lower bound on the number of samples required for spectral approximation using classical random Fourier features. This bound is closely related to the leverage score lower bound of Theorem \ref{thm:main-lev-lb} and the leverage score characterization given by the maximization problem in Lemma \ref{lem:altlev-lb}.

Our goal is to show that if we take $s$ samples $\veta_1, \veta_2, \dots, \veta_s$
from the distribution defined by $p$, for $s$ too small, then there is an
$\bs\alpha = (\alpha_1, \alpha_2, \dots, \alpha_n) \in \mathbb{R}^n$ such that with
at least constant probability,
\begin{equation}
\bs\alpha^\T (\matK+ \lambda \matI_n)\bs\alpha < \frac{2}{3} \bs\alpha^\T
(\matZ\matZ^\conj+\lambda\matI_n)\bs\alpha. \label{eq:rrineq}
\end{equation}

Informally, a frequency $\veta$ with high ridge leverage score implies by Lemma \ref{lem:altlev-lb} the existence of $\valpha$ which is concentrated at $\veta$ (i.e. $|\bv{z}(\veta)^* \valpha|^2$ is large compared to $\XNormS{\matPhi^\conj \valpha}{L_2(d\mu)}+\lambda \TNormS{\valpha}$.) If $\veta$ is not sampled with high enough probability then $ \bs\alpha^\T (\matK+ \lambda \matI_n)\bs\alpha$ will not be well approximated. Formally,
by \eqref{eq:bochner}:
\begin{align*}
\bs\alpha^T \matK \bs\alpha &= \sum_{j,k} \alpha_j \alpha_k \cdot k(\x_j, \x_k)\\
&= \sum_{j,k} \int_{\RR^d} e^{-2\pi i \veta^T (\x_j-\x_k)} \alpha_j
\alpha_k p(\veta)\,d\veta\\
&= \int_{\RR^d} \left(\sum_{j=1}^n \alpha_j e^{-2\pi i \veta^T \x_j
}\right) \left(\sum_{k=1}^n \alpha_k e^{2\pi i \veta^T \x_k}\right) p(\veta)\,d\veta\\
&= \int_{\RR^d} p(\veta)\left|\sum_{j=1}^n \alpha_j e^{2\pi i \veta^T \x_j}
\right|^2\,d\veta.
\end{align*}
Also, by the definition of $\matZ$ and $\varphi$ (see
Section~\ref{sec:random-fourier-features}), we have
\begin{align*}
\bs\alpha^T \matZ\matZ^\conj \bs\alpha &= \left\|\sum_{j=1}^n \alpha_j
\varphi(\x_j)\right\|_2^2\\
&= \sum_{k=1}^s \left| \sum_{j=1}^n \alpha_j \cdot \frac{1}{\sqrt{s}} e^{2\pi
	i \veta_k^T \x_j} \right|^2\\
&= \frac{1}{s} \sum_{k=1}^s \left|\sum_{j=1}^n \alpha_j e^{2\pi i \veta_k^T
	\x_j}\right|^2,
\end{align*}
where $\veta_1, \veta_2, \dots, \veta_s$ are the $s$ samples from the
distribution given by $p$. Hence, \eqref{eq:rrineq} is equivalent to
\begin{equation}
\int_{\RR^d} p(\veta)\left|\sum_{j=1}^n
\alpha_j e^{2\pi i \veta^T \x_j}\right|^2\,d\veta + \frac{1}{3} \lambda
\|\bs\alpha\|_2^2 < \frac{2}{3} \cdot \frac{1}{s} \sum_{k=1}^s \left|\sum_{j=1}^n
\alpha_j e^{2\pi i \veta_k \x_j}\right|^2. \label{eq:rrintform}
\end{equation}

We again use the same construction of $n$ data points $\x_1, \x_2, \dots,
\x_n \in \RR^d$, according to the construction in Definition
\ref{def:alpha}. Moreover, we define $\veta^*$ to
be
\[
\veta^* = {\arg\max}_{\veta \in \{\veta_1, \veta_2, \dots, \veta_s\}} \|\veta\|_2.
\]
% \[
%  \veta^* = (\eta_1^*, \eta_2^*, \dots, \eta_d^*),
% \]
% where for each $j=1,2,\dots,d$,
% \[
%  \eta_j^* = \max_{1\leq k\leq s} |\eta_{k,j}|.
% \]
We also let $\bs\alpha = (\alpha_1, \alpha_2, \dots, \alpha_n)$ be given by
\[
\alpha_j = f_{\veta^*,b,\rad}(\x_j),
\]
where $b = \rad/8\sqrt{\log(n/\lambda)}$. We show
that this choice of data points and $\bs\alpha$ satisfies \eqref{eq:rrintform}
with high probability.

\begin{lem} \label{lem:linfbd}
	Under the preconditions of Theorem \ref{thm:rr-samples}, with probability 0.99  over the samples we have $\|\veta^*\|_\infty \leq 80\sqrt{\log n_\lambda}$.
\end{lem}
\begin{proof}
	Let $\eta_1$ be a random variable with density $p(\eta_1) = (2\pi)^{-1/2} e^{-\eta_1^2 / 2}$. The limits on $n$ and $\lambda$ alongside Claim~\ref{claim:cdfnormal} imply that $\Pr(|\eta_1| \geq 80 \sqrt{\log n_\lambda}) < n^{-129}_\lambda / 100$. Now, consider the $s d$ different entires in $\veta_1, \dots, \veta_s$. Each of these entries are distributed identically as $\eta_1$, so by union-bound the probability that the maximum value is bigger than $80 \sqrt{\log n_\lambda}$ is bounded by $s d n^{-129}_\lambda / 100$. Since $s \leq n_\lambda$ and $d \leq n \leq n^{128}_\lambda$, we have $s d \leq n^{129}_\lambda $, so the the probability that the maximum value is bigger than $80 \sqrt{\log n_\lambda}$ is bounded by $1 / 100$. The lemma now follows by observing that $\|\veta^*\|_\infty$ is smaller than this value.
\end{proof}

First, we upper bound the first term on the left side of \eqref{eq:rrintform}.
Note that by Lemmas~\ref{lem:linfbd} and \ref{lem:alphabarnorm}, with probability at least $0.99$
over the samples $\veta_1, \veta_2, \dots, \veta_s$, we have
\begin{align}
\int_{\RR^d} p(\veta)\left|\sum_{j=1}^n
\alpha_j e^{2\pi i \veta^\T \x_j}\right|^2\,d\veta &= \int_{\RR^d}
\frac{1}{(\sqrt{2\pi})^d} e^{-\|\veta\|_2^2/2}
\left|\sum_{j=1}^n \alpha_j e^{2\pi i \veta^\T \x_j}\right|^2\,d\veta \nonumber\\
&= \|\matPhi^* \bs\alpha\|_{L_2(d\mu)}^2 \nonumber\\
&\leq 8n^2 \left(\frac{3}{4\rad}\right)^d \cdot p(\veta^*) + 3\lambda n \nonumber,
\end{align}
where we have let $\veta = \veta^*$. Now, in order to estimate $p(\veta^*)$, note that by
Claims~\ref{claim:gaussiansamp} and \ref{claim:gaussiansamphighdim}, we have that with probability at least $1-e^{-1}$ over
the samples $\veta_1, \veta_2, \dots, \veta_s$,
\[
p(\veta^*) \leq \frac{B_d}{(\log s)^{\frac{d-2}{2}} s},
\]
where
\[
B_d = \begin{cases} 8\quad &\text{if $d=1$}\\ \frac{(d-1)^{\frac{d-1}{2}}}{(2\pi)^{d/2}} \quad &\text{if $d > 1$} \end{cases}.
\]
Thus, with probability at least $1 - e^{-1} - 1/100 \geq 0.5$, we have
\begin{align}
\int_{\RR^d} p(\veta)\left|\sum_{j=1}^n
\alpha_j e^{2\pi i \veta^\T \x_j}\right|^2\,d\veta &\leq 8n^2 \left(\frac{3}{4\rad}\right)^d \cdot
\frac{B_d}{(\log s)^{\frac{d-2}{2}} s} + 3\lambda n. \label{eq:ekbd}
\end{align}

Next, we bound the right side of \eqref{eq:rrintform} from below. Note that by $b = \rad/8\sqrt{\log(n/\lambda)}$ and with the choice of $v=\rad$ and $\veta = \veta^*$, Lemma~\ref{lem:azlbd} holds true. Therefore we have,
\begin{align}
\frac{1}{s}\sum_{k=1}^s \left|\sum_{j=1}^{n} \alpha_j e^{2\pi i \veta^\T_k \x_j}
\right|^2 &\geq \frac{1}{s} \left|\sum_{j=1}^{n} \alpha_j e^{2\pi i \veta^* \cdot \x_j}
\right|^2 \nonumber\\
&= \frac{1}{s} |\bs\alpha^* \z (\veta^*)|^2
\nonumber\\
&\geq \frac{1}{s} \left(\frac{n}{4\cdot 2^d}\right)^2 = \frac{n^2}{2^{2d+4} s},
\label{eq:alphazbound}
\end{align}
by Lemma~\ref{lem:azlbd}.

We also require the following estimate of $\|\bs\alpha\|_2^2$, which is provided
by Lemma~\ref{lem:alpha2bound}:
\begin{align}
\|\bs\alpha\|_2^2 \leq 4n. \label{eq:anorm}
\end{align}

We also need the bound:
\begin{equation}
8n^2 \left(\frac{3}{4\rad}\right)^d \cdot \frac{B_d}{(\log s)^{\frac{d-2}{2}} s} \leq \frac{n^2}{3\cdot 2^{2d+4}s}
\label{eq:annoyingbound}
\end{equation}
This bound obviously holds if
$$
8 \left(\frac{3}{\rad}\right)^d \cdot \frac{B_d}{(\log s)^{\frac{d-2}{2}}} \leq \frac{1}{3\cdot 2^4}
$$
We now distinguish between the case of $d=1$ and the case of $d > 1$.  For $d=1$, the inequality is
$$
\frac{192\sqrt{\log s}}{\rad}\leq \frac{1}{48}
$$
Now, the conditions $R\geq \log 2000 n_\lambda$ and $s \leq n_\lambda/832$ imply that $s\leq \exp(R/2000)/832 \leq \exp((R/10000)^2)$ and so
$$
\frac{192\sqrt{\log s}}{\rad}\leq \frac{192}{10000} \leq \frac{1}{48}
$$
as required. For $d > 1$, we first note that the conditions on $\rad$ imply that $\rad \geq 40$ and $d<\rad/25$ ($n_\lambda \geq 17$ and $\lambda \leq n^{1-1/128}$ imply that
$R \geq 2000 \log n_\lambda \geq 15 \log n$ and $d \leq 2 \log n / \log \log n \leq 0.4 \log n$ from which the bound $d < \rad / 25$ follows) so we have
\begin{align*}
8 \left(\frac{3}{\rad}\right)^d \cdot \frac{B_d}{(\log s)^{\frac{d-2}{2}}} &= 8 \left(\frac{3}{\rad}\right)^d \cdot \frac{(d-1)^{\frac{d-1}{2}}}{(2\pi)^{d/2}(\log s)^{\frac{d-2}{2}}} \\
		 &\leq \frac{8 \cdot (3/2)^{d/2}\cdot(d-1)^{\frac{d-1}{2}}}{\rad^d (\log s)^{\frac{d-2}{2}}}\\
		 & \leq \frac{8 \cdot (3/2)^{d/2}\cdot(d-1)^{\frac{d-1}{2}}}{\rad^d} \\
		 & \leq \frac{8 \cdot (3/2)^{d/2}}{\rad^{d/2} 25^{d/2}} \\
		 & \leq \frac{8 \cdot (3/2)^{d/2}}{40^{d/2} 25^{d/2}} \\
		 & \leq \frac{8 \cdot (3/2)^{d/2}}{30^{d}} \\
		 & \leq 1/48
\end{align*}
as required.

Finally, by combining \eqref{eq:ekbd}, \eqref{eq:alphazbound},
\eqref{eq:anorm}, and \eqref{eq:annoyingbound} we have that with probability at least 0.4,
\begin{align*}
\int_{\RR^d} p(\veta)\left|\sum_{j=1}^n \alpha_j e^{2\pi i \veta^\T \x_j}\right|^2\,d\veta + \frac{1}{3}\lambda \|\bs\alpha\|_2^2 &\leq 8n^2 \left(\frac{3}{4\rad}\right)^d \cdot \frac{B_d}{(\log s)^{\frac{d-2}{2}} s} + 3\lambda n +\frac{4}{3}\lambda n\\
&\leq \frac{n^2}{3\cdot 2^{2d+4}s} +\frac{13 \lambda n s}{3 s}\\
&\leq\frac{n^2}{3\cdot 2^{2d+4}s} +\frac{ n^2}{3\cdot 2^{2d+4} s}\\
&\leq \frac{n^2}{3\cdot 2^{2d+3} s}\\
&\leq \frac{2}{3} \cdot \frac{1}{s}\sum_{k=1}^s \left|\sum_{j=1}^{n} \alpha_j
e^{2\pi i \veta^\T_k \x_j} \right|^2
\end{align*}

\end{document}